\documentclass{amsart}[12 pt]

\usepackage{latexsym}
\usepackage{amsmath,amssymb,dsfont,amsfonts}
\usepackage{mathrsfs}
\usepackage{verbatim}
\usepackage{caption}
\usepackage{graphicx}
\usepackage{graphics}
\usepackage{adjustbox}
\usepackage{mathtools}
\usepackage{geometry}
\usepackage{booktabs}
\usepackage{enumitem}
\usepackage{xcolor}
\usepackage{algorithm}
\usepackage{algpseudocode}


\newtheorem{proposition}{Proposition}[section]
\newtheorem{theorem}{Theorem}[section]
\newtheorem{definition}{Definition}[section]

\newtheorem{lemma}{Lemma}[section]
\newtheorem{remark}{Remark}[section]

\newcommand{\R}{\mathbb{R}}

\DeclareMathOperator*{\argmin}{arg\,min} 

\numberwithin{equation}{section}

\allowdisplaybreaks

\begin{document}

\markboth{}{}

\title[Learning Stochastic Optimal Control by Off-Model Training and Importance Sampling]{Adaptive Learning via Off-Model Training and Importance Sampling for Fully Non-Markovian Optimal Stochastic Control. Complete version}

\author{Dorival Le\~ao}

\address{Departamento de Matem\'atica Aplicada e Estat\'istica. Universidade de S\~ao
Paulo, 13560-970, S\~ao Carlos - SP, Brazil} \email{leao@estatcamp.com.br}

\author{Alberto Ohashi}

\address{Departamento de Matem\'atica, Universidade de Brasília, 13560-970, Bras\'ilia - Distrito Federal, Brazil}\email{ohashi@mat.unb.br}

\author{Simone Scotti}
\address{Universit\`a di Pisa, DEM, via Ridolfi 10, Pisa Italy and Universit\'e Paris Cit\'e, LPSM} \email{simonescotti@unipi.it}

\author{Adolfo M. Dias da Silva}

\address{Departamento de Matem\'atica, Universidade de Brasília, 13560-970, Bras\'ilia - Distrito Federal, Brazil}\email{adolfoamds@gmail.com}

\date{\today}

\keywords{} \subjclass{Primary: 60H35; Secondary: 65C30}

\begin{center}
\end{center}
\begin{abstract}
This paper studies continuous-time stochastic control problems whose controlled states are fully non-Markovian and depend on unknown model parameters. Such problems arise naturally in path-dependent stochastic differential equations, rough-volatility hedging, and systems driven by fractional Brownian motion. Building on the discrete skeleton approach developed in earlier work, we propose a Monte Carlo learning methodology for the associated embedded backward dynamic programming equation. Our main contribution is twofold. First, we construct explicit dominating training laws and Radon--Nikodym weights for several representative classes of non-Markovian controlled systems. This yields an off-model training architecture in which a fixed synthetic dataset is generated under a reference law, while the dynamic programming operators associated with a target model are recovered by importance sampling. Second, we use this structure to design an adaptive update mechanism under parametric model uncertainty, so that repeated recalibration can be performed by reweighting the same training sample rather than regenerating new trajectories. For fixed parameters, we establish non-asymptotic error bounds for the approximation of the embedded dynamic programming equation via deep neural networks. For adaptive learning, we derive quantitative estimates that separate Monte Carlo approximation error from model-risk error. Numerical experiments illustrate both the off-model training mechanism and the adaptive importance-sampling update in structured linear-quadratic examples.
\end{abstract}

\maketitle

\tableofcontents

\section{Introduction}
Stochastic control problems driven by the Brownian filtration are often formulated in continuous-time on a compact interval $[0,T]$, while their numerical resolution typically requires a discrete-time approximation amenable to dynamic programming. In the Markovian case, this passage is classical and leads to a large numerical literature based on dynamic programming (see e.g. \cite{bertsekas}), regression methods, and more recently deep learning (see e.g. \cite{hure}, \cite{Han}). Beyond the Markovian framework, however, the situation is substantially more delicate. Even when the underlying noise is a Brownian motion, the controlled state may be fully non-Markovian in the sense that one cannot reduce it to a Markovian situation without adding infinitely many degrees of freedom. This is the typical situation found in controlled systems driven by fractional Brownian functionals or rough stochastic volatility models (see e.g. \cite{BankBayerHagerRiedelNauen2025}, \cite{rvol}). In such cases, the value process cannot in general be reduced to a finite-dimensional deterministic equation, and the construction of implementable numerical schemes for near-optimal controls remains a major challenge. In fact, any attempt to approximate optimal strategies along $[0,T]$ via naive discretization schemes driven by Brownian samples $\{B(t_i)\}_{i=1}^{N_T}$ over refining partitions $0=t_0 < t_1 < \ldots < t_{N_T}=T$ with size $N_T$ will end up in sub-optimal solutions and suffers from curse of dimensionality.

Based on a series of papers \cite{LEAO_OHASHI2017.1}, \cite{LEAO_OHASHI2013} and \cite{LEAO_OHASHI2017.2}, \cite{leaoohashi} developed a general approximation methodology for continuous-time control problems driven by controlled states adapted w.r.t. a given multi-dimensional Brownian motion $B$. Their philosophy is to project the original system onto a discrete-type skeleton generated by the Brownian hitting time

$$T_1 =\inf\{t \ge 0; |B(t)|= \epsilon\}$$
and by solving a suitable backward dynamic programming equation for the corresponding embedded control problem as $\epsilon\downarrow 0$. That result shows that, for fixed discretization level $\epsilon$, solving the embedded dynamic programming equation (\ref{valuefunc}) along $m$ steps (see (\ref{ektsteps})) yields near-optimal controls for the original continuous-time (possibly fully non-Markovian) problem as $\epsilon\downarrow 0$.

The purpose of the present paper is to develop a concrete Monte Carlo methodology for the embedded backward dynamic programming equation (\ref{valuefunc}) arising from the discrete skeleton of \cite{leaoohashi}, with particular emphasis on fully non-Markovian continuous-time control problems. Our objective is to construct a deep-learning-based numerical scheme that is simultaneously feasible for complex path-dependent systems, scalable for repeated learning and recalibration, and robust under parametric model uncertainty. In the framework proposed here, these requirements are addressed through a model-based stochastic control methodology in which importance sampling is a structural ingredient of both the training design and the adaptive updating procedure for mitigating model risk parameters.

Recent years have brought new numerical approaches to stochastic systems with memory beyond the classical Markovian setting. For a broader, though necessarily non-exhaustive, discussion of fully non-Markovian stochastic control, we refer the reader to \cite{leaoohashi}. Particularly relevant for the present paper is the recent work on stochastic control with signatures \cite{BankBayerHagerRiedelNauen2025}, where non-Markovian control problems are addressed by parameterizing open-loop controls through functionals of the driver’s path signature and optimizing over this class by Monte Carlo methods. In a different direction, motivated by rough-volatility applications (see e.g. \cite{rvol}, \cite{gatheral}), \cite{MotteHainaut2024} rely on finite-dimensional Markov approximations combined with Hamilton–Jacobi–Bellman equations, duality methods, or multifactor representations. In this direction, see also \cite{jaber}, \cite{alfonsi}, \cite{bayer}, \cite{rosenbaum}, \cite{fukasawa} and \cite{horvath}. More recently, kernel-weighted signature features have been proposed by \cite{HagerHarangPelizzariTindel2026} as explicit representations of Volterra-type memory. These works are closely related to the present paper in that they all seek mathematically tractable approaches to non-Markovian systems, but they proceed from viewpoints that are substantially different from the one adopted here.

A different issue, which is central for the present paper, arises when the controller does not have access to a perfectly specified model and must update parameter estimates over time as new information becomes available. In such a situation, the main difficulty is no longer only to approximate the control problem once, but to do so in a way that remains computationally feasible under repeated recalibration. Related questions have already been studied in adaptive stochastic control, including Bayesian and adaptive robust formulations in discrete-time Markovian settings; see, for instance, \cite{ChenMyung2020} and the recent contributions \cite{Erhan1,Erhan2}. Our objective is to construct a stochastic-control learning architecture that remains scalable under successive parameter updates, admits quantitative Monte Carlo error control, and still applies when the controlled state exhibits non-Markovian features. To the best of our knowledge, this combination of scalability under successive updates and quantitative error control for deep-learning architectures has not been systematically addressed, even in classical Markovian setting. It is precisely this combination of model uncertainty, numerical scalability, and complex state dependence that motivates the present work.

Our main contribution is the construction of a deep-learning Monte Carlo scheme for stochastic control problems of the form
\[
\inf_{u \in U_0^T}\mathbb{E}\big[\varphi(X^{u}(T))\big],
\]
where $U^T_0$ is a space of bounded adapted strategies over a compact interval $[0,T]$, the controlled state $X^{u}$ depends on a deterministic parameter $\theta$ which represents the model component to be learned or updated along the sequential decision procedure. The scope of the method includes several representative continuous-time control problems with genuinely non-Markovian features, studied in detail in Section~2.3. In the rough-volatility hedging problem, the controlled state is 

\begin{equation}\label{intrex1}
X^u(t)=  (S(t),Y^u(t)),
\end{equation} 
where $S$ denotes the risky asset price and $Y^u$ the wealth process associated with a trading strategy $u$, while the volatility is driven by a fractional Ornstein--Uhlenbeck factor (see e.g. \cite{cheridito}) with Hurst parameter $0 < H < \frac{1}{2}$ in the regime $H\approx 0$. In the path-dependent SDE setting, the state is driven by 
\begin{equation}\label{intrex2}
dX^u(t)=\alpha(t,X^u|_{t},u(t))\,dt+\sigma(t,X^u|_{t},u(t))\,dB(t),
\end{equation}
so that the dynamics depend on the whole past trajectory of the state $X^u|_t = \{X^u(s); s\le t\}$. In the fractional-noise setting, the state is driven by 
\begin{equation}\label{intrex3}
dX^u(t)=\varrho(X^u(t),u(t))\,dt+\sigma\,dB^H(t),
\end{equation}
and the non-Markovian feature is induced by the fractional Brownian motion $B^H$ itself. The reader should think each model (\ref{intrex1}), (\ref{intrex2}) and (\ref{intrex3}) depends on a parameter $\theta$ associated with their coefficients.  

An important ingredient of the present paper is the introduction of dominating training laws $\mu$ with Radon-Nikodym derivatives $r_j$ which realize

\begin{equation}\label{transINTR}
K_j(b,a; dxdx'):=\mathbb{P}[ (\mathcal{W}_j, \Delta X^a_j) \in dxdx' | \Xi_{j-1}= b]= r_j(a,x'; b)\mu(dx')\nu(dx)
\end{equation}
where $\Xi_{j-1}$ is the history (\ref{feedH}) of the controlled system for a given parameter model $\theta$. The law $\mu$ must be interpreted as the training data for solving the dynamic programming equation (\ref{valuefunc}) associated with a parametric set which describes the uncertainty on the coefficients of the controlled states. Concretely speaking, $r_j$ is typically dictated by the state space regions visited by the underlying controlled states parameterized by any class of coefficients in (\ref{intrex1}), (\ref{intrex2}) and (\ref{intrex3}) that the agent believes that drive the controlled dynamics. The variables $(\Delta X^a_i)_{i\ge 1}$ must be interpreted as Euler-Maruyama-type increments of the models (\ref{intrex1}), (\ref{intrex2}) and (\ref{intrex3}) described in section \ref{EXsection} and adapted to the filtration generated by an i.i.d sequence with common law $(T_1,B(T_1)) \stackrel{(d)}{=}\nu$. 

A central component of the paper is the explicit construction of admissible training laws $\mu$ equipped with Radon derivatives $r_j$ supporting the numerical resolution of these problems. Section \ref{MSECTION} is devoted to this issue for the controlled systems (\ref{intrex1}), (\ref{intrex2}) and (\ref{intrex3}). Once these reference laws are available, the learning architecture becomes genuinely off-model: the training sample is generated under the reference law, while the dynamic programming equation associated with a given model parameter is recovered by a subtle use of importance sampling associated with $r_j$. The paper provides a quantitative analysis of this scheme. For a fixed model parameter, Theorems \ref{mainresult} and \ref{mainresultRS} establish non-asymptotic convergence rates for the deep-learning Monte Carlo approximation of the embedded backward dynamic programming equation (\ref{valuefunc}). 

The construction of $(\mu,r_j)_{j\ge 1}$ presented in Theorems \ref{pdsdeth}, \ref{FBMtheorem} and \ref{rvoltheorem} is sufficiently rich to support learning across different model specifications in a compact set of parameters $\Theta$. Suppose the controlled dynamics (for instance (\ref{intrex2})) is driven by an unknown deterministic parameter $\theta^\star$. Section~\ref{adaptivesec} presents an adaptive Monte Carlo numerical scheme which produces a scalable and efficient algorithm for updating the parameters $\theta$ of the controlled models towards $\theta^\star$. At the conceptual level, the value functionals are defined recursively by
\[
\mathbb{V}_m^\theta(\mathbf{o}_m)=\varphi\Big(x_0+\sum_{i=1}^m y_i\Big),
\qquad
\mathbb{V}_j^\theta(\mathbf{o}_j)=\min_{a\in \mathbb{A}}\mathbf{U}_j^\theta(\mathbf{o}_j,a), \quad 0\leq j\leq m-1,
\]
where (\ref{transINTR}) yields 
\begin{equation}\label{whyr}
\mathbf{U}_j^\theta(\mathbf{o}_j,a)
=
\int_{\mathbb{H}^j}
\mathbb{V}_{j+1}^\theta\!\big(\pi_2(\mathbf{o}_j),x,y\big)\,
r_{j+1}^\theta(a,y,\mathbf{o}_j)\mu(dy)\nu(dx).
\end{equation}
Here, $\mathbf{o}_j = (w_1, y_1, \ldots, w_j,y_j) \in \mathbb{H}^j$ represents the history of the controlled system taking values on a suitable augmented state-space $\mathbb{H}^j$, $\mathbb{A}$ is a compact action space, $\pi_2$ denotes the projection onto the relevant coordinates (see (\ref{pi2def})) and $r^\theta_j$ is the importance-sampling weight associated with the controlled state driven by $\theta \in \Theta$. The main computational difficulty is that the parameter \(\theta^\star\) is unknown. Thus, the transition law (\ref{transINTR}) is not fixed a priori, but belongs to the family
\[
K^\theta_j(b,a; dxdx')_{\theta\in\Theta}.
\]
For a given compact subset $\Theta$ of parameters, we can explicitly construct importance samplings weights and training measures $(\mu, r^\theta_j)_{j\ge 1}$, where $\mu$ only depends on $\Theta$. If one recomputes the backward dynamic programming equation by fresh simulation every time the current estimate of \(\theta^\star\) is updated, then the resulting algorithm would become prohibitively expensive. Indeed, each new parameter estimate would require regenerating trajectories of the controlled system under the new law and recomputing the associated Monte Carlo approximations of the continuation operators. 

We propose to decouple \emph{sampling} from \emph{model updating} by means of the importance sampling weight as follows: Rather than resampling new samples from $K^\theta_j$ whenever \(\theta\) changes, we generate a synthetic sample under the fixed dominating reference law \(\mu\) and then reweight the sample through the density \(r_j^\theta\). More precisely, suppose that for each time $j$, we simulate $M$ samples
\[
(w_{j,1},y_{j,1}),\dots,(w_{j,M}, y_{j,M})
\]
independently according to the reference law $\nu \otimes \mu$. Then
\[
\mathbf{U}^{\theta}_j (\mathbf{o}_j,a)
\approx
\frac1M\sum_{p=1}^M
\mathbb{V}_{j+1}^\theta\big(\pi_2(\mathbf{o}_j),w_{j+1,p}, y_{j+1,p}\big)\,
r^\theta_{j+1}\big(a,y_{j+1,p}, \mathbf{o}_j\big),
\]
for $j=m-1, \ldots, 0$. Consequently, when the current parameter estimate changes from \(\theta\) to \(\theta'\), there is no need to regenerate the next-state sample. One simply updates the weights
\begin{equation}\label{isreason}
r^\theta_j\big(a,y_{j,p}, \mathbf{o}_j\big)
\quad\leadsto\quad
r^{\theta'}_j\big(a,y_{j,p}, \mathbf{o}_j\big).
\end{equation}
for $j=m,\ldots, 1$. 

The proposed adaptive Monte Carlo scheme admits quantitative estimates as follows: Concerning (\ref{intrex2}), Proposition \ref{adapresSDE} yields the total error is decomposed as

\begin{equation}\label{inest}
\mathbb{E}_{M}\Big|
\widehat{\mathbb{V}}_{j,M}^{\theta}(O_j)-\mathbb{V}_j^{\theta^\star}(O_j)
\Big|
\lesssim_{\mathbb{P}}
\mathbb{E}_{M}\Big|
\widehat{\mathbb{V}}_{j,M}^{\theta}(O_j)-\mathbb{V}_j^{\theta}(O_j)
\Big|
+
|\theta^\star-\theta|,
\end{equation}
for $j=m,\ldots, 0$. The first term in the right-hand side of (\ref{inest}) is the the Monte Carlo learning error under the estimated model $\theta$, controlled by Theorems \ref{mainresult}, while the second term is the model-risk contribution induced by the discrepancy between $\theta$ and $\theta^\star$. Here, $\mathbb{E}_M$ denotes the conditional expectation w.r.t. the Monte Carlo training set and $O_j\stackrel{(d)}{=} j-\text{fold product measure of}~ \nu\otimes \mu$. Similar estimate holds for models (\ref{intrex1}) and (\ref{intrex3}). See Proposition \ref{LipvaluePARVOL} and Remark \ref{LipvaluePARfbm}.

The construction of dominating training laws in the present paper is partly inspired by \cite{hure}, where such a domination structure is assumed a priori in order to derive convergence rates for discrete-time Markov decision processes attached to neural networks. Our contribution goes further in a different direction: for the concrete continuous-time fully non-Markovian controlled models (\ref{intrex1}), (\ref{intrex2}) and (\ref{intrex3}), we construct explicit admissible dominating laws and the corresponding Radon--Nikodym weights, which then become the basis of our adaptive importance-sampling scheme under model uncertainty via (\ref{isreason}) that we propose in the present work. While importance sampling is classical in stochastic control and reinforcement learning (see e.g. \cite{KappenRuiz2016}, \cite{GobetTurkedjiev2017}, \cite{HannaNiekumStone2021}), its use in the present paper appears to be of a different nature. Existing works mainly employ importance sampling for variance reduction under a fixed model, or for off-policy correction across policies. By contrast, our framework allows us to an explicit construction of dominating training laws for the controlled state process in order to reuse a single dataset across successive model updates. The importance sampling weights crucially depend on the size $\varepsilon$ of the barrier where the driving Browian motion lives. This allows the dynamic programming equation (\ref{valuefunc}) to be updated by reweighting rather than by resimulation, and it provides a natural mechanism for warm-start initialization of the neural networks under repeated recalibration. The resulting role of importance sampling is therefore not merely statistical; it is structural, since it is precisely what makes the adaptive deep-learning scheme scalable under model risk.

Section \ref{sec:numerical} illustrates the proposed adaptive Monte Carlo methodology with numerical experiments for linear quadratic control problems . First, we investigate off-model training by analyzing the effect of different exploration strategies on the numerical performance of the scheme, with experiments devoted to the mean variance hedging in the rough-volatility model. Second, we investigate parametric model risk in a simple example, where the experiments show the real effectiveness of the adaptive importance sampling scheme facing parametric model risk. In this way, the numerical section highlights the two practical facets of the paper: off-model learning and adaptive updating under model uncertainty.

The paper is organized as follows. Section \ref{reviewsec} presents a quick overview of the methodology of the article \cite{leaoohashi}. In particular, it presents the backbone backward dynamic programming algorithm (\ref{valuefunc}) investigated in the present article. Section \ref{MSECTION} presents the construction of the importance sampling weights associated with the controlled systems (\ref{intrex1}), (\ref{intrex2}) and (\ref{intrex3}). It also presents the main structural assumptions of the present article namely the assumptions H0-H1-H2 and R1. Section \ref{mainressec} presents the main results of the article, namely Theorems \ref{mainresult} and \ref{mainresultRS}. The adaptive learning integrated into dynamic programming via importance sampling is presented in Section \ref{adaptivesec}. Section \ref{sec:numerical} presents the numerical experiments which illustrate Theorem \ref{mainresult} and Section \ref{MSECTION}. Sections \ref{appendix1}, \ref{appendix2}, \ref{appendix3} and \ref{appendix4} are devoted to the proofs of Theorem \ref{mainresult} and \ref{mainresultRS}.

\section{Preliminaries and a brief review of \cite{leaoohashi}}\label{reviewsec}
This section recalls the basic structure developed by \cite{leaoohashi}. More importantly, we present the dynamic programming algorithm associated to a controlled system that will be the object of the Monte Carlo study of the present article. Before this discussion, let us present the standing notation used throughout this paper. We write $a\lesssim b$ for two positive quantities to express an estimate of the form
$a \le C b$, where $C$ is a generic constant which may differ from line to line. If $\gamma$ is a parameter, then $a\lesssim_{\gamma} b$ means that $a \le C b$, where the constant $C$ depends on $\gamma$. We are going to fix a $d$-dimensional Brownian motion $B = \{B^{1},\ldots,B^{d}\}$ on $(\Omega, \mathbb{F}, \mathbb{P})$, where $\Omega$ is the space $C(\mathbb{R}_+;\mathbb{R}^d) := \{f:\mathbb{R}_+ \rightarrow
\mathbb{R}^d~\text{continuous}\}$, $\mathbb{P}$ is the Wiener measure on $\Omega$ such that $\mathbb{P}\{B(0) = 0\}=1 $ and $\mathbb{F}:=(\mathcal{F}_t)_{t\ge 0}$ is the usual $\mathbb{P}$-augmentation of the natural filtration generated by the Brownian motion. If $X$ is a process with left-hand limits, then we denote $\Delta X(t):= X(t)- X(t-)$, where $X(t-):= \lim_{s\uparrow t } X(s)$. The notation $x_M = \mathcal{O}_\mathbb{P}(y_M)$ as $M\rightarrow +\infty$ means that there exists $c>0$ such that $\mathbb{P}\{|x_M|> c|y_M|\}\rightarrow 0$ as $M\rightarrow +\infty$. Sometimes, we also write $x_M\lesssim_{\mathbb{P}} y_M$ to shorten notation. The finite constant $T$ is the terminal time of the stochastic control problem. The symbol $\top$ denotes the transpose operation acting on matrices.

Let $U^T_0$ be the set of all $\mathbb{F}$-progressively adapted processes on $[0,T]$ taking values on a compact (uncountable) action space $\mathbb{A}\subset \mathbb{R}^p$. Let $\varphi: \mathbb{R}^q\rightarrow \mathbb{R}$ be a globally Lipschitz function. The set of controls $U^T_0$ gives rise to a large class of $q$-dimensional controlled systems by the filtration $\mathbb{F}$

$$U^T_0 \ni \phi \mapsto X^\phi $$
where the process $\{X^\phi(t); 0\le t\le T\}$ is typically an $\mathbb{R}^q$-valued non-Markovian (for each control $\phi$) integrable process as discussed in the Introduction. The article \cite{leaoohashi} developed a numerical scheme to produce a near optimal control $\phi^{\star,\eta}$

\begin{equation}\label{INTROpr3}
\mathbb{E}\Big[\varphi\big(X^{\phi^{*,\eta}}(T)\big)\Big] \le \inf_{\phi\in U^T_0}\mathbb{E}\Big[\varphi\big(X^\phi(T)\big)\Big]+\eta,
\end{equation}
for an arbitrary error bound $\eta>0$. The article \cite{leaoohashi} proposes a methodology for computing a near optimal control $\phi^{\star,\eta}$ realizing (\ref{INTROpr3}), by evaluating the control problem  


$$
\mathbb{E}\Big[\varphi\big(X^{k,\phi^{*,\eta}}(T)\big)\Big] \le \inf_{u\in U^{e(k,T)}_0}\mathbb{E}\Big[\varphi\big(X^{k,u}(T)\big)\Big]+\eta,
$$
where the controlled state $\mathcal{X} = \{X^{k,u}; u\in U_0^{e(k,T)}\}$ is a \textit{discrete version} of the original controlled system $\phi \mapsto X^\phi$, where $e(k,T)$ is a suitable number of steps to recover (\ref{INTROpr3}) over the entire period $[0,T]$ as an arbitrary accuracy level $\varepsilon_k \downarrow 0$ as $k\rightarrow +\infty$. The class of controls $U^T_0$ is replaced by a set $U^{e(k,T)}_0$ of stepwise-constant processes parameterized by $\varepsilon_k \downarrow 0$ as $k\rightarrow +\infty$ and adapted to a suitable pure jump process constructed from the Wiener space that we describe in the next section. 

\begin{remark} 
Whenever necessary, in case the payoff function $\varphi:\mathbb{R}^q\rightarrow \mathbb{R}$ is only locally Lipschitz, we will assume the controlled state is bounded by a possibly large arbitrary constant. In concrete applications arising in industrial processes and finance, this is not at all a restrictive assumption. For instance, in practice, the partial hedging problem in incomplete markets, without loss of generality, can be considered in terms of risky asset prices bounded by a large constant.   
\end{remark}

 
\subsection{Imbedding scheme}\label{discretesec}
Throughout this section, we fix a accuracy level $\varepsilon_k \in (0,1)$. The imbedding procedure will be based on a class of pure jump processes driven by suitable waiting times which describe the local behavior of the Brownian motion. We briefly recall the basic properties of this skeleton. For more details, we refer the reader to the work \cite{leaoohashi}. We set $T^{k}_0:=0$ and

\begin{equation}\label{stopping_times}
T^{k}_n := \inf\{T^{k}_{n-1}< t <\infty;  | B(t) - B(T^{k}_{n-1}) |_{\max} = \varepsilon_k\}, \quad n \ge 1,
\end{equation}
and $|\cdot |_{\max}$ in (\ref{stopping_times}) corresponds to the maximum norm on $\mathbb{R}^d$. This implies

$$
\Delta T^k_{n}:=T^k_n-T^k_{n-1}=\min_{j\in\{1,2,\dots, d\}}{\{\Delta^{k,j}_n\}}~a.s,
$$
where

\begin{equation}\label{Deltacor}
\Delta^{k,j}_n := \inf\{0< t <\infty;  | B^j(t+T^{k}_{n-1}) - B^j(T^{k}_{n-1}) | = \varepsilon_k\}, \quad n \ge 1.
\end{equation}
Then, we define $A^k :=(A^{k,1}, \cdots , A^{k,d})$ by

\begin{equation}\label{akprocess}
A^{k,j} (t) := \sum_{n=1}^{\infty} \left( B^j(T^{k}_{n}) - B^j(T^{k
}_{n-1})  \right) \mathds{1}_{\{T^{k}_n\leq t \}};~t\ge0, ~ j=1, \ldots , d.
\end{equation}

By the strong Markov property, we observe that

\begin{enumerate}
  \item The jumps $\Delta A^{k,j}(T^k_n) = A^{k,j}(T^k_n)-A^{k,j}(T^k_{n}-); n=1, 2, \ldots$ are independent and identically distributed (iid).
  \item The waiting times $\Delta T^k_n; n=1, 2,\ldots$ are iid random variables in $\mathbb{R}_+$.
  \item The families $(\Delta A^{k,j}(T^k_n); n=1, 2, \ldots)$ and $(\Delta T^k_n; n=1, 2,\ldots)$ are independent when $d=1$ and, otherwise, they are dependent. 
\end{enumerate}
Let $\mathbb{F}^k$ be the filtration generated by $A^k$. In order to recover the stochastic control problem (\ref{INTROpr3}) over the $[0,T]$, we define

\begin{equation}\label{ektsteps}
e(k,t):= \Big\lceil \frac{\varepsilon_k^{-2} t}{\chi_d}\Big\rceil; 0\le t\le T,
\end{equation}
where $\lceil x\rceil$ is the smallest integer greater or equal to $x\ge 0$ and

\begin{equation}\label{kappaDEF}
\chi_d:=\mathbb{E}\min\{\tau^1, \ldots, \tau^d\},
\end{equation}
where $(\tau^j)_{j=1}^d$ is an iid sequence of random variables with distribution $\inf\{t>0; |W(t)|=1\}$ for a real-valued standard Brownian motion $W$. From Lemma A3 in \cite{leaoohashi}, for each $t\in [0,T]$, we know that
\begin{equation}\label{tekt}
T^k_{e(k,t)}\rightarrow t~\text{a.s and in}~L^p(\mathbb{P}),
\end{equation}
as $k\rightarrow \infty$, for each $t\ge 0$ and $p\ge 1$.

\begin{remark}\label{growthlemma}
The number $e(k,T)$ should be interpreted as the number of necessary steps to compute our discrete-type dynamic programming equation. Moreover, one can check
$$\frac{1}{2 d}\le \chi_d.$$
Therefore, for given $k\ge 1$ and $T$, the number of periods $e(k,T)$ grows no faster than the dimension of the driving Brownian motion.
\end{remark}

Let $\mathbb{F}^k$ be the raw filtration generated by $A^k$. We observe that

$$\mathcal{F}^k_{T^k_n} = \sigma (\Delta A^k(T^k_i), \Delta T^k_i; 1\le i\le n),$$
for $n\ge 1$.

\begin{definition}\label{discreteskeleton}
For $\mathcal{T} := \{T^k_n; n\ge 0\}$, the structure $\mathscr{D} = \{\mathcal{T}, A^{k}; k\ge 1\}$ is called a \textbf{discrete-type skeleton} for the Brownian motion.
\end{definition}

Let $U^{k,e(k,T)}_{0}$ be the class of $\mathbb{F}^k$-predictable processes of the form

\begin{equation}\label{controlform0}
u(t) = \sum_{j=1}^{e(k,T)}u_{j-1}\mathds{1}_{\{T^k_{j-1}< t\le T^k_j\}}; 
\end{equation}
for $0\le t\le T$, where $u_{j-1}$ is an $\mathbb{A}$-valued $\mathcal{F}^k_{T^k_{j-1}}$-measurable random variable. Any element $u \in U^{k,e(k,T)}_0$ can be represented by a list $u_0, \ldots, u_{e(k,T)-1}$.  

Let $O_T(\mathbb{F}^k)$ be the set of $\mathbb{F}^k$-optional processes of the form

$$Z^k(t) = \sum_{n=0}^{e(k,T)} Z^k(T^k_n)\mathds{1}_{\{T^k_n\le t < T^k_{n+1}\}},$$
for $0\le t\le T$, where $Z^k(T^k_n)$ is $\mathcal{F}^k_{T^k_{n}}$-measurable for every $n\ge 0$ and $k\ge 1$.

Let us now present two concepts which will play a key role in this work.
\begin{definition}\label{GASdef}
A \textbf{controlled imbedded discrete structure} $\mathcal{Y} = \big((Y^k)_{k\ge 1},\mathscr{D}\big)$ consists of the following objects: a discrete-type skeleton $\mathscr{D}$ and a map $u\mapsto Y^{k,u}$ from $U^{k,e(k,T)}_0$ to $O_T(\mathbb{F}^k)$ such that

\begin{equation}\label{antiprop}
Y^{k,u}(T^k_{n+1})~\text{depends on the control only at}~(u_0, \ldots, u_n),
\end{equation}
for each integer $n\in \{0,\ldots,e(k,T)-1\}$.
\end{definition}
Controlled imbedded discrete structures mimic $\mathbb{F}$-adapted continuous-time controlled processes $X^u$. For the impatient reader, we refer to Section \ref{EXsection} for examples. In view of concrete applications, we need to impose a natural form on the increments of a controlled imbedded discrete structure $\mathcal{Y}$. In order to describe those restrictions, we need to introduce some objects. Let us define

$$\mathbb{I}_k^o:=\Big\{ (i^k_1, \ldots, i^k_d); i^k_\ell\in \{-1,0,1\}~\forall \ell \in \{1,\ldots, d\}~\text{and}~\sum_{j=1}^d|i^k_j|=1   \Big\}$$
and $$\mathbb{I}_k := \Big\{ \varepsilon_k \left(i^k_1 \mathds{1}_{\{\mid i^k_1 \mid =1\}} + z^k_1 \mathds{1}_{\{\mid i^k_1 \mid \neq 1\}} , \ldots , i^k_d \mathds{1}_{\{\mid i^k_d \mid =1\}} +z^k_d \mathds{1}_{\{\mid i^k_d \mid \neq 1\}} \right) ; (i_1^k , \ldots , i^k_d) \in \mathbb{I}_k^o,$$
$$(z^k_1, \ldots , z^k_d) \in (-1,1)^d \Big\}. $$

For obvious reasons, $\mathbb{W}_k:=(0,+\infty)\times \mathbb{I}_k$ will be called as the \textit{noise space}. The $n$-fold Cartesian product of $\mathbb{W}_k$ is denoted by $\mathbb{W}_k^n$ and a generic element of $\mathbb{W}^n_k$ will be denoted by

$$
(w^k_1, \ldots, w^k_n)\in \mathbb{W}^n_k,
$$
where $w^k_r = (s^k_r,\tilde{i}^k_r)\in \mathbb{W}_k$ for $1\le r\le n$. Let us define

$$\Delta A^k(T^k_n):=(\Delta A^{k,1}(T^k_n), \ldots, \Delta A^{k,d}(T^k_n)),$$
where

$$\Delta A^{k,j}(T^k_n) = B^j (T^k_n) - B^j (T^k_{n-1}),$$
for $1\le j\le d; n,k\ge 1$. Observe that $(\Delta T^k_n,\Delta A^{k}(T^k_n)) \in \mathbb{W}$ a.s. for $n\ge  1$. Let us define

\begin{equation}\label{caliak}
\mathcal{A}^k_n:= \Big(\Delta T^k_1, \Delta A^{k}(T^k_1), \ldots, \Delta T^k_n, \Delta A^{k}(T^k_n)\Big)\in \mathbb{W}^n_k~a.s.
\end{equation}
and 

\begin{equation}\label{caliwk}
\mathcal{W}^k_n:= (\Delta T^k_n, \Delta A^{k}(T^k_n)); n \ge 1.
\end{equation}
Observe that the strong Markov property yields $\{\mathcal{W}^k_n; n\ge 1\}$ is iid. 

We set $\mathbb{H}_k:= \mathbb{W}_k \times \mathbb{R}^q$, where $\mathbb{R}^q$ is the state-space. We denote $\mathbb{H}^{i}_k$ as the $i$-fold Cartesian product of $\mathbb{H}_k$. In the sequel, we fix an initial condition $x_0\in \mathbb{R}^q$. The elements of $\mathbb{H}^{j}_k$ will be denoted by

$$
\textbf{o}^{k}_i := \Big( (w^k_1,y^k_1), \ldots, (w^k_i, y^k_i) \Big).
$$
For convenience, we set $\mathbf{o}^k_0:=(0,0,x_0)$ and $\mathbb{H}_k^0:= \{(0,0,x_0)\}$. We need to introduce the following projections $\pi_2$ and $\pi_3$ as follows: For $\textbf{o}^{k}_i = \big( (w^k_1,y^k_1), \ldots, (w^k_i, y^k_i) \big)$, we set

\begin{equation}\label{pi2def}
\pi_2(\textbf{o}^{k}_i) := \big( (s^k_1,y^k_1), \ldots, (s^k_{i}, y^k_i) \big)\quad \pi_3(\textbf{o}^{k}_i) := \big( w^k_1, \ldots, w^k_i\big),
\end{equation}
for $i\ge 1$.



In the sequel, we describe the dynamics of a controlled imbedded discrete structure $\mathcal{X}  = \big( (X^k)_{k\ge 1}, \mathscr{D} \big)$ as a function of the state-space and the noise as follows: Fix a set of Borel functions

$$\blacktriangle x_j : \big(\mathbb{R}_+\times \mathbb{R}^q\big)^{j-1} \times \mathbb{A}\times (0,\infty)\times \mathbb{R}^d \rightarrow \mathbb{R}^q,\quad \ell_j : \mathbb{W}^j_k\rightarrow \mathbb{R}^d,$$
for $j=1,\ldots, e(k,T)$, where we set $\big(\mathbb{R}_+\times \mathbb{R}^q\big)^{0}:=\{0\}\times \mathbb{R}^q$.

For a given admissible control $u = (u_0, \ldots, u_{e(k,T)-1})\in U^{e(k,T)}_0$, we set $\Xi^{k,u}_0 := (0,0,x_0)$ and we assume a controlled imbedded discrete structure $\mathcal{X}  = \big( (X^k)_{k\ge 1}, \mathscr{D} \big)$ satisfies the following dynamics:

\

\noindent \textbf{Assumption (E1):}

\begin{equation}\label{feedH}
\Delta X^{k,u}(T^k_j) = \blacktriangle x_j \Big(\pi_2(\Xi^{k,u}_{j-1}),u_{j-1} , \Delta T^k_j , \ell_j(\mathcal{A}^k_j) \Big),
\end{equation}
where

\begin{equation}\label{Xioperator}
\Xi^{k,u}_j :=\Big((\mathcal{W}^k_1, \Delta X^{k, u}(T^k_1 ) ), \ldots, (\mathcal{W}^k_j, \Delta X^{k, u}(T^k_j) ) \Big),
\end{equation}
for $1\le j\le e(k,T)$.

\subsection{Examples}\label{EXsection}
Next, we illustrate Assumption (E1) with three fundamental examples.
      
\subsubsection{Partial hedging with rough stochastic volatility}\label{Ex2}
For a two-dimensional Brownian motion $(B^1,B^2)$, let $W$ be the real-valued Brownian motion 

$$W:= \rho B^1 + \bar{\rho} B^2,$$
where $\bar{\rho}:=\sqrt{1-\rho^2}$ for $-1 \le \rho \le 1$. Let $\mathbf{C}^\lambda_0$ be the space of $\lambda$-H\"{o}lder continuous real-valued functions on $[0,T]$ and starting at zero equipped with the usual norm. For $0 < H < \frac{1}{2}$, let us define

\begin{equation}\label{RKHker}
\begin{split}
K_{H,1}(t,s)&:=c_H t^{H-\frac{1}{2}}s^{\frac{1}{2}-H}(t-s)^{H-\frac{1}{2}},\\
K_{H,2}(t,s)&:=c_H(1/2-H)s^{\frac{1}{2}-H}\int_s^t u^{H-\frac{3}{2}}(u-s)^{H-\frac{1}{2}}du,
\end{split}
\end{equation}
for $0< s < t$, where $c_H$ is a suitable constant (see e.g \cite{ohashifrancys}). For each $f \in \mathbf{C}^\lambda_0$, we define the linear map

\begin{equation}\label{LambdaH}
(\Lambda_Hf)(t):=\int_0^t [f(t)-f(s)]\partial_s K_{H,1}(t,s)ds - \int_0^t \partial_s K_{H,2}(t,s)f(s)ds,
\end{equation}
for $0\le t\le T.$ By Theorem 2.2 in \cite{ohashifrancys} and the linearity of $\Lambda_H$, 

$$W^H := \Lambda_H(W)=\rho(\Lambda_H B^1) + \bar{\rho}(\Lambda_HB^2)$$
is a fractional Brownian motion correlated with $B^1$. The risky asset price is 

\begin{equation}\label{RSVM}
dS(t) =  \mu_{\text{drift}} S(t) dt + S(t) \vartheta(V(t)) dB^1(t)
\end{equation}
where, for simplicity, $\vartheta:\mathbb{R}_+\rightarrow \mathbb{R}_+$ is bounded and the underlying market interest rate is zero. Here, $V(t)= \exp (Z(t))$ and  

$$dZ(t) = \zeta dW^H(t)- \beta(Z(t)-\varkappa)dt
$$
is the fractional Ornstein-Uhlenbeck process for $0 < H <\frac{1}{2}$, $\zeta,\beta>0$, $\mu_{\text{drift}} \in \mathbb{R}$ and $ \varkappa\in \mathbb{R}$. It is well-known (see e.g Prop A1 in \cite{cheridito})

\begin{equation}\label{OUh}
Z(t) =\varkappa + e^{-\beta t}(z_0-\varkappa) + \zeta W^H(t) -\beta \zeta e^{-\beta t}\int_0^t W^H(u)e^{\beta u}du,
\end{equation}
for $0\le t\le T$. The model (\ref{RSVM}) is a rough stochastic volatility model as described by \cite{gatheral}.  
For further details about rough stochastic volatility models, see e.g. \cite{rvol}. For a given strike $K >0$, a price $c$ and $q\in [1, \infty)$, we look for a numerical algorithm to compute a near optimal control $\phi^\star$ 

\begin{equation}\label{mvhdef}
\phi^\star \in \argmin_{\phi\in U^T_0}\mathbb{E}\Big|Y^\phi(T) - \big(K- S(T) \big)^+\Big|^q
\end{equation}
where $Y^{\phi^\star}(0) = c$ and  

$$Y^\phi(t) = c +  \int_0^t\phi(r)dS(r).$$
Observe the controlled state is $X^\phi(t)= (S(t), Y^\phi(t))^\top$ for $0\le t\le T$. If $q=2$, this is a linear-quadratic stochastic control problem. 
\begin{remark}
In case $\rho=\pm 1$, there exists only one risky asset price and one Brownian motion so that the market is complete. In this case, there exists $(c^\star, \phi^\star) \in \mathbb{R}_+\times U^T_0$ realizing 

$$Y^{\phi^\star}(T) = \big(K- S(T) \big)^+,~Y^{\phi^{\star}}(0)=c^\star,$$
where $c^\star = \mathbb{E}_{\mathbb{Q}}[(K-S_T)^+]$ and $\mathbb{Q}$ is the unique martingale measure. 
\end{remark}


The controlled imbedded discrete structure for $X^\phi(t) = (S(t), Y^\phi(t))^\top$ is given by $X^{k,u} = \big(S^{k}, Y^{k,u}\big)^\top$, where

$$\Delta S^{k}(T^k_j) = \mu_{\text{drift}} S^{k}(T^k_{j-1})\Delta T^k_j + S^{k}(T^k_{j-1})\vartheta( V^{k}(T^k_{j-1}))\Delta A^{k,1}(T^k_j)$$

and 

$$Y^{k,u}(T^k_j) = \sum_{\ell=1}^j u_{\ell-1} \Delta S^{k}(T^k_\ell),$$
for $u = (u_0, u_1, \ldots, u_{e(k,T)-1}) \in U^{k,e(k,T)}_0$, $1\le j\le e(k,T)$ and $Y^{k,\phi}(0)=c$. The process $V^k$ is a discrete version of $V$ which we describe as follows. In the sequel, we set

$$\bar{t}_k = \max\{T^k_n; T^k_n \le t\},\quad \bar{t}^{+}_k:= \min\{T^{k}_n ; \bar{t}_k < T^{k}_n\}.
$$
For each $i=1,2$, we define

\begin{equation}\label{discreterough1}
B^{k,i}_H(t):=  \int_0^{\bar{t}_k} \partial_s K_{H,1}(\bar{t}_k,s)\big[A^{k,i}(\bar{t}_k) - A^{k,i}(\bar{s}^{+}_k)\big]ds - \int_0^{\bar{t}_k}\partial_s K_{H,2}(\bar{t}_k,s)A^{k,i}(s)ds,
\end{equation}
and
$$W^k_H(t):= \rho B^{k,1}_H(t) + \bar{\rho} B^{k,2}_H(t); 0\le t\le T.$$
The processes $(B^{k,1}_H,B^{k,2}_H)$ are the $\mathscr{D}$-imbedded discretizations of $(B^1_H,B^2_H)$. For further details, we refer reader to  \cite{ohashifrancys}. An imbedded discrete structure for the volatility process is given by

\begin{equation}\label{OUdiscrete}
Z^k(T^k_n):=\varkappa + e^{-\beta T^k_n}(z_0-\varkappa) + \zeta W^k_H(T^k_n) -\beta \zeta e^{-\beta T^k_n}\int_0^{T^k_n} W^{k}_H(s)e^{\beta \bar{s}_k}ds
\end{equation}
and 

\begin{equation}\label{rvola}
V^k(T^k_n):= \exp \big( Z^k(T^k_n)\big),
\end{equation}
for $1\le n \le e(k,T)$ and $Z^k(0)=z_0$. We have 

\begin{eqnarray}
\nonumber\Delta X^{k,u}(T^k_j) &=& \left(
                                     \begin{array}{c}
                                       \Delta S^{k}(T^k_j) \\
                                       \Delta Y^{k,u}(T^k_j) \\
                                     \end{array}
                                   \right)\\
\label{deltarough}&=& \blacktriangle x_j \Big(\pi_2(\Xi^{k,u}_{j-1}),u_{j-1} , \Delta T^k_j , \ell_j(\mathcal{A}^k_j) \Big)
\end{eqnarray} 
where

$$\ell_j(\mathcal{A}^k_j) = \big(V^k(T^k_j), \Delta A^{k,1}(T^k_j) \big),$$
for $1\le j\le e(k,T)$. It is important to notice that $V^k(T^k_j)$ is a function of the whole path $\mathcal{A}^k_{j} = (\mathcal{W}^k_1, \ldots, \mathcal{W}^k_j)$ and not only of $\mathcal{W}^k_j$. Next, we present an elementary representation of $W^k_H$ which is implemented in the present work. The proof of Lemma \ref{repWH} is postponed to Section \ref{appendix4}.  

\begin{lemma}\label{repWH}
$$W^k_H(T^k_n) = \rho B^{k,1}_H(T^k_n) + \bar{\rho} B^{k,2}_H(T^k_n),$$
for $n\ge 0$. Here, $B_H^{k,i}(T_n^k)=0$, for $n=0,1$ and $i=1,2$ and 
$$B_H^{k,i}(T_n^k)= \sum_{j=2}^n \Delta A^{k,i}(T_j^k)\,K_{H,1}(T_n^k,T_{j-1}^k)
   + \sum_{j=1}^{n-1} \Delta A^{k,i}(T_j^k)\,K_{H,2}(T_n^k,T_j^k),$$
for $n\ge 2$ and $i=1,2$.    
\end{lemma}

\subsubsection{Path dependent SDE} \label{Ex4}
Let $\mathbf{D}_{q,T}$ be the space of $q$-dimensional cadlag paths on $[0,T]$. For $w \in \mathbf{D}_{q,T}$, we define $w|_t:= w(s); 0\le s\le t$ and $w|_t= w(t); s >t$. Let us define $\Lambda :=\{(t,w|_t); t\ge 0, w \in \mathbf{D}_{q,T}\}$ equipped with the metric

$$d\big((t,w); (t',w') \big):= \| w|_t - w'|_{t'}\|_\infty + |t-t'|,$$
for $(t,w), (t',w') \in \Lambda$, where $\|\cdot\|_\infty$ is the sup norm. Let $X^u$ be a $q$-dimensional controlled SDE

\begin{equation}\label{pdsdeBM}
dX^u(t) = \alpha(t,X^u|_t,u(t))dt + \sigma(t,X^u|_t,u(t))dB(t),
\end{equation}
driven by the $d$-dimensional Brownian motion $B$ with a given initial condition $X^u(0)=x_0\in \mathbb{R}^q$. The coefficients $(\alpha, \sigma)$ are non-anticipative mappings defined on $\Lambda$ and they satisfy the following Lipschitz property: There exists a constant $K_{Lip}$ such that  

$$|\alpha(t,w|_t,a) - \alpha(t',w'|_{t'},b)| + |\sigma(t,w|_t,a) - \sigma(t',w'|_{t'},b)|\le K_{Lip} \Big\{ d\big((t,w); (t',w') \big)+ |a-b|\Big\},$$
for every $(t,w), (t',w') \in \Lambda$ and $a,b \in \mathbb{A}$. 

For a given payoff $\varphi:\mathbb{R}^q\rightarrow \mathbb{R}$, the stochastic control problem is 

$$\phi^\star \in \argmin_{\phi\in U^T_0}\mathbb{E}[\varphi (X^\phi(T))].$$
The imbedded discrete structure for $X^u$ is given by 
\begin{eqnarray}\label{deltaBM}
\nonumber\Delta X^{k,u}(T^k_{n})&=&\alpha\big(T^k_n, X^{k,\phi}|_{T^k_{n-1}},u_{n-1}\big)\Delta T^k_n\\
\nonumber& &\\
\nonumber&+& \sigma\big(T^k_n, X^{k,\phi}|_{T^k_{n-1}},u_{n-1}\big)\Delta A^k (T^k_n)\\
&=& \blacktriangle x_n \Big(\pi_2(\Xi^{k,u}_{n-1}),u_{n-1} , \Delta T^k_n , \ell_n(\mathcal{A}^k_n) \Big)
\end{eqnarray}
where 
$$
\ell_n(\mathcal{A}^k_n) = \Delta A^{k}(T^k_n),
$$
for $1\le n\le e(k,T)$. The lack of Markov property comes from $(\alpha,\sigma)$ and not of driving noise.

\subsubsection{SDEs driven by fractional Brownian motion with $\frac{1}{2} < H < 1$}\label{Ex3} 
Let $X^u$ be the controlled process 

\begin{equation}\label{limsdefbm}
dX^u(t) = \varrho(X^u(t),u(t))dt + \sigma dB_H(t),
\end{equation}
where $X(0)=x_0\in \mathbb{R}$, $\sigma$ is a constant and $\varrho:\mathbb{R}\times \mathbb{A}\rightarrow \mathbb{R}$ is Lipschitz in the sense that

$$|\varrho(x,c) - \varrho(y,c')|\le \| \varrho\| \{|x-y|+ |c-c'|\},$$
for every $x,y \in \mathbb{R}$ and $c,c' \in \mathbb{A}$. The driving noise $B_H$ is the fractional Brownian motion

$$B_H(t)= \int_0^t K(t,s)dB(s),$$
where $K$ is the kernel of the Riemann-Liouville fractional Brownian motion given by 

$$K(t,s) = \sqrt{2H}(t-s)^{H-\frac{1}{2}},$$
for $s < t$ and $\frac{1}{2} < H < 1$. For a given payoff $\xi:\mathbb{R}\rightarrow \mathbb{R}$, the stochastic control problem is 

$$\phi^\star \in \argmin_{\phi\in U^T_0}\mathbb{E}[\xi (X^\phi(T))].$$
The imbedded discrete structure for $X^u$ is given by 
\begin{eqnarray}\label{deltaFBM}
\nonumber\Delta X^{k,u}(T^k_{n})&=&\varrho\big(X^{k,\phi}(T^k_{n-1}),u_{n-1}\big)\Delta T^k_n\\
\nonumber&+& \sigma\Delta B^k_{H}(T^k_{n})\\
&=& \blacktriangle x_n \Big(\pi_2(\Xi^{k,u}_{n-1}),u_{n-1} , \Delta T^k_n , \ell_n(\mathcal{A}^k_n) \Big)
\end{eqnarray}
where 

$$\ell_{n}(\mathcal{A}^k_n)=\Delta B_H^k(T^k_n)$$
and 

\begin{equation}\label{discFBM}
B_H^k(T^k_n):= \int^{T^k_n}_0 \frac{\partial K }{\partial s}(T^k_n,s)A^k(s)ds,
\end{equation}
for $0\le n\le e(k,T)$. It is important to notice that $B^k_H(T^k_n)$ is a function of the whole path $\mathcal{A}^k_{n} = (\mathcal{W}^k_1, \ldots, \mathcal{W}^k_n)$ and not only of $\mathcal{W}^k_n$. See \cite{LEAO_OHASHI2017.2} for details. The lack of Markov property comes from the driving noise and not from the coefficients $(\varrho,\sigma)$.

\subsection{The dynamic programming algorithm}\label{DPtheory} 
Throughout this section, we are going to fix a controlled imbedded discrete structure $\mathcal{X}  = \big( (X^k)_{k\ge 1}, \mathscr{D} \big)$

\begin{equation}\label{strong0}
u\mapsto X^{k,u}
\end{equation}
satisfying Assumption (E1) and converging to an $\mathbb{F}$-adapted controlled process $\phi \mapsto X^\phi$
\begin{equation}\label{strong}
\sup_{\phi\in U^{k,e(k,T)}_0}\mathbb{E}\sup_{0\le t\le T}|X^{k,\phi}(t) - X^{\phi}(t)|\rightarrow 0,
\end{equation}
as $k\rightarrow +\infty$. 

\begin{remark}
The examples described in Section \ref{EXsection} satisfy conditions (\ref{strong0}) and (\ref{strong}) with explicit convergence rates. See \cite{leaoohashi} for further details.
\end{remark}


Starting with $\mathbb{V}^k_{e(k,T)}(\mathbf{o}^k_{e(k,T)}) := \varphi\big(x_0+ \sum_{i=1}^{e(k,T)} y^k_i\big)$, we set

\begin{eqnarray}
\nonumber\mathbf{U}^k_j(\mathbf{o}^k_j,\theta) &:=& \int_{\mathbb{W}}\mathbb{V}^k_{j+1}\Big(\mathbf{o}^k_{j}, \mathfrak{X}^k_{j+1}(\theta, \mathbf{o}^k_{j}, w)\Big)\nu^k(dw)\\
\label{valuefunc}\mathbb{V}^k_j(\mathbf{o}^k_j) &:=& \inf_{\theta\in \mathbb{A}}\mathbf{U}^k_j(\mathbf{o}^k_j,\theta),
\end{eqnarray}
for $\mathbf{o}^k_j = (w^k_1, y^k_1, \ldots, w^k_j,y^k_j)$ and $j=e(k,T)-1, \ldots, 0$. The transition kernel associated with $\mathcal{X}$ is 
\begin{equation}\label{fXc}
\mathfrak{X}^k_{j+1}(\theta, \mathbf{o}^k_{j}, w):=\Big(w, \blacktriangle x_{j+1}\big(\pi_2(\textbf{o}^k_j),\theta, s, \ell_{j+1}(\pi_3(\mathbf{o}^k_j),w)\big)\Big)\in \mathbb{W}_k\times \mathbb{R}^q
\end{equation}
for $j=e(k,T)-1, \ldots, 0$. The function 
\begin{eqnarray*}
\mathbf{o}^k_j \mapsto \mathbf{U}^k_j(\mathbf{o}^k_j,\theta) &=& \mathbb{E}\Big[ \mathbb{V}^k_{j+1}\Big(\mathbf{o}^k_{j}, \mathfrak{X}^k_{j+1}(\theta, \mathbf{o}^k_{j}, \mathcal{W}^k_1)\Big) \Big]
\end{eqnarray*}
is called the optimal state-action value function at step $j$ and $\mathbb{V}^k$ is the value function. Here, we recall $\mathcal{W}^k_1  \stackrel{d}{=}  (\Delta T^k_{1},\Delta A^k(T^k_{1}))$ with law $\nu^k$.

For a given $\eta>0$, there exists a universally  measurable function $C^\eta_{k,j}:\mathbb{H}^{j}_k\rightarrow\mathbb{A}$ such that

\begin{equation}\label{opcontrol}
\mathbb{V}^{k}_j(\mathbf{o}^k_{j})\ge \int_{\mathbb{W}_k}\mathbb{V}^{k}_{j+1}\Big(\mathbf{o}^{k}_{j}, \mathfrak{X}^k_{j+1}(C^{\eta}_{k,j}(\mathbf{o}^k_j), \mathbf{o}^k_{j}, w^k)\Big)\nu^k(dw^k) - \eta,
\end{equation}
for every $\mathbf{o}^k_{j}\in \{\mathbb{V}^{k}_j < + \infty\}$, where $j=e(k,T)-1,\ldots, 0$. In particular, if $\mathbb{H}^{j}_k = \{\mathbb{V}^k_j < +\infty\}$, for $j=e(k,T)-1, \ldots, 0$, then for $u\in U^{k,e(k,T)}_0$, we shall define the control $u^{k,\eta}_j$

\begin{equation}\label{explicitcontrolTEXT}
u^{k,\eta}_j := C^\eta_{k,j}(\Xi^{k,u}_j); j=e(k,T)-1, \ldots, 0
\end{equation}

\begin{remark}
For a given $\epsilon >0$, if we set $\eta_k = \frac{\epsilon}{e(k,T)}$ and $v^{k,\eta_k} = (v^{k,\eta_k}_0, \ldots, v^{k,\eta_k}_{e(k,T)-1})$, where each $v^{k,\eta_k}_j$ is constructed via (\ref{opcontrol}) and (\ref{explicitcontrolTEXT}) with $\eta= \eta_k$, then 

$$\mathbb{E}\Big[\varphi \Big(X^{k,v^{k,\eta_k}}(T^k_{e(k,T)})\Big)\Big] < \inf_{u \in U^{k,e(k,T)}_0}\mathbb{E}\Big[\varphi \Big(X^{k,u}(T^k_{e(k,T)})\Big)\Big]+\epsilon.$$
\end{remark}

\begin{theorem}[Theorem 4.2 of \cite{leaoohashi}]\label{thp}
Let $u\mapsto X^{k,u}$ be an imbedded discrete structure associated with an $\mathbb{F}$-adapted controlled process $X$ as described in (\ref{strong0}) and (\ref{strong}). For a given $\epsilon>0$, let $u^{k,\epsilon} = (u^{k,\epsilon}_0, \ldots, u^{k,\epsilon}_{e(k,T)-1})$ be a near optimal control realizing

$$\mathbb{E}\Big[\varphi \Big(X^{k,u^{k,\epsilon}}(T^k_{e(k,T)})\Big)\Big] < \inf_{u \in U^{k,e(k,T)}_0}\mathbb{E}\Big[\varphi \Big(X^{k,u}(T^k_{e(k,T)})\Big)\Big]+\frac{\epsilon}{3}.$$

Then, $u^{k,\epsilon} \in U^T_0$ is near optimal w.r.t. the original stochastic control problem driven by the Brownian motion

$$\inf_{u \in U^{T}_0} \mathbb{E}[\varphi (X^{u}(T))]+\epsilon > \mathbb{E}[\varphi (X^{u^{k,\epsilon}}(T))]$$
as $k\rightarrow +\infty $.


\end{theorem}
For explicit rates of convergence, we refer the reader to \cite{leaoohashi}.

\section{Construction of importance sampling weights and training data}\label{MSECTION}
This section presents explicit expressions for the importance sampling weights and training data associated with the examples described in section \ref{EXsection}. The importance sampling weights will play a major role in updating our Monte Carlo numerical scheme under parametric model risk. We will fix $k\ge 1$ once and for all. For this reason, in order to shorten notation, we omit the dependence on $k$ in the hitting times (\ref{stopping_times}), in the driving noise (\ref{akprocess}), in the information sets (\ref{caliak}) and (\ref{caliwk}) and in the value functions (\ref{valuefunc}). For the remainder of this paper, we set 

$$m=e(k,T).$$
For any stepwise constant process $Z$ jumping at the hitting times (\ref{stopping_times}), we denote

$$Z_n:= Z(T^k_n),\quad \Delta Z_n:= Z_n-Z_{n-1},\quad \varepsilon= \varepsilon_k,$$ 
for simplicity. A fixed history of the controlled state (\ref{Xioperator}) will be denoted by $\Xi_n = \mathbf{o}_n$ for $0\le n\le m$. In order to estimate $(\mathbb{V}_j)_{j=0}^{m-1}$, we will make use of a training data denoted by

\begin{equation}\label{Ondef}
O_n= \big(\mathcal{W}_1, Y_1, \ldots, \mathcal{W}_n, Y_n\big); 1\le n\le m.
\end{equation}

Recall that the $\ell$-fold Cartesian product is written as  

$$\mathbb{H}^\ell = \big(\mathbb{W}\times \mathbb{R}^q\big)^\ell; \ell\ge 1. $$
Recall the elements of $\mathbb{H}^{\ell}$ are denoted by

$$
\textbf{o}_\ell = (w_1,y_1, \ldots, w_\ell, y_\ell).
$$
For convenience, we set $\mathbf{o}_0:=(0,0,x_0)$ and $\mathbb{H}^0:= \{(0,0,x_0)\}$, where the initial condition $x_0$ is fixed once and for all.

\

\noindent \textbf{Assumption H0:} We will assume the pair $(\mathcal{W}_i, Y_i)_{i=1}^{m}$ is i.i.d and generated by a product probability measure $\nu\otimes \mu$ on $\mathbb{H}=\mathbb{W}\times\mathbb{R}^q$, where $\nu$ is the law of $(\Delta T_1, \Delta A_1)$. The sequence $\{Y_i; i\ge 1\}$ is iid generated by a measure $\mu$ and $Y_i$ is independent of $\mathcal{W}_i$ for every $i\ge 1$.

\

For a probability measure $\mu$ in $\mathbb{R}^q$ to be an admissible training measure in Assumption H0, we will impose that it dominates the controlled dynamics in the following sense.

\

\noindent \textbf{Assumption H1}:  We assume there exists a transition function $r_j: \mathbb{A}\times \mathbb{R}^q \times \mathbb{H}^{j-1} \rightarrow \mathbb{R}_+$ such that

\


\begin{eqnarray}
\nonumber\mathbb{P}[ (\mathcal{W}_j, \Delta X^a_j) \in dxdx' | \Xi_{j-1}= b]&=& r_j(x,a,x'; b)\mu(dx')\nu(dx)\\
\label{rnr}&=& r_j(a,x'; b)\mu(dx')\nu(dx)
\end{eqnarray}
where

\begin{equation}\label{rinfty}
\|r\|_\infty:=\max_{1\le j\le m}\sup_{x' \in \mathbb{R}^{q}, a \in \mathbb{A}, b \in \mathbb{H}^{j-1}}|r_j(a,x'; b)| < \infty.
\end{equation}
Moreover, there exists $\|r\|$ such that

\begin{equation}\label{rlip}
|r_j(a,x; b) - r_j(a',x; b')|\le \|r\|\{ |a-a'| + |b-b'|\}
\end{equation}
for every $x \in \text{supp}\mu$, $a,a' \in \mathbb{A}$, $b,b' \in \mathbb{H}^{j-1}$ and $1\le j\le m$.  

\begin{remark}
Observe that $r_j$ does not depend on $x \in \mathbb{W}$ in (\ref{rnr}). Moreover, H1 is an intrinsic property associated with the the controlled state and not with the Neural Network architecture.     
\end{remark}

We also assume Lipschitz property of the increment of the controlled process in the sense of the following hypothesis:

\

\noindent \textbf{Assumption H2}: There exists a constant $C(e)$ such that 


\begin{equation}\label{rho_M}
\Big|\blacktriangle x_j \Big(\pi_2(\mathbf{o}_{j-1}),a , s , e \Big) - \blacktriangle x_j \Big( \pi_2(\mathbf{o}'_{j-1}),a' , s , e  \Big)\Big|\lesssim_T C(e)\big\{|a-a'| + |\pi_2(\mathbf{o}_{j-1}) - \pi_2(\mathbf{o}'_{j-1})|\big\}, 
\end{equation}
for every $e\in \mathbb{R}^d$, $a,a' \in \mathbb{A}, s \in (0,T]$ and $\mathbf{o}_{j-1} \in \mathbb{H}^{j-1}$ for $j=m, \ldots, 1$.  
\begin{remark}
In typical examples (see sections \ref{Ex2} and \ref{Ex4}), we have  
\begin{equation}\label{tyC}
C(e)= \| q_1\| + \|q_2\| | e|_{\mathbb{R}^d},
\end{equation}
where $\|q_1\|$ and $\|q_2\|$ are Lipschitz constants from a controlled SDE and 

$$e\stackrel{d}{=} (\Delta A^{1}_1, \ldots, \Delta A^{d}_1).$$
\end{remark}

Let 
$$\rho_M:= \max_{1\le p\le M}C\Big(  (\Delta A^{1}_1, \ldots, \Delta A^{d}_1)_{p}\Big),$$
where $\{ (\Delta A^{1}_1, \ldots, \Delta A^{d}_1)_{p};1\le p\le M\}$ is iid and $C$ is given by (\ref{tyC}).  For sake of simplicity, we will adopt (\ref{tyC}) in the remainder of this paper.  

\begin{remark}
Since $|\cdot|_{\max}\le |\cdot|_{\mathbb{R}^d}\le \sqrt{d}|\cdot|_{\max}$, then 

\begin{equation}\label{growthrho}
|\rho_M|\le \|q_1\| + \|q_2\|\sqrt{d}\varepsilon,
\end{equation}
for every $M\ge 1$. The fact that one can bound $|\rho_M|$ by a constant independent of $M$ is due to the very particular structure of our imbedding scheme, where the increments $\{ (\Delta A^{1}_1, \ldots, \Delta A^{d}_1)_p; 1\le p\le M\}$ are bounded by $\epsilon$ uniformly w.r.t. the number of Monte Carlo samples $M$. This particular property will have a direct impact on the convergence rate of Theorems \ref{mainresult} and \ref{mainresultRS}. See Lemma \ref{leest4}.     
\end{remark}



\begin{remark}
For the analysis of the regression-based backward scheme, we generate training inputs under a product measure so that are i.i.d. This assumption is purely algorithmic and concerns only the training distribution. The i.i.d. is imposed only for the synthetic training measure used in the Monte Carlo regression; the true controlled dynamics remain fully non-Markovian and path-dependent.

\end{remark}


\subsection{Dynamic Programming for Randomized Strategies}\label{rssection}

We now discuss a randomized version of the dynamic programming equation (\ref{valuefunc}) written in terms of generic classes of probability measures on the action space rather than deterministic (pure) policies. This point of view will be important to treat examples where dominating training measures $\mu \in \mathbb{R}^q$ realizing (\ref{rnr}) and (\ref{rinfty}) in Assumption H1 are \textit{not} available. This will be the case fort the partial hedging control problem for rough stochastic volatility models (see Proposition \ref{nonePCONE}).

Let $\mathcal{P}(E)$ be the set of all probability measure on a metric space $E$ equipped with the Borel sigma-algebra. 

\begin{remark}
One can easily check 
$$\mathbb{V}_j(\mathbf{o}_j) = \inf_{\varphi_j(\cdot| \mathbf{o}_j)}\int_{\mathbb{A}} \int_{\mathbb{H}} \mathbb{V}_{j+1}(\mathbf{o}_j, x, x')\mathbb{P}[(\mathcal{W}_{j+1},\Delta X^a_{j+1}) \in dxdx'| \Xi_j=\mathbf{o}_j]\varphi_j(da|\mathbf{o}_j),$$
for $j=m-1, \ldots, 0$, where the infimum is taking w.r.t. all probability kernels $\varphi_j: \mathbb{H}^j\rightarrow \mathcal{P}(\mathbb{A})$. 
\end{remark}

Let $\mathcal{H}_j$ be a set of density functions (over the action space $\mathbb{A}$) $h_j: \mathbb{H}^j\times \mathbb{A}\rightarrow \mathbb{R}_+$; for $0\le j\le {m-1}$ such that there exists a constant $C(\mathcal{H})$ such that 

\begin{equation}\label{rsetd1}
\max_{0\le j\le m-1}\sup_{h_j \in \mathcal{H}_j}\sup_{b \in \mathbb{H}^j}\| h_{j}(b)\|_{L^\infty(\mathbb{A})}\le C(\mathcal{H}) < \infty. 
\end{equation}
We set 

$$\mathcal{H} := \mathcal{H}_0\times \ldots \times \mathcal{H}_{m-1}.$$


For each $h_j: \mathbb{H}^j\times \mathbb{A}\rightarrow \mathbb{R}_+ \in \mathcal{H}_j$, we define a probability kernel $\pi_j: \mathbb{H}^j\rightarrow \mathcal{P}(\mathbb{A})$

\begin{equation}\label{rsetd0}
\pi_j(da\mid \mathbf{o}_j) = h_{j}(\mathbf{o}_j, a)\,da, 
\end{equation}
for $\mathbf{o}_j \in \mathbb{H}^j$, $0\le j\le {m-1}$.

To keep notation simple, we set 

\begin{eqnarray}
\nonumber K_j(\mathbf{o}_j,a;dxdx') &:=& \mathbb{P}\big[ (\mathcal{W}_{j+1},\Delta X^{a}_{j+1}) \in dxdx' | \Xi_j = \mathbf{o}_j \big]\\
\label{Ks}&=& \mathbb{P}\big[\Delta X^{a}_{j+1}\in dx' | \Xi_j = \mathbf{o}_j \big]\nu(dx),
\end{eqnarray}
for $j=m-1, \ldots, 0$.

The dynamic programming equation for the set of randomized strategies $\mathcal{H} = \mathcal{H}_0\times \ldots \times \mathcal{H}_{m-1}$ 
described in (\ref{rsetd0}) is given as follows: Starting with $\mathbb{V}^{\mathcal H}_{m}(\mathbf{o}_{m})
:= \varphi\big(x_0+\sum_{i=1}^{m} y_i\big)$, we set

\begin{align}\label{randomDP}
\nonumber \mathbf{U}^\mathcal{H}_j(\mathbf{o}_j,\pi_j)
&:= \int_{\mathbb{A}} \int_{\mathbb{H}}
   \mathbb{V}^{\mathcal H}_{j+1}(\mathbf{o}_j, x, x')\,
   K_j(\mathbf{o}_j,a;dxdx')\,
   \pi_j(d a\mid \mathbf{o}_j)\\
\mathbb{V}^{\mathcal H}_j(\mathbf{o}_j)
&:= \inf_{\pi_j\in \mathcal{H}_j} \mathbf{U}^\mathcal{H}_j(\mathbf{o}_j,\pi_j)
\end{align}
for $j=m-1, \ldots, 0$.

\begin{remark}\label{dualrepU}
By using (\ref{fXc}), applying Fubini and change of variables formula, we can write  

\begin{eqnarray*}
\mathbf{U}^\mathcal{H}_j(\mathbf{o}_j,\pi_j) &=&  \mathbb{E}\Big \langle \mathbb{V}^{\mathcal H}_{j+1}(\mathbf{o}_j, \mathfrak{X}_{j+1}(\cdot,\mathbf{o}_j,\mathcal{W}_1)), \pi_j(\cdot| \mathbf{o}_j)\Big\rangle,
\end{eqnarray*}
for $j=m-1,\ldots, 0$, where the bracket 

\begin{equation}\label{bracket}
\big \langle g, \kappa   \big \rangle:= \int_{\mathbb{A}} g(a)\kappa(da),
\end{equation}
for $\kappa \in \mathcal{P}(\mathbb{A})$ and $g \in L^1(\kappa)$.
\end{remark}

We now discuss the structure associated with the dynamic programming equation (\ref{randomDP}). The following lemma is elementary, so we omit the proof for sake of conciseness.   
\begin{lemma}
For each sequence of probability kernels $\{\pi_j; 0\le j\le {m-1}\}$ of the form (\ref{rsetd0}) and an initial condition $\mathbf{o}_0=(0,0,x_0)$, we can associate a unique probability measure $\mathbb{P}^{\pi}_{\mathbf{o}_0} \in \mathcal{P}(\mathbb{H}^0\times (\mathbb{A}\times \mathbb{H})^{m})$ with representation

$$\mathbb{P}^\pi_{\mathbf{o}_0}(dz) = \delta_{x_0}(d\mathbf{o}_0) \prod_{j=0}^{m-1}K_j(\mathbf{o}_j,a_j;dw_{j+1}dy_{j+1})
    \pi_j(da_j\mid \mathbf{o}_j) $$ 
where we denote $z_{m} = (\mathbf{o}_0,a_0,w_1,y_1, a_1,w_2, y_2,\dots,a_{m-1},w_{m},y_{m})$. Reciprocally, any probability $\mu \in \mathcal{P}(\mathbb{H}^0\times (\mathbb{A}\times \mathbb{H})^{m})$ with transition probabilities inherited from (\ref{Ks}) is uniquely determined by a sequence of probability kernels from $\mathbb{H}^j$ to $\mathcal{P}(\mathbb{A})$, for $j=0,\ldots, m-1$.



\end{lemma} 

For a given initial condition, $\mathbf{o}_0 = (0,0,x_0)\in \mathbb{H}^0$ and $\pi=\{\pi_j; 0\le j\le m-1\} \in \mathcal{H}$, we set

$$\mathbb{E}^\pi_{\mathbf{o}_0}[\varphi (X_m)]:= \mathbb{E}_{\mathbb{P}^\pi_{\mathbf{o}_0}}[\varphi (F)] = \int_{\mathbb{H}^0\times (\mathbb{A}\times \mathbb{H})^m} \varphi(F(z_{m}))\mathbb{P}^\pi_{\mathbf{o}_0}(dz_{m})$$
where the functional $F: \mathbb{H}^0\times (\mathbb{A}\times \mathbb{H})^m\rightarrow \mathbb{R}^n$ is defined as follows:
$$F(z_m) = x_0+\sum_{j=1}^m \blacktriangle x_j \big(\pi_2(\mathbf{o}_{j-1}),a_{j-1}, s_j, \ell_j(\pi_3(\mathbf{o}_{j-1}), w_j)\big)$$ 
where $z_m = (\mathbf{o}_0,a_0,w_1,y_1, a_1,w_2, y_2,\dots,a_{m-1},w_m, y_m)$, $\mathbf{o}_{j-1} = (w_1, y_1, \ldots, w_{j-1}, y_{j-1})$ for $2\le j\le m$.
The control problem associated with $\mathcal{H}$ is 

\begin{equation}\label{discHp}
\pi^\star \in \argmin_{\pi \in \mathcal{H}}\mathbb{E}^{\pi}_{\mathbf{o}_0}\big[ \varphi(X_m) \big].
\end{equation}

\begin{proposition}
For each $\mathbf{o}_j\in\mathbb{H}^j$, assume the infimum in 
(\ref{randomDP}) is attained by some kernel $\pi_j^\star(\cdot\mid \mathbf{o}_j)\in \mathcal{H}_j$, for $0\le j\le m-1$. Let 
      \[
      \pi^\star := (\pi_0^\star,\dots,\pi_{m-1}^\star) \in \mathcal{H}.
      \]
      Then $\pi^\star$ is optimal in the restricted class $\mathcal{H}$. That is, 
      
$$\mathbb{E}^{\pi^\star}_{\mathbf{o}_0}\bigl[\varphi(X_m)\bigr] = \inf_{\pi \in \mathcal{H}}\mathbb{E}^{\pi}_{\mathbf{o}_0}[\varphi(X_m)].$$
\end{proposition}
\begin{proof}
Fix $j$ and $\mathbf{o}_j \in \mathbb{H}^j$. Define a minimizer by $\pi_j^\star(\cdot\mid \mathbf{o}_j)$. Let $\pi^\star$ be the policy composed of the minimizers $\{\pi_j^\star\}$. By construction, 
\[
\mathbb{V}^\mathcal{H}_j(\mathbf{o}_j)
=
\int_{\mathbb A}\int_{\mathbb H}
      \mathbb{V}^\mathcal{H}_{j+1}(\mathbf{o}_j,w_{j+1},y_{j+1})
      K_j(\mathbf{o}_j,a;)\,\pi_j^\star(da\mid \mathbf{o}_j).
\]
We now claim for each $j=0,1,\dots,m$ and each $\mathbf{o}_j\in\mathbb{H}^j$, we have
\begin{equation}\label{rop1}
\mathbb{V}_j^{\mathcal{H}}(\mathbf{o}_j)
= \mathbb{E}^{\pi^\star}_{\mathbf{o}_0}\bigl[\,\varphi(X_m)\mid O_j=\mathbf{o}_j\,\bigr],
\end{equation}
and therefore, in particular,
\begin{equation}\label{rop2}
\mathbb{V}_0^{\mathcal{H}}(\mathbf{o}_0)
= \mathbb{E}^{\pi^\star}_{\mathbf{o}_0}\bigl[\varphi(X_m)\bigr] = \inf_{\pi \in \mathcal{H}}\mathbb{E}^{\pi}_{\mathbf{o}_0}[\varphi(X_m)]
\end{equation}
Indeed, assertion (\ref{rop1}) follows from the tower property and a simple induction starting with $\mathbb{V}^\mathcal{H}_m(\mathbf{o}_m) = \varphi (x_0 + \sum_{j=1}^m y_j)$. In particular, $\mathbb{V}_0^{\mathcal{H}}(\mathbf{o}_0)
= \mathbb{E}^{\pi^\star}_{\mathbf{o}_0}\bigl[\varphi(X_m)\bigr]$. The fact $\mathbb{E}^{\pi^\star}_{\mathbf{o}_0}\bigl[\varphi(X_m)\bigr] = \inf_{\pi \in \mathcal{H}}\mathbb{E}^{\pi}_{\mathbf{o}_0}[\varphi(X_m)]$ is a consequence of the dynamic programming equation (\ref{randomDP}).  

\end{proof}
\begin{remark}
In general, $\mathcal{H}$ is not convex and, in practice, near optimal controls for (\ref{discHp}) are enough.    
\end{remark}

We are now in position to present the analogous assumption H1 in the context of randomized strategies. For each randomized strategy $\pi_j:\mathbb{H}^j\rightarrow \mathcal{P}(\mathbb{A}) \in \mathcal{H}_j$, we set 

$$\mu^{\pi_j}_{j}(dx| \mathbf{o}_j):=\int_{\mathbb{A}} \mathbb{P}[\Delta X^a_{j+1} \in dx| \Xi_j=\mathbf{o}_j]\pi_j(da|\mathbf{o}_j)$$
for $j=m-1,\ldots, 0$,

\begin{remark}\
$$\mathbf{U}^\mathcal{H}_{j}(\mathbf{o}_j,\pi_j) = \int_{\mathbb{W}}\int_{\mathbb{H}} \mathbb{V}^{\mathcal{H}}_{j+1}(\mathbf{o}_j,x,x')\mu_j^{\pi_j}(dx'|\mathbf{o}_j)\nu(dx),$$
for $\pi_j \in \mathcal{H}_j, \mathbf{o}_j \in \mathbb{H}^j$, $j=m-1, \ldots, 0$
\end{remark}

\

\noindent \textbf{Assumption (R1):} There exists a dominating measure $\mu \in \mathcal{P}(\mathbb{R}^q)$ such that 

$$\mu^{\pi_j}_j(dx|\mathbf{o}_j) << \mu(dx)$$
for every $\pi_j \in \mathcal{H}_j, \mathbf{o}_j \in \mathbb{H}^j; j=0, \ldots, m-1$. Moreover, the associated Radon-Nikodym derivatives $\rho^{\pi_j}_j: \mathbb{H}^j\times \mathbb{R}^q\rightarrow \mathbb{R}_+$ satisfies 

\begin{equation}\label{rhoinfty}
\|\rho\|_\infty:=\max_{0\le j\le m-1}\sup_{\pi_j \in \mathcal{H}_j}\sup_{x \in \text{supp}\mu, b \in \mathbb{H}^{j}}|\rho^{\pi_j}_j(b, x)| < \infty,
\end{equation}
and
\begin{equation}\label{rholip}
|\rho^{\pi_j}_j(b, x) - \rho^{\pi_j}_j(b', x) |\le \|\rho\| |b-b'|, 
\end{equation}
for every $b,b' \in \mathbb{H}^j, x \in \text{supp}\mu,\pi_j,\pi'_j \in \mathcal{H}_j$ with $j=m-1, \ldots, 0$.

\begin{remark}
The hedging task is a typical example of control problem where assumption H1 is not fulfilled but R1 does. See Proposition \ref{nonePCONE}. Observe assumption R1 is associated not only w.r.t. the controlled state but also to the class $\mathcal{H}$.    
\end{remark}

\subsection{Construction of the importance sampling weights and training measures for SDEs driven by Brownian motion}\label{trainingEX}
This section presents a detailed description of the samples associated with the training stage of the value functions in a given stochastic control problem. 
As a warming-up to the more complex case of SDEs driven by fractional Brownian motion and partial hedging for rough stochastic volatility models, we start with the simplest case of path-dependent SDEs driven by the Brownian motion. For sake of simplicity of exposition, we will study the assumptions H1-H2 for the controlled SDE where the driving noise is the one-dimensional Brownian motion. In the multi-dimensional case, the driving noise $(\Delta A^1_1, \ldots, \Delta A^d_1)$ has a more complex distribution and the argument must be split into two parts by conditioning on the component of the $d$-dimensional Brownian motion which has hit the barrier $[-\epsilon,\epsilon]$. We refer the reader to Section \ref{incCASE} for details.   

In the sequel, for a given history $\Xi_{n-1} = (w_0,y_0, \ldots, w_{n-1},y_{n-1})$, we denote 

$$\mathbf{y}_{n-1}:= \sum_{\ell=1}^{n-1}y_\ell,$$ 
for $n\ge 1$ and a fixed initial condition $y_0=x_0 \in \mathbb{R}$. Throughout this section, we denote $J\stackrel{(d)}{=}\Delta T_1$ with density $f_J$, $\mathbf{B}$ is a symmetric Bernoulli random variable taking values $\pm \epsilon$. We also set 

\begin{equation}\label{Rfield}
R_{c,v} := cJ
       + v\mathbf B,
\end{equation}
for $(c,v) \in \mathbb{R}^2$. We also denote (see (\ref{densR}))

$$\mathcal{R}(c,v;\cdot):=\text{density function of}~R_{c,v}.$$
for $(c,v) \in \mathbb{R}^2$ with $c\neq 0, v \neq 0$.  


\begin{remark}
Since $\Delta T_1$ is random and its sharpest lower bound is zero, we need to impose a strictly positive lower bound $\underline{M}>0$. Otherwise, we shall incur in a degenerated configuration with non-zero probability. In practice, this is not a restriction and indeed it prevents numerical instability in the Monte Carlo scheme.       
\end{remark}

We will assume the following set of conditions on the coefficients of the SDE (\ref{pdsdeBM})

\

\begin{itemize}

\item[(B1)] (\emph{Lipschitz property}) The coefficients $(\alpha,\sigma)$ are globally Lipschitz. That is, there exists a constant $K_{Lip}$ such that  

$$|\alpha(t,w,a) - \alpha(t',w',b)| + |\sigma(t,w,a) - \sigma(t',w',b)|\le K_{Lip} \Big\{ d\big((t,w); (t',w') \big)+ |a-b|\Big\},$$
for every $(t,w), (t',w') \in \Lambda$ and $a,a' \in \mathbb{A}$. 

\item[(B2)](\emph{Compactness}) There exist $0<c_{\min} <  c_{\max}<\infty$ and $v_{\max} $such that
$$
c_{\min}\ \le\ |\,\alpha(t,w,a)\,|\ \le\ c_{\max}\quad \text{and}\quad |\sigma(t,w,a)|\le v_{\max},
$$
for every $(t,w,a)\in\Lambda\times\mathbb A$.

\item[(B3)] (\emph{Truncation}) In the control step we use $J$ with $ 0 <\underline{M} \le J\le \overline{M}$ a.s.\ for some $0 <\underline{M} < \overline{M} < \infty$. Equivalently, the law of $J$ is the truncation of $\Delta T_1$ to $[\underline{M},\overline{M}]$ with density $f_J$.  
\end{itemize}

\

In the sequel, for a given history $\Xi_{n-1} =\mathbf{o}_{n-1}= (w_0,y_0, \ldots, w_{n-1},y_{n-1})$ with $\Xi_0= \mathbf{o}_0=(0,0,x_0)$ and $s_0=0$, we denote 

$$t_{n-1}:= \sum_{\ell=0}^{n-1}s_\ell,$$ 
for $n\ge 1$ and a fixed initial condition $y_0=x_0 \in \mathbb{R}$. The dynamics of $\Delta X^a_n$ conditioned on $\Xi_{n-1} = \mathbf{o}_{n-1}=(w_0, y_0, \ldots, w_{n-1},y_{n-1})$ is given by   
\begin{equation}\label{deltaXaSDE}
 \Delta X^a_n = \alpha\big(t_{n-1}, \big\{y_{i}\}_{i=0}^{n-1}, a\big)J +\sigma\big(t_{n-1}, \{y_i\}_{i=0}^{n-1},a\big) \textbf{B},
\end{equation}
where $J$ is independent of $(\textbf{B},\Xi_{n-1})$, for $n\ge 1$. Observe that with a slight abuse of notation, we write $\alpha(t_{n-1}, \big\{y_{i}\}_{i=0}^{n-1}, a\big)$ as the path-dependent non-anticipative coefficient $\alpha(t_{n-1}, \gamma_{n-1}(\pi_2(\mathbf{o}_{n-1})), a)$, where 

$$\pi_2(\mathbf{o}_{n-1}) = (s_0, y_0,\ldots, s_{n-1},y_{n-1}),$$
and $\gamma_{n-1}(\pi_2(\mathbf{o}_{n-1})) \in \mathbf{D}_{1,T}$ is a stepwise cadlag path given by the constant function $x_0$ for $n=1$ and 

$$\gamma_{n-1}(\pi_2(\mathbf{o}_{n-1})) = x_0 + \sum_{i=1}^{n-1} y_i 1\!\!1_{ \{t_i \leq \cdot \} }, $$
for $n\ge 2.$
Similar discussion for the non-anticipative kernel $\sigma$. Observe that the lack of Markov property comes from $(\alpha,\sigma)$ rather than an extrinsic noise. In this case, the value function will be a function only of 

$$\mathbb{V}_n(\mathbf{o}_n) = \mathbb{V}_n(s_1, y_1, \ldots, s_{n},y_{n}),$$
for $n=m, \ldots, 0$. Therefore, optimal controls associated with (\ref{deltaXaSDE}) will be of feedback-type, i.e., they depend only on the variables $(s_1, y_1, \ldots, s_n,y_n)$. See Proposition 4.1 in \cite{leaoohashi}. 

In the sequel, we consider the random field (\ref{Rfield}) restricted to a bounded away from zero drift and bounded volatility, i.e.,  
$$
R_{c,v} = cJ
       + v\mathbf B.
$$
for $0 < c_{\min} \le |c|\le c_{\max}$ and $|v|\le v_{\max}$. 


\begin{remark}\label{rangeR}
Observe that 

$$\text{Range}~R_{c,v} \subset \Big[-v_{\max} \varepsilon - c_{\max}\overline{M}, v_{\max} \varepsilon + c_{\max} \overline{M}\Big]~a.s.,$$
whenever $c_{\min} \le|c| \le c_{\max}$ and $|v|\le v_{\max}$.  
\end{remark}
In what follows, for simplicity of notation, we set 
\begin{equation}\label{compactKPDSDE}
K= \Big[-v_{\max} \varepsilon - c_{\max}\overline{M}, v_{\max} \varepsilon + c_{\max} \overline{M}\Big].
\end{equation}
Let $\mathcal{M}_K$ be the set of probability densities $q$ with support on $K$ such that the following condition is fulfilled: There exists a positive finite constant $Q_K$ such that
$$\inf_{x\in K} q(x) \ge Q_K >0.$$

In order to present the main result of this section, we need the following technical result. 
\begin{lemma}\label{fdeltab}
Let $\tau=\inf\{t>0:\,|W_t|=1\}$ for a standard Brownian motion $W$, and let 
\[
f_\Delta(t)=\sum_{k\in\mathbb Z}(-1)^k\,g_{y_k}(t),
\qquad
y_k:=2k+1,\quad
g_y(t):=\frac{|y|}{\sqrt{2\pi}\,t^{3/2}}\,e^{-y^2/(2t)},\quad t>0,
\]
be the density of $\tau$. Then, for any strictly positive constant $C$
\[
\sup_{t\ge C}\,|f'_\Delta(t)|\;<\;\infty,
\]
and
\[
f'_\Delta(t)\;=\;\sum_{k\in\mathbb Z}(-1)^k\,g'_{y_k}(t)
\qquad \text{for all }t\ge C,
\]
with the series converging uniformly on $[C,\infty)$.
\end{lemma}
The proof of Lemma \ref{fdeltab} is postponed to Section \ref{appendix4}.

\begin{theorem}\label{pdsdeth}
Suppose (B1-B2-B3) are fulfilled.  Then, for a given history $\mathbf{o}_{j-1}=(w_1, y_1, \ldots, w_{j-1},y_{j-1})$ and a control value $a$, the conditional law of $(\mathcal{W}_j, \Delta X^a_j)$ is given by   

$$
\mathbb{P}[ (\mathcal{W}_j, \Delta X^a_j) \in dx'dx | \Xi_{j-1}= \mathbf{o}_{j-1}]= r_j(a,x; \mathbf{o}_{j-1})\mu(dx)\nu(dx'),
$$
where $\mu \in \mathcal{M}_K$ with density $q$ and 

$$r_j(a,x;\mathbf{o}_{j-1})= \frac{\mathcal{R}(c(\mathbf{o}_{j-1},a),v(\mathbf{o}_{j-1},a);x)}{q(x)},$$
for $v(\mathbf{o}_{j-1},a) = \sigma (t_{j-1}, \{y_i\}_{i=0}^{j-1},a)$ and $c(\mathbf{o}_{j-1},a) = \alpha (t_{j-1}, \{y_i\}_{i=0}^{j-1},a)$. Moreover, assumptions H1-H2 are fulfilled, where (\ref{rinfty}), (\ref{rlip}) and (\ref{growthrho}) are satisfied. 
\end{theorem}
\begin{proof}
Let $R_{c,v} = cJ+ v \mathbf{B}$. Since $u \mapsto cu + sv \varepsilon$ is monotone (for $c \neq 0, s=\pm 1$), by the Jacobian method, $R_{c,v}$ has the density 

\begin{equation}\label{densR}
\mathcal{R}(c,v;x) = \frac{1}{2 |c|} \sum_{s = \pm 1}f_J  \Big( \frac{x-sv\varepsilon}{c}\Big) \mathds{1}_{\{\frac{x-sv\varepsilon}{c} \in [\underline{M},\overline{M}]  \}},
\end{equation}
where $f_J$ is the density of $J$. Moreover, 
\begin{equation}\label{lev1pd}
\sup_{x\in \mathbb{R}} \mathcal{R}(c,v;x)\le \frac{\| f_J \|_{\infty}}{|c|}\le \frac{\| f_J \|_{\infty}}{c_{\min}},  
\end{equation}
uniformly in $(c,v)$ such that $c_{\min} \le |c| \le c_{\max}$ and $|v|\le v_{\max}$. The law of $\Delta X^a$ given the history $\Xi_{n-1} = \mathbf
{o}_{n-1} = (s_1, y_1, \ldots, s_{n-1},y_{n-1})$ is given by $\mathcal{R}(c,v;\cdot)$ whose support is the compact set $K$, where $v=v(\mathbf{o}_{n-1},a) = \sigma (t_{n-1}, \{y_i\}_{i=0}^{n-1},a)$ and $c=c(\mathbf{o}_{n-1},a) = \alpha (t_{n-1}, \{y_i\}_{i=0}^{n-1},a)$. Now, take $q \in \mathcal{M}_K$ and we set 

\begin{equation}\label{densR1}
r_j(a,x,\mathbf{o}_{j-1}) = \frac{\mathcal{R}(c,v;x)}{q(x)}.
\end{equation}
By assumptions (B1,B2,B3) and (\ref{lev1pd}), we have  

\begin{equation}\label{densR2}
|r_j(a,x,\mathbf{o}_{j-1})|\le \frac{\| f_J\|_\infty}{Q_K c_{\min}},
\end{equation}
uniformly in $x \in \mathbb{R}$, $\mathbf{o}_{j-1} \in ([\underline{M},\overline{M}] \times \mathbb{R})^{j-1}$ and $j \in \{1, \ldots, m\}$. 
By using the Lipschitz property of $(\alpha,\sigma)$  (assumption (B1)), we just need to check the existence of a constant $C$ such that  

$$\sup_{x\in \mathbb{R}} \Bigg|\frac{\partial \mathcal{R}(c,v;x)}{\partial v}\Bigg| + \sup_{x\in \mathbb{R}} \Bigg|\frac{\partial \mathcal{R}(c,v;x)}{\partial c}\Bigg| \le C, $$
for every $(c,v) \in \{(r_1,r_2); |r_2|\le v_{\max}, c_{\min}\le |r_1| \le c_{\max}\}$. This will ensure the existence of the constant $\|r\|$ in (\ref{rlip}) given in Assumption (H1). By using Lemma \ref{fdeltab}, we observe

$$\frac{\partial \mathcal{R}(c,v;x)}{\partial v} = \frac{1}{2 |c|} \sum_{s = \pm 1}f'_J  \Big( \frac{x-sv\varepsilon}{c}\Big) \frac{-s\varepsilon}{c} \mathds{1}_{\{\underline{M}\le \frac{x-sv\varepsilon}{c} \le \overline{M} \}} $$
and

\begin{eqnarray*}
\frac{\partial \mathcal{R}(c,v;x)}{\partial c} &=& \frac{-1}{2 c^2} \sum_{s = \pm 1}f_J  \Big( \frac{x-sv\varepsilon}{c}\Big) \mathds{1}_{\{\underline{M}\le \frac{x-sv\varepsilon_k}{c} \le \overline{M}\}}\\ 
&+& \frac{1}{2c} \sum_{s=\pm 1} f'_J  \Big( \frac{x-sv\varepsilon}{c}\Big) \frac{-(x-sv)\varepsilon}{c^2} \mathds{1}_{\{\underline{M}\le \frac{x-sv\varepsilon}{c} \le \overline{M} \}}.
\end{eqnarray*}
Then, 

\begin{equation}\label{LipRden}
\sup_{x \in \mathbb{R}}\Bigg|\frac{\partial \mathcal{R}(c,v;x)}{\partial v}\Bigg| + \sup_{x \in \mathbb{R}}\Bigg|\frac{\partial \mathcal{R}(c,v;x)}{\partial c}\Bigg| \le \frac{\varepsilon }{c^2_{\min}}\|f'_J\|_\infty + \frac{\|f_J\|_\infty}{c^2_{\min}} + \frac{\|f'_J\|_\infty \overline{M}}{c^2_{\min}},
\end{equation}
for every $(c,v) \in \{(r_1,r_2); |r_2|\le v_{\max}, c_{\min}\le |r_1| \le c_{\max}\}.$ By using Assumptions (B1-B2-B3) jointly with the Lipschitz property (\ref{LipRden}), we may conclude the proof that Assumption (H1) is fulfilled. Property (\ref{rho_M}) in Assumption (H2) is an immediate consequence of Assumption (B1). 
\end{proof}

\subsection{Construction of training samples for SDEs driven by fractional Brownian motion}
Fix a smooth $\varrho:\mathbb{R}\times \mathbb{A}\rightarrow \mathbb{R}$, where $\mathbb{A}$ is compact. The dynamics of $\Delta X^a_n$ conditioned on $\Xi_{n-1} = (w_1,y_1, \ldots, w_{n-1},y_{n-1})$ is given by   
$$
 \Delta X^a_n = \varrho \big(\mathbf{y}_{n-1}, a\big)J +\sigma \ell_{n}(\mathcal{A}_{n-1},J),
$$
where $J\stackrel{d}{=}\Delta T_1$ is independent of $(\mathcal{W}_n,O_{n-1})$ (and hence of $(\mathcal{A}_{n-1}, \mathcal{W}_n)$ as well) and   
\begin{equation}\label{fracell}
\ell_n(\mathcal{A}_{n-1},J)=\sum_{i=1}^{n-1}\Delta A_i[K(T_{n-1},T_i) - K(T_{n-1}+J,T_i)  ], 
\end{equation}
for $n\ge 2$ and we set $\ell_1(\cdot)=0$. Here, for simplicity, we assume that $K$ is the kernel of the Riemann-Liouville fractional Brownian motion given by 

$$K(t,s) = \sqrt{2H}(t-s)^{H-\frac{1}{2}},$$
for $s < t$ and $\frac{1}{2} < H < 1$. The next elementary result connects $\ell_n$ with the imbedding scheme (\ref{discFBM}) developed in \cite{LEAO_OHASHI2017.2}.

\begin{lemma}\label{ibp1}
Let 

\begin{eqnarray*}
L(\mathcal{A}_{n-1}, \mathcal{W}_n)&:=& \sum_{i=1}^n A_{i-1}\int_{T_{i-1}}^{T_i} \frac{\partial K}{\partial s}(T_n,s)ds - \sum_{i=1}^{n-1} A_{i-1}\int_{T_{i-1}}^{T_i} \frac{\partial K}{\partial s}(T_{n-1},s)ds\\
&=& \sum_{i=1}^n A_{i-1}\Big\{K(T_n,T_i) - K(T_n,T_{i-1})\Big\} -  \sum_{i=1}^{n-1} A_{i-1}\Big\{K(T_{n-1},T_i) - K(T_{n-1},T_{i-1})\Big\} 
\end{eqnarray*}
for $n\ge 2$. Then, 
$$L(\mathcal{A}_{n-1}, \mathcal{W}_n)=\sum_{i=1}^{n-1}\Delta A_i\,[K(T_{n-1},T_i)-K(T_n,T_i)],$$
for $n\ge 2$. 
\end{lemma}
\begin{proof}
Fix $n\ge 2$ and write $D^n_i = K(T_n,T_i)$. 
We set
\[
R_n:=\sum_{i=0}^{n-1} A_i [D^n_{i+1}-D^n_i].
\]
By the discrete integration--by--parts identity
\[
\sum_{i=0}^{n-1} A_i\,(D^n_{i+1}-D^n_i)
= A_{n}D^n_n - A_0 D^n_0 - \sum_{i=0}^{n-1} D^n_{i+1} (A_{i+1}-A_{i}),
\]
Using $K(T_n,T_n)=0$,
\[
R_n=-\sum_{i=0}^{n-1} (A_{i+1}-A_{i}) K(T_n,T_{i+1}).
\]
Similarly, if we set $D^{n-1}_i = K(T_{n-1},T_i)$ and $R_{n-1}=\sum_{i=0}^{n-2} A_i [D^{n-1}_{i+1}-D^{n-1}_i]$, we then have 

$$R_{n-1}= -\sum_{i=0}^{n-2} (A_{i+1}-A_{i}) K(T_{n-1},T_{i+1})$$

Therefore
\begin{eqnarray*}
R_n-R_{n-1}&=& -\sum_{i=1}^{n-1} (A_{i}- A_{i-1})\,K(T_n,T_i)
+ \sum_{i=1}^{n-2} (A_i - A_{i-1})\,K(T_{n-1},T_i)\\
&=& \sum_{i=1}^{n-2}(A_i - A_{i-1})\,\big[ K(T_{n-1},T_i)-K(T_n,T_i)\big]
- (A_{n-1} - A_{n-2})K(T_n,T_{n-1}).
\end{eqnarray*}
Since $K(T_{n-1},T_{n-1})=0$, this simplifies to

$$R_{n} - R_{n-1}= \sum_{i=1}^{n-1}(A_i - A_{i-1}) [K(T_{n-1},T_i)-K(T_n,T_i)].$$
This concludes the proof. 
\end{proof}

\begin{remark}\label{growthtau}
Let $f_J$ be the density of $J$. From \cite{Burq_Jones2008}, we know that there exist $U_0>0$ and constants $C_\Delta,\gamma_\Delta>0$ such that
\[
f_J(u)\ \le\ C_\Delta\,e^{-\gamma_\Delta u}\qquad\text{for all }u\ge U_0.
\]
\end{remark}

For a realized history $b=\Xi_{j-1}, c \in \mathbb{R}$ and a control value $a$, we denote 
\begin{equation}\label{FBMpathwise}
\psi_{b}(u):=\sum_{i=0}^{j-1}\Delta A_i\,\Big[K(T_{j-1},T_i)-K(T_{j-1}+u,T_i)\Big],
\qquad
\phi_{b,c}(u):=c u+\psi_{b}(u),
\end{equation}
for $u \ge 0$.

The controlled increment $\Delta X^a_j$ conditioned on $\Xi_{j-1} = b_{j-1}$ and with control slope $c_{j-1}:=\varrho(\mathbf{y}_{j-1},a)$ has the representation
\begin{equation}\label{DeltaXFBMlaw}
\Delta X_j^a \stackrel{(d)}{=} \phi_{b_{j-1},c_{j-1}}(J),
\end{equation}
where $J$ is independent of $(\Xi_{j-1},\mathcal W_j)$ and $J\stackrel d=\Delta T_1$ truncated at $[\underline{M},\overline{M}]$ for $0 < \underline{M} < \overline{M} < \infty$.

In the sequel, let 

$$f_{\phi_{b_{j-1},c_{j-1}}}$$ 
be the density of $\phi_{b,c}(J)$ for $c = c_{j-1}$ and a history $b = \Xi_{j-1} = b_{j-1}$. See (\ref{fphibc}) for details. Let us define the following constants

$$C_1:=\sqrt{2H}\big(H-\frac{1}{2}\big)\varepsilon m \underline{M}^{H-\frac{3}{2}},\quad C_0:=\sqrt{2H}\varepsilon m \overline{M}^{H-\frac{1}{2}}.$$

In order to prove that $\Delta X^a$ satisfies Assumption H1, we need the following natural regularity conditions on the drift of the controlled SDE (\ref{limsdefbm}).  

\

\begin{itemize}

\item[(A1)] The drift $\varrho$ is globally Lipschitz. That is, there exists a constant $\|\varrho\|$ such that  

$$|\varrho(y,a) - \varrho(y',a')|\le \|\varrho\|\big\{ |y-y'|+ |a-a'|\big\},$$
for every $y,y' \in \mathbb{R}$ and $a,a' \in \mathbb{A}$. 

\item[(A2)] (\emph{control separation}) There exist $0<c_{\min}<  c_{\max}<\infty$ such that
\[
  c_{\min}\ \le\ |\,\varrho(y,a)\,|\ \le\ c_{\max}\qquad\text{for all }(y,a)\in\mathbb R\times\mathbb A.
\]
    \item[(A3)] (\emph{large control separation})   $C_1 \le \frac{1}{2}c_{min}$.  
    \end{itemize}

\begin{lemma}\label{suplevel1}
Suppose (A1-A2-A3) are fulfilled. Let $\mu$ be the two-sided Laplace probability distribution of the form 

\[
\mu(dx')=q_{\beta}(x')\,dx',\qquad q_{\beta}(x')=\frac{\beta}{2}e^{-\beta|x'|}.
\]
where  $0<\beta\le \frac{\gamma_\Delta}{c_{\max}}$. Then, for a given history $b_{j-1}=\Xi_{j-1}=(w_1, y_1, \ldots, w_{j-1},y_{j-1})$ and a control value $a$, the conditional law of $(\mathcal{W}_j, \Delta X^a_j)$ is given by   

$$
\mathbb{P}[ (\mathcal{W}_j, \Delta X^a_j) \in dxdx' | \Xi_{j-1}= b_{j-1}]= r_j(a,x'; b_{j-1})\mu(dx')\nu(dx),
$$
where 
$$r_j(a,x';b) = \Bigg(\frac{f_{\phi_{b_{j-1},c_{j-1}}}}{q_\beta}\Bigg)(x'), $$
for $a \in \mathbb{A}$, $\Xi_{j-1} = b_{j-1}$, $c_{j-1} = \varrho(\mathbf{y}_{j-1},a)$ and $x' \in \mathbb{R}$ and $1\le j\le m$. Moreover, there exists a constant $C$ which depends on $C_\Delta, \gamma_\Delta, \beta,c_{min},c_{max},\epsilon,m, \bar{M},C_0$ such that

$$\|r\|_\infty=\max_{1\le j\le m}\sup_{x' \in \mathbb{R}, a \in \mathbb{A}, b \in \mathbb{H}^{j-1}}|r_j(a,x'; b)|\le C < \infty.$$

\end{lemma}
\begin{proof}
Fix $1\le j\le m$, a control value $a$, a history $b_{j-1}=\Xi_{j-1}$ and $c_{j-1} = \varrho(\mathbf{y}_{j-1},a)$. To keep notation simple, we write $b=b_{j-1}, c=c_{j-1}$, $\phi_{b,c} = \phi_{b_{j-1},c_{j-1}}$ and $\psi_{b} = \psi_{b_{j-1}}$. The variables $(c,b)$ and $0 < \beta \le \frac{\gamma_\Delta}{c_{\max}}$ are fixed once and for all. Positive constants which only depends on $H$ will be generically denoted by $c_H$.   Observe that 

$$\phi'_{b,c}(u)= c + \psi'_{b}(u),$$

$$\psi'_b(u)=-c_H\sum_{i=1}^{j-1} \Delta A_i \big( T_{j-1} + u- T_i \big)^{H-\frac{3}{2}}$$
Then,

\begin{equation}\label{psip}
|\psi'_b(u)|\le C_1
\end{equation}
for every $u\ge \underline{M}$. Observe that if $c < -c_{min}$, then assumption (A3) and (\ref{psip}) yield 
\begin{eqnarray*}
\phi'_{b,c}(u)&<& -c_{min} + \psi'_b(u)\\
&=& -\frac{1}{2}c_{min} - \frac{1}{2}c_{min}+ \psi'_b(u)\\
&<& -\frac{1}{2}c_{min},
\end{eqnarray*}
for $u \ge \underline{M}$. If $c > c_{min}$, then 

\begin{eqnarray*}
\phi'_{b,c}(u) = c + \psi'_b(u) &>& c_{min} - C_1 =  \frac{1}{2}c_{min} + \frac{1}{2}c_{min}-C_1\\
&\ge & \frac{1}{2}c_{min},
\end{eqnarray*}
for $u \ge \underline{M}$. Then, $\phi_{b,c}: [\underline{M},\overline{M}]\rightarrow \mathbb{R}$ it is a $C^1$ strictly monotone function and it has a continuous inverse denoted by $u_{b,c}$, where 

\begin{equation}\label{bphip}
|\phi'_{b,c}(u)|\ge \frac{1}{2}c_{min},
\end{equation}
for every $u\ge \underline{M}$. Observe that 

$$|\psi_b(u)|\le C_0,$$
for every $ \underline{M}\le u\le \overline{M}$. Summing up, we arrive at the inequality

\begin{eqnarray}
\nonumber|\phi_{b,c}(u)|&\le& |c|u + |\psi_b(u)|\\
\nonumber&\le& |c|u + C_0\\
\nonumber&\le& c_{max} u + C_0\\
\label{phii}&=& u c_{\max}+ C_0,
\end{eqnarray}
for $ \underline{M}\le u \le \overline{M}$. The inequality (\ref{phii}) and the fact that the variable $u_{b,c}(y)>0$ yield 

$$u_{b,c}(y)\ge \frac{(|y| - C_0)^+}{c_{\max}}=: L_1(y),$$
for every $y \in \phi_{b,c}\big(\big[\underline{M},\overline{M} \big]\big)$. Jacobian method applied to (\ref{DeltaXFBMlaw}) allows us to state that the conditional density of $\Delta X^a_{j}| \Xi_{j-1}=b$ is given by 

 
\begin{equation}\label{fphibc}
f_{\phi_{b,c}}(y) = \frac{f_J (u_{b,c}(y))}{|\phi'_{b,c}(u_{b,c}(y))|};\quad y \in \phi_{b,c}\big(\big[\underline{M},\overline{M} \big]\big),  
\end{equation}
where we shall set $f_{\phi_{b,c}}(y) = 0$ whenever $y \notin \phi_{b,c}\big(\big[\underline{M},\overline{M} \big]\big)$. Let us fix $y \in \phi_{b,c}\big(\big[\underline{M},\overline{M} \big]\big)$. From Remark \ref{growthtau}, we can pick up $U_0$ such that 
\[
f_J(u)\ \le\ C_\Delta\,e^{-\gamma_\Delta u}\qquad\text{for all }u\ge U_0\ge \underline{M}.
\]
With such $U_0$, we define $F_0 = \sup_{\underline{M}\le u\le U_0} f_J (u)$, $Z_0=c_{\max}U_0 +C_0$. Observe that

$$f_{\phi_{b,c}}(y)\le \frac{1}{\frac{1}{2}c_{min}}\sup_{v \ge L_1(y)} f_J (v).$$ 
Now, $L_1(y) > U_0\Longleftrightarrow |y|> Z_0$ so that 

$$\frac{1}{\frac{1}{2}c_{min}}\sup_{v \ge L_1(y)} f_J (v)\mathds{1}_{\{|y|> Z_0\}} \le  \frac{C_{\Delta}}{\frac{1}{2}c_{min}} \exp \Big(-\frac{\gamma_\Delta}{c_{\max}} \big( |y| - C_0\big)^+ \Big)\mathds{1}_{\{|y| > Z_0\}}.$$
Now, $L_1(y) \le U_0\Longleftrightarrow |y|\le Z_0$ so that 
   
$$\frac{1}{\frac{1}{2}c_{min}}\sup_{v \ge L_1(y)} f_\Delta (v)\mathds{1}_{\{|y|\le Z_0\}} \le  \frac{F_0}{\frac{1}{2}c_{min}} \mathds{1}_{\{|y| \le Z_0\}}.$$  

Hence,

\begin{eqnarray*}
f_{\phi_{b,c}}(y)&\le& \frac{F_0}{\frac{1}{2}c_{min}}\mathds{1}_{\{|y| \le Z_0\}}\\
&+& \frac{C_{\Delta}}{\frac{1}{2}c_{min}} \exp \Big(-\frac{\gamma_\Delta}{c_{\max}} \big( |y| - C_0\big)^+ \Big)\mathds{1}_{\{|y| > Z_0\}},
\end{eqnarray*}
for $y\in \mathbb{R}$. Now, for $q_{\beta}(y) = \frac{\beta}{2} \exp (-\beta|y|)$ with $0<\beta\le \frac{\gamma_\Delta}{c_{\max}}$, we have 

\begin{eqnarray}
\nonumber\frac{f_{\phi_{b,c}}(y)}{q_\beta(y)}&\le& \frac{2 F_0}{\beta \frac{1}{2}c_{min}}\exp (\beta |y|)\mathds{1}_{\{|y|\le Z_0\}} +\frac{2 C_{\Delta}}{\beta \frac{1}{2}c_{min}} \exp \Big(-\frac{\gamma_\Delta}{c_{\max}} \big( |y| - C_0\big) \Big)\mathds{1}_{\{|y| > Z_0\}}\\
\label{unr_j}&\lesssim& \exp \Big(-\Big(\frac{\gamma_\Delta}{c_{\max}} - \beta\Big)|y| + \frac{\gamma_\Delta}{c_{\text{max}}}
 C_0 \Big)\lesssim \exp \Big(-\Big(\frac{\gamma_\Delta}{c_{\max}} - \beta\Big)|y|\Big),
\end{eqnarray}
for every $y \in \mathbb{R}$. We set 

$$r_j(a,x';b) = \frac{\text{Law}[\Delta X^a_j | \Xi_{j-1} = b]}{d\mu} = \Bigg(\frac{f_{\phi_{b,c}}}{q_\beta}\Bigg)(x'), $$
for $a \in \mathbb{A}$, $\Xi_{j-1} = b$, $c = \varrho(\mathbf{y}_{j-1},a)$ and $x' \in \mathbb{R}$. The estimate (\ref{unr_j}) allows us to conclude the proof.
\end{proof}

\begin{lemma}\label{supder}
Under Assumptions (A1-A2-A3), there exist finite positive constants $L_1$ and $L_2$ which depend on $c_{min}, c_{max}, m, \epsilon, H$ and $\underline{M}$ such that 

\begin{equation}\label{derfc}
\Big|\frac{\partial f_{\phi_{b,c}}}{\partial c}(x)\Big|\le L_1
\end{equation}

and 

\begin{equation}\label{derfb}
 \Big\|\nabla_b f_{\phi_{b,c}}(x)\Big\|\le L_2, 
\end{equation}
for every $1\le j\le m, x \ge \underline{M}, b \in \mathbb{W}^{j-1}, |c| \in [c_{min},c_{max}]$ and $x \in \mathbb{R}$.  
\end{lemma}
\begin{proof}
In the sequel, $C$ is a constant which may differ from line to line. Let $u_{b,c}$ be the inverse of $\phi_{b,c}:[\underline{M},\overline{M}]\rightarrow \phi_{b,c}([\underline{M},\overline{M}])$. Recall we set $f_{\phi_{b,c}}(x)=0$ whenever $x \notin \phi_{b,c}\big(\big[\underline{M},\overline{M} \big]\big)$. Let $F(c,u) = \phi_{b,c}(u) - x' = cu + \psi_b(u) -x'$. Here $x'$ is a constant (the evaluation point in the density). For a given base $c_0$ with $|c_0| > c_{min}$, we can apply Assumption (A3) and (\ref{psip}) to infer that

$$\partial_u F(c_0,u_0) = c_0 + \psi'_b(u_0)\neq 0,$$
and hence there is a neighborhood $\mathcal{U}$ of $c_0$ and a unique $C^1$ map $u: \mathcal{U}\rightarrow [\underline{M},\overline{M}]$ such that

$$F(c,u(c)) = 0~\text{and}~u(c_0)=u_0,$$ 
for every $c \in \mathcal{U}$. That is, $\phi_{b,c}(u(c)) = x'$ for every $c \in \mathcal{U}$. 
Hence, we can make use of chain rule and implicit function theorem to infer that for each $x' \in \text{int} (\text{Range}~\phi_{b,c})$, we have  

\begin{equation}\label{impF}
\frac{\partial u_{b,c}(x')}{\partial c} = \frac{-u_{b,c}(x')}{\phi'_{b,c}(u_{b,c}(x'))}.
\end{equation}
Then,  
\[
\partial_c f_{\phi_{b,c}}(x')
=\text{sgn}(\phi'(u))\Bigg[ - \frac{u\,f'_J(u)}{\phi'^2(u)}
+ \frac{u\,f_J(u)\,\psi_b''(u)}{\phi'^3(u)}
- \frac{f_J(u)}{\phi'^2(u)}\Bigg|_{u=u_{b,c}(x')}\Bigg],
\]
where $\phi'=\phi'_{b,c}$. Observe 

$$\psi^{''}_b(u)=-c_H\sum_{i=1}^{j-1} \Delta A_i \big( T_{j-1} + u- T_i \big)^{H-\frac{5}{2}}$$
so that 

\begin{equation}\label{psisec}
|\psi^{''}_b(u)|\lesssim_H \epsilon m \underline{M}^{H-\frac{5}{2}}
\end{equation}
for every $u\ge \underline{M}$. By using (\ref{bphip}), (\ref{fdeltab}), (\ref{psisec}) and Lemmas 2 and 3 of \cite{Burq_Jones2008}, we conclude that (\ref{derfc}). Similarly, a direct calculation shows that if $c > c_{min}$ then  

$$
\nabla_b f_{\phi_{b,c}}(x')
= \Bigg\{ \frac{f'_J(u_{b,c}(x')) \phi'_{b,c}(u_{b,c}(x')) - f_J(u_{b,c}(x')) \psi''_{b}(u_{b,c}(x'))}{|\phi'_{b,c}(u_{b,c}(x'))|^2}\Bigg\}\nabla_b u_{b,c}(x')
$$
for $x' \in \phi_{b,c}\big(\big[\underline{M},\overline{M} \big]\big)$. If $c < - c_{min}$, then 

$$
\nabla_b f_{\phi_{b,c}}(x')
= \Bigg\{ -\frac{f'_J(u_{b,c}(x')) \phi'_{b,c}(u_{b,c}(x')) + f_J(u_{b,c}(x')) \psi''_{b}(u_{b,c}(x'))}{|\phi'_{b,c}(u_{b,c}(x'))|^2}\Bigg\}\nabla_b u_{b,c}(x'),
$$
for $x' \in \phi_{b,c}\big(\big[\underline{M},\overline{M} \big]\big)$. We will check that there exists a constant $C$ such that 
$$
\|\nabla_b  f_{\phi_{b,c}}(x')\|\le C,
$$
for every $1\le j\le m, x \ge \underline{M}, b \in \mathbb{W}^{j-1}, |c| \in [c_{min},c_{max}]$ and $x' \in \mathbb{R}$. Similar to the argument related to the formula (\ref{impF}), we have

$$
\frac{\partial u_{b,c}(x')}{\partial b_\ell}= \frac{-\frac{\partial \phi_{b,c}}{\partial b_\ell}(u_{b,c}(x'))}{\phi'_{b,c}(u_{b,c}(x'))} = \frac{-\frac{\partial \phi_{b,c}}{\partial b_\ell}(u_{b,c}(x'))}{c+ \psi'_{b,c}(u_{b,c}(x'))},
$$
for $x' \in \text{int (Range}~\phi_{b,c})$ and $1\le \ell\le j-1$. Observe we can write

$$\psi_b(u)= \sum_{i=0}^{j-1} \Delta A_i \Bigg\{ \Bigg( \sum_{\ell=i+1}^{j-1} \Delta T_\ell\Bigg)^{H-\frac{1}{2}} - \Bigg( \sum_{\ell=i+1}^{j-1} \Delta T_\ell+z\Bigg)^{H-\frac{1}{2}}    \Bigg\}$$
and 

$$\nabla_b \phi_{b,c}(z) = \nabla_b \psi_b(z),$$
 for $z \in [\underline{M},\overline{M}]$. Therefore,

$$\|\nabla_b \phi_{b,c}(z)\|\le \big (H-1/2\big) \epsilon m \underline{M}^{H-\frac{3}{2}}$$ 
for every $z \in [\underline{M},\overline{M}]$, $b \in \mathbb{W}^{j-1}$, $1\le j\le m$ and $|c| > c_{min}$. From (\ref{bphip}), we know that $|\phi'_{b,c}(u)|\ge \frac{1}{2}c_{min}
$, for every $u \ge \underline{M}$. This concludes that there exists a constant $C$ which depends on $c_{min}, m, \epsilon, H$ and $\underline{M}$ such that  

$$\| \nabla_b u_{b,c}(x') \|\le C$$
for every $x' \in \mathbb{R}, b \in \mathbb{W}^{j-1}$ and $|c| > c_{min}$. Summing up the estimates given in Lemmas \ref{suplevel1} and \ref{fdeltab}, there exists a constant $C$ which depends on $C_1, c_{min}, c_{max}, \underline{M}, \epsilon,H$ such that (\ref{derfb}) holds true. This concludes the proof.  
\end{proof}
Summing up Lemmas \ref{suplevel1}, \ref{supder} and the global Lipschitz property of the drift $\varrho$, we arrive at the following result.

\begin{theorem}\label{FBMtheorem}
Assume (A1-A2-A3) are fulfilled. Let $\mu$ be the two-sided Laplace probability distribution of the form 

\[
\mu(dx')=q_{\beta}(x')\,dx',\qquad q_{\beta}(x')=\frac{\beta}{2}e^{-\beta|x'|},
\]
where  $0<\beta\le \frac{\gamma_\Delta}{c_{\max}}$. Then, for a given history $b_{j-1}=\Xi_{j-1}=(w_1, y_1, \ldots, w_{j-1},y_{j-1})$ and a control value $a$, the conditional law of $(\mathcal{W}_j, \Delta X^a_j)$ is given by   

$$
\mathbb{P}[ (\mathcal{W}_j, \Delta X^a_j) \in dxdx' | \Xi_{j-1}= b_{j-1}]= r_j(a,x'; b_{j-1})\mu(dx')\nu(dx),
$$
where 

$$r_j(a,x';b_{j-1}):= \frac{f_{\phi_{b_{j-1},c_{j-1}}}(x')}{q_{\beta}(x')}$$
and $f_{\phi_{b_{j-1},c_{j-1}}}$ is the density of $\phi_{b_{j-1},c_{j-1}}(J)$ for $c_{j-1}=\varrho(\mathbf{y}_{j-1},a)$. Moreover, (\ref{rinfty}) and (\ref{rlip}) hold true. 
\end{theorem}

\begin{remark}As shown by the proof of Lemma \ref{suplevel1}, the practical implementation of the importance-sampling density for controlled SDEs driven by fractional Brownian motion, requires numerically inverting the map \(u\mapsto \phi_{b,c}(u)=cu+\psi_b(u)\) described in (\ref{FBMpathwise}). Hence, the main bottleneck is the stable computation of \(\phi_{b,c}^{-1}\) and the corresponding Jacobian factor along $\phi^{-1}_{b,c}$. This makes the implementation of importance sampling substantially more delicate than in the classical path-dependent diffusion setting.
\end{remark}

\subsection{Construction of the importance sampling weights and training measures for rough stochastic volatility models}\label{trrv}
This section is dedicated to the construction of the training probability measures $\mu \in \mathcal{P}(\mathbb{R}_+\times \mathbb{R})$ such that Assumption (R1) is fulfilled for a given family of strategies $\mathcal{H}$ as defined in (\ref{rsetd1}) and (\ref{rsetd0}). 
The idea can be summarized as follows: Let 


$$\Delta S_n\stackrel{(d)}{=} \mu_{\text{drift}}S_{n-1}J + S_{n-1} \vartheta(V_{n-1})\Delta A_1$$
be a rough stochastic volatility model of the risky asset price. The agent follows the following steps:  
\begin{enumerate}
  \item The agent chooses a range $[s_{\min} , s_{\max}]$ with $0 < s_{\min} < s_{\max}$ such that he believes the price should remain during the lifetime of the contingent claim.     
  \item The agent chooses a range $[\theta_{\min} , \theta_{\max}]$ with $0 < \theta_{\min} < \theta_{\max}$ such that he believes the volatility of the risky asset price should remain during the lifetime of the contingent claim. That is, $\theta_{\min}\le |\vartheta(\cdot)|\le \theta_{\max}$.   
  \item He fixes $0 < c_{\min}:= s_{\min}|\mu_{\text{drift}}| < s_{\max}|\mu_{\text{drift}}|=: c_{\max}$ and $0< v_{\min} := s_{\min}\theta_{\min} <  s_{\max}\theta_{\max}:= v_{\max} <\infty.$ Let $K$ be the set of all admissible instantaneous increments
  
  $$K=[-v_{\max}\epsilon - c_{\max}\overline{M}, v_{\max}\epsilon + c_{\max}\overline{M}].$$   
\end{enumerate}
Then he designs a training stage based on a pushforward law $\mathcal{Z}_{\#} \beta_K$, for $\beta_K(dadx)=G(a,r)\lambda(da)q(r)dr$, where  
 \begin{itemize}
   \item A training density $q$ bounded below on $K$. 
   \item An action distribution $\lambda$.  
   \item A joint kernel $G(a,x)$ bounded below on $\mathbb{A}\times K$. 
   \item $\mathcal{Z}(a,z) = (z,az)^\top$. 
\end{itemize}

\subsubsection{Complete Market case}
This section is devoted to a detailed discussion on the rough volatility model. For pedagogical reasons, we start with the case $\rho=\pm 1$ so that there is only one underlying Brownian motion. The case $-1 < \rho < 1$ will be treated in the next section. The goal is to show that the rough volatility model satisfies Assumption R1 for a large class of training measures and provide analytical expressions for the correspondent importance sampling weights.

From Lemma \ref{repWH}, we recall that the imbedding (\ref{discreterough1}) (in the one-dimensional case) can be written as $W^H_n=0$ for $0\le n\le 1$ and 

\begin{equation}\label{nrfbm}
W^H_n =\sum_{j=2}^n \Delta A_j K_{H,1}(T_n,T_{j-1}) + \sum_{j=1}^{n-1}\Delta A_jK_{H,2}(T_n,T_j) 
\end{equation}
for $n\ge 2$. Moreover, 

$$\ln V_{n} = \varkappa + e^{-\beta T_n}(z_0-\varkappa) + \zeta W^H_n -\beta \zeta e^{-\beta T^k_n}\sum_{j=1}^n W^H_{j-1} e^{\beta T_{j-1} }\Delta T_{j}.$$

In the sequel, all the parameters of the the fractional Ornstein-Uhlenbeck process will be fixed, namely $0 < H <\frac{1}{2}$, $\zeta,\beta>0$, $\mu_{\text{drift}} \neq 0$ and $ \varkappa\in \mathbb{R}$. 

\begin{remark}
Observe that $V_n$ is an explicit function $\mathbf{z}$ of $\mathcal{A}_n = (\mathcal{W}_1, \ldots, \mathcal{W}_n)$ that we denote

$$V_n = \mathbf{z}(\mathcal{A}_n).$$  
\end{remark}
Observe that the range of the process $\Delta X^a_n$ conditioned on $\Xi_{n-1}$ is a subset of a slice of a cone in $\mathbb{R}^2$ parameterized by $a \in \mathbb{A}$:

$$\Delta X^a_n|\Xi_{n-1} \subset \{x' \in \mathbb{R}\times \mathbb{R}; x'_2 = ax'_1\}.$$
Fix $0 < s_{\min} < s_{\max} < \infty$ and we set $\Delta x_0 :=S_0$. Let $\mathfrak{T}_n: (\mathbb{R}^2)^n\rightarrow \mathbb{R}_+$ defined by 

\begin{equation}\label{truncf}
\mathfrak{T}_n(y_0, \ldots, y_n):=\min \Big\{ \max \Big\{ |\sum_{i=0}^n y_{i,(1)}|, s_{\min} \Big\}, s_{\max} \Big\},
\end{equation}
for $n\ge 0$ and $\mathfrak{T}_0:= S_0$. Here, the first coordinate of $y_i$ is denoted by $y_{i,(1)}$. One can easily check that $\mathfrak{T}_n$ is $\sqrt{n}$-Lipschitz ($1\le n\le m-1$) and

\begin{equation}
0 < s_{\min}\le |\mathfrak{T}_n(y_0, \ldots, y_n)|\le s_{\max},
\end{equation}
for every $(y_0, \ldots, y_n) \in (\mathbb{R}^2)^{n+1}$ with $0\le n \le m-1$. For a control value $a\in\mathbb A$ and a history $\Xi_{j} =\mathbf{o}_{j}=(w_0, y_0, \ldots, w_{j},y_{j})$ (with $\Xi_0 = (0,0,x_0)$), we define
\[
c(\mathbf{o}_{j},a) = \left(
            \begin{array}{c}
              c^{(1)}(\mathbf{o}_{j},a) \\
              c^{(2)}(\mathbf{o}_{j},a) \\
            \end{array}
          \right)
:=\begin{pmatrix}
\mu_{\mathrm{drift}} \mathfrak{T}_j(y_0, \ldots, y_{j})\\
a\,\mu_{\mathrm{drift}}\mathfrak{T}_j(y_0, \ldots, y_{j})
\end{pmatrix},
\]
\[v(\mathbf{o}_{j},a) = \left(
            \begin{array}{c}
              v^{(1)}(\mathbf{o}_{j},a) \\
              v^{(2)}(\mathbf{o}_{j},a) \\
            \end{array}
          \right)
:=\begin{pmatrix}
\mathfrak{T}_j(y_0, \ldots, y_{j})\,\vartheta(\mathbf{z}(\bar w_{j}))\\
a\,\mathfrak{T}_j( y_0, \ldots, y_{j})\,\vartheta(\mathbf{z}(\bar w_{j}))
\end{pmatrix},
\]
where, to shorten notation, we set $\bar{w}_{j}:= (w_0, \ldots, w_{j})$, with $w_0=(0,0)$.  






\

We will assume the following assumptions:

\begin{itemize}


\item[(C1)] $\vartheta: \mathbb{R}_+\rightarrow \mathbb{R}_+$ is a bounded Lipschitz function: There exists $\theta_{\max}$ such that 

$$0 < |\vartheta(z)|\le \theta_{\max},$$
for every $z\ge 0$. 
  
\item[(C2)]  We assume $0 < \underline{M}\le J \le \overline{M}$ a.s.\ for some $0 < \underline{M} < \overline{M}$; equivalently, the law of $J$ is the truncation of $\Delta T_1$ to $[\underline{M},\overline{M}]$.
\end{itemize}


Factoring out the common direction $(1,a)^{\!\top}$ gives
\begin{equation}\label{rlde}
\text{Law}[\Delta X^a_j|\Xi_{j-1}=\mathbf{o}_{j-1}]\stackrel{d}{=} R_{c^{(1)}(\mathbf{o}_{j-1}),v^{(1)}(\mathbf{o}_{j-1})}\left(
                                                                                   \begin{array}{c}
                                                                                     1 \\
                                                                                     a \\
                                                                                   \end{array}
                                                                                 \right)
\end{equation}
where $c^{(1)}(\mathbf{o}_{j-1},a) = c^{(1)}(\mathbf{o}_{j-1}) $ and $v^{(1)}(\mathbf{o}_{j-1},a) = v^{(1)}(\mathbf{o}_{j-1})$ do not depend on controls, for $j=m,\ldots, 1$.   

\begin{remark}
Observe that $\Delta X^a_j$ conditioned on the information set $\Xi_{j-1}$ described in (\ref{rlde}) is slightly different from the theoretical rough volatility model. The distinction lies on the truncation function (\ref{truncf}) which is necessary in order to avoid degenerate samples reaching zero. 
\end{remark}


The following elementary result is important for constructing training measures $\pi_j:\mathbb{H}^j\rightarrow \mathcal{P}(\mathbb{A})$ satisfying the Lipschitz property (\ref{rholip}) in Assumption R1.   
\begin{lemma}\label{lippr}
If the training data $\bar{w}_{n-1}  = (w_0, \ldots, w_{n-1})$ is generated by $\nu$, where $0 < \underline{M}\le J \le \overline{M}$, then $\bar{w}_{n-1}\mapsto \vartheta\circ \mathbf{z}(\bar{w}_{n-1})$ is globally Lipschitz for each $1\le n\le m$ and hence, 

$$(w_0,y_0, \ldots,w_{n-1},y_{n-1})=\mathbf{o}_{n-1}\mapsto \mathfrak{T}_{n-1}(y_0,\ldots, y_{n-1})\vartheta\circ \mathbf{z}(\bar{w}_{n-1})$$
is a bounded globally Lipschitz function for each $1\le n\le m$.    
\end{lemma}
\begin{proof}
 Recall 

$$\ln \mathbf{z} (\mathcal{A}_n) = \varkappa + e^{-\beta T_n}(z_0-\varkappa) + \zeta W^H_n -\beta \zeta e^{-\beta T^k_n}\sum_{j=1}^n W^H_{j-1} e^{\beta T_{j-1} }\Delta T_{j},$$ 
where $W^H_n$ follows (\ref{nrfbm}) for a single Brownian motion. We recall $W^H_n=0$ for $0\le n\le 1$ and 


$$W^H_n= \sum_{j=2}^n \Delta A_j\,K_{H,1}(T_n,T_{j-1})
   + \sum_{j=1}^{n-1} \Delta A_j\,K_{H,2}(T_n,T_j),$$
for $n\ge 2$. We shall write 

\begin{equation}\label{dkh2} 
K_{H,2}(T_n,T_j)  = \Bigg(\sum_{\ell=1}^j \Delta T_\ell\Bigg)^{\frac{1}{2}-H}\sum_{\ell=j}^{n-1} \int_0^{\Delta T_{\ell+1}}\Big(y+ \sum_{i=1}^{\ell}\Delta T_i\Big)^{H-\frac{3}{2}} y^{H-\frac{1}{2}}dyc_H\Big(\frac{1}{2}-H\Big),
\end{equation}
for $1\le j\le n-1$ and

\begin{equation}\label{dkh1}
K_{H,1}(T_n,T_{j-1})=c_H \Bigg( \sum_{\ell=1}^n \Delta T_\ell \Bigg)^{H-\frac{1}{2}}\Bigg(\sum_{\ell=1}^{j-1} \Delta T_\ell\Bigg)^{\frac{1}{2}-H}\Bigg(\sum_{\ell=j}^n\Delta T_\ell\Bigg)^{H-\frac{1}{2}},
\end{equation}
for $2\le j\le n$, where $0 < \underline{M}\le \Delta T_r < T_n\le m \overline{M}$ for every $1\le r \le n$. Now, fix $2\le n\le m$. From (\ref{dkh1}) and (\ref{dkh2}), we first observe that 

\begin{equation}\label{boundKH2}
K_{H,2}(T_n,T_j)\le c_H\Big(\frac{1}{2}-H\Big) \underline{M}^{\frac{1}{2}-H}\underline{M}^{H-\frac{3}{2}}m \int_0^{\overline{M}}y^{H-\frac{1}{2}}dy
\end{equation},
\begin{equation}\label{boundKH1}
K_{H,1}(T_n,T_{j-1})\le \underline{M}^{2H-1} \overline{M}^{\frac{1}{2}-H},
\end{equation}
for every $2\le j\le n$. With an abuse of notation, in order to shorten notation, in the sequel we simply write 

$$\frac{\partial W^H_n}{\partial \Delta A_j}, \quad \frac{\partial K_{H,1}(T_n,T_{i-1})}{\partial \Delta T_i},\quad \frac{\partial K_{H,2}(T_n,T_{i})}{\partial \Delta T_i}$$ 
to denote the deterministic derivatives of the representative functions of $W^H_n, K_{H,1}$ and $K_{H,1}$ evaluated at $\bar{w}_n = (w_0, \ldots, w_n)$, where we recall $w_i$ is a realization of $(\Delta A_i,\Delta T_i)$. Then, (\ref{boundKH1}) and (\ref{boundKH2}) imply

$$\max_{1\le j\le n}\Big|\frac{\partial W^H_n}{\partial \Delta A_j}\Big|\le c_H \Big( \frac{1}{2}-H \Big) m\underline{M}^{-1}\int_0^{\overline{M}}y^{\frac{1}{2}-H}dy+ \underline{M}^{H-\frac{1}{2}}.$$
Now, fix $2\le i\le n-1$. Then, 

$$\frac{\partial W^H_n}{\partial \Delta T_i}=\Delta A_i\Big\{\frac{\partial K_{H,1}(T_n,T_{i-1})}{\partial \Delta T_i}+ \frac{\partial K_{H,2}(T_n,T_{i})}{\partial \Delta T_i}\Big\} $$ 
Again, using the fact that $J \ge \underline{M} >0$ a.s., we can safely take derivative of (\ref{dkh1}) and (\ref{dkh2}) w.r.t. $\Delta T_i$ to infer that there exists a constant $C(\underline{M},m,H)$ such that 

$$\max_{2\le n \le m}\max_{2\le i\le n}\Big|\frac{\partial K_{H,1}(T_n,T_{i-1})}{\partial \Delta T_i}\Big|+ \max_{2\le n \le m}\max_{1\le i\le n-1}\Big|\frac{\partial K_{H,2}(T_n,T_{i})}{\partial \Delta T_i}\Big|\le C(\underline{M},m,H).$$
Similar analysis for $\frac{\partial W^H_n}{\partial \Delta T_i}$ if $i=n$ or $i=1$. By using the fact that $\max_{1\le i\le m}|\Delta A_i|\le \epsilon$ a.s. we can safely state that $\bar{w}_n\mapsto \mathbf{z}(\bar{w}_n)$ has a bounded gradient and hence Lipschitz. Since $\vartheta$ is Lipschitz, the composition $\vartheta \circ \mathbf{z}$ is Lipschitz. By assumption, $\vartheta$ is bounded so that  $\vartheta \circ \mathbf{z}$ is a bounded Lipschitz function. The product of bounded Lipschitz functions $\mathfrak{T}_n \times \vartheta\circ \mathbf{z}$ is a bounded Lipschitz function. This concludes the proof.      
\end{proof}

We now aim to justify why one has to rely on randomized strategies as described in section \ref{rssection} rather than non-randomized ones. We start with an elementary remark whose proof we left to the reader. 
\begin{lemma}\label{mtlemma}
Let $\mu$ be a finite measure on a measurable space $(X,\mathcal F)$.
If $\{E_i\}_{i\in I}$ are pairwise disjoint measurable sets with 
$\mu(E_i)>0$ for all $i$, then the index set $I$ is at most countable.
\end{lemma}

\begin{proposition}\label{nonePCONE}
Fix $j\in\{1,\dots,m\}$ and $\mathbf{o}_{j-1}\in \mathbb{H}_{j-1}$. There is no probability measure $\mu \in \mathcal{P}(\mathbb{R}^2)$ such that

$$\text{Law}[\Delta X^a_j|\Xi_{j-1}=\mathbf{o}_{j-1}] << \mu$$
for every $a \in \mathbb{A}$. 
\end{proposition}

\begin{proof}
Recall that $\mathbb{A}$ is uncountable. Fix $j\in\{1,\dots,m\}$ and a history $\mathbf{o}_{j-1}\in \mathbb{H}_{j-1}$. By construction,
\[
\mathrm{Law}\big[\Delta X_j^a \mid \Xi_{j-1}=\mathbf{o}_{j-1}\big]
\stackrel{d}{=}
R_{c^{(1)}(\mathbf{o}_{j-1}),\,v^{(1)}(\mathbf{o}_{j-1})}
\begin{pmatrix}
1\\
a
\end{pmatrix},
\]
for every $a\in \mathbb{A}$, where $R_{c^{(1)}(\mathbf{o}_{j-1}),\,v^{(1)}(\mathbf{o}_{j-1})}$ is an absolutely continuous real-valued random variable. In particular, its law is not concentrated at
$0$. Hence, there exists $r>0$ such that either
\[
p_r^+ :=
\mathbb{P}\Big(
R_{c^{(1)}(\mathbf{o}_{j-1}),\,v^{(1)}(\mathbf{o}_{j-1})}\ge r
\Big)>0,
\]
or
\[
p_r^- :=
\mathbb{P}\Big(
R_{c^{(1)}(\mathbf{o}_{j-1}),\,v^{(1)}(\mathbf{o}_{j-1})}\le -r
\Big)>0.
\]

Assume first that $p_r^+>0$. For each $a\in \mathbb{A}$, define
\[
L_{a,+}^{(r)}:=\{(x,ax)\in\mathbb{R}^2:\ x\ge r\}.
\]
Then
\[
\Big\{
R_{c^{(1)}(\mathbf{o}_{j-1}),\,v^{(1)}(\mathbf{o}_{j-1})}\ge r
\Big\}
\subset
\Big\{
\Delta X_j^a\in L_{a,+}^{(r)}
\Big\},
\]
and therefore
\[
\mathrm{Law}\big[\Delta X_j^a \mid \Xi_{j-1}=\mathbf{o}_{j-1}\big]
\big(L_{a,+}^{(r)}\big)\ge p_r^+>0.
\]
Moreover, if $a\neq a'$, then
\[
L_{a,+}^{(r)}\cap L_{a',+}^{(r)}=\varnothing,
\]
because if $(x,y)$ belongs to the intersection, then
\[
(x,y)=(x,ax)=(x,a'x)
\]
with $x\ge r>0$, which implies $a=a'$. Assume now that $p_r^->0$. For each $a\in \mathbb{A}$, define
\[
L_{a,-}^{(r)}:=\{(x,ax)\in\mathbb{R}^2:\ x\le -r\}.
\]
Then
\[
\Big\{
R_{c^{(1)}(\mathbf{o}_{j-1}),\,v^{(1)}(\mathbf{o}_{j-1})}\le -r
\Big\}
\subset
\Big\{
\Delta X_j^a\in L_{a,-}^{(r)}
\Big\},
\]
and therefore
\[
\mathrm{Law}\big[\Delta X_j^a \mid \Xi_{j-1}=\mathbf{o}_{j-1}\big]
\big(L_{a,-}^{(r)}\big)\ge p_r^- >0.
\]
Again, if $a\neq a'$, then
\[
L_{a,-}^{(r)}\cap L_{a',-}^{(r)}=\varnothing,
\]
because if $(x,y)$ belongs to the intersection, then
\[
(x,y)=(x,ax)=(x,a'x)
\]
with $x\le -r<0$, hence $x\neq 0$, and again $a=a'$. Thus, in either case, there exists an uncountable family of pairwise disjoint measurable subsets $\{E_a:\ a\in \mathbb{A}\}$ of $\mathbb{R}^2$ such that
\[
\mathrm{Law}\big[\Delta X_j^a \mid \Xi_{j-1}=\mathbf{o}_{j-1}\big](E_a)>0
\]
for every $a\in \mathbb{A}$. Then, if there is a dominating probability measure $\mu\in\mathcal{P}(\mathbb{R}^2)$ it must have the property $\mu(E_a)>0$  for each $a\in \mathbb{A}$. Therefore, $\mu$ assigns strictly positive mass to an uncountable
family of pairwise disjoint measurable sets, which contradicts Lemma \ref{mtlemma}. Hence no such probability
measure $\mu$ exists.


\end{proof}

\begin{remark}
\label{rem:meaning}
Proposition \ref{nonePCONE} shows that over a given full admissible uncountable control set \(\mathbb A\), deterministic strategies do not provide a sufficiently rich exploration mechanism to support a single globally dominating training law. Any algorithm that continues to rely on deterministic strategies together with a presumed global reference law can only be interpreted as a method adapted to the portion of the state--action space actually visited by its training procedure, and not as a globally justified scheme for the original control problem. 
\end{remark}

We now make use again of the random field (\ref{Rfield})
$$R_{c,v}=c J + v \mathbf{B}$$
for $(c,v)$ satisfying 

$$s_{\min}|\mu_{\text{drift}}| \le|c| \le s_{\max}|\mu_{\text{drift}}|,\quad 0 < |v|\le \theta_{\max} s_{\max}.$$ 
Since $0 < \underline{M} \le J \le \overline{M}$ and $|\mathbf{B}|\le \varepsilon$, we observe we shall fix a compact subset $K\subset \mathbb{R}$

\begin{equation}\label{compactKRV}
K= \Big[-v_{\max} \varepsilon - c_{\max}\overline{M}, v_{\max} \varepsilon + c_{\max} \overline{M}\Big]
\end{equation}
with the constants $0 < c_{\min} < c_{\max}$, $v_{\max}>0$ defined by 

$$
0 < c_{\min}= s_{\min}|\mu_{\text{drift}}| < s_{\max}|\mu_{\text{drift}}|= c_{\max},\quad 0< s_{\max}\theta_{\max}= v_{\max} <\infty.
$$
By construction,  
$$\text{Range}~(R_{c,v}) \subset K\subset \mathbb{R}~a.s.,$$
whenever $c_{\min}\le |c|\le c_{\max}$ and $0 < |v|\le v_{\max}$. See Remark \ref{rangeR}. 

Under Assumptions (C1-C2), we can use (\ref{densR}) in the proof of Theorem \ref{pdsdeth} to state: For each fixed history $\Xi_{j-1} = \mathbf
{o}_{j-1} = (w_0, y_0, \ldots, w_{j-1}, y_{j-1})$, $v^{(1)}(\mathbf{o}_{j-1}) = \mathfrak{T}_{j-1}(y_0, \ldots, y_{j-1})\vartheta(\tau(\bar{w}_{j-1}))$ and $c^{(1)}(\mathbf{o}_{j-1}) = \mu_{\text{drift}} \mathfrak{T}_{j-1}(y_0, \ldots, y_{j-1})$, we do have 

$$\text{Law}~R_{c^{(1)}(\mathbf{o}_{j-1}),v^{(1)}(\mathbf{o}_{j-1})}<< \text{Leb}$$ 
with density $\mathcal{R}(c^{(1)}(\mathbf{o}_{j-1}),v^{(1)}(\mathbf{o}_{j-1}); \cdot)$. Proposition \ref{nonePCONE} motivates us to follow the randomization philosophy as described in Section \ref{rssection} and, for this purpose, we define 
$$\mathcal{G}(\mathbb{A})= \Big\{\lambda \in \mathcal{P}(\mathbb{A}); g = \frac{d\lambda}{d\text{Leb}}, g(a)\ge 0~\forall a \in \mathbb{A}\Big\}.$$

\


For  given $\lambda \in \mathcal{G}(\mathbb{A})$, let $\mathcal{H}_j$ be a set of probability kernels $\pi_j$ of the form

\begin{equation}\label{rsetd3}
\pi_j(\cdot\mid b)\ll\lambda,\qquad
\pi_j(da\mid b) = h_{j}(b,a)\,\lambda(da),
\end{equation}
where a set of density functions $h_{j}:\mathbb{H}^j\times \mathbb A\to[0,\infty)$ satisfies a bound

\begin{equation}\label{rsetd4}
\max_{0\le j\le m-1} \sup_{h_j \in \mathcal{H}_j}\sup_{b\in\mathbb H^j}\|h_{j}(b)\|_{L^\infty(\mathbb{A})} \;\le\; C(\mathcal{H}) <\infty,
\end{equation}
for a constant $C(\mathcal{H})$, where $\mathcal{H} = \mathcal{H}_0\times \ldots \times \mathcal{H}_{m-1}$. 


\

Under Assumptions (C1-C2), for each history $\Xi_{j} = \mathbf{o}_{j}$, we again recall from the proof of Theorem \ref{pdsdeth} that $\text{Law}~R_{c^{(1)}(\mathbf{o}_{j}),v^{(1)}(\mathbf{o}_{j})} << \kappa$, for every positive measure $\kappa \in \mathcal{M}_K$ with Radon-Nikodym derivative $q$. Let us define

$$\overline{r}_j(z, \mathbf{o}_{j}):= \frac{\mathcal{R}(c^{(1)}(\mathbf{o}_j),v^{(1)}(\mathbf{o}_j);z)}{q(z)},$$
for $j=m-1,\ldots, 0$ and $z \in K$. Similar to (\ref{densR2})

\begin{equation}\label{baraest}
|\overline{r}_j(x, \mathbf{o}_{j})|\le \frac{\| f_J\|_\infty}{Q_K c_{\min}},
\end{equation}
uniformly in $x \in K$, $\mathbf{o}_{j} \in \mathbb{H}^j$ and $j \in \{0, \ldots, m-1\}$.

In the sequel, we need to work with a 2-dimensional cone parameterized by the compact subset $K$ defined by (\ref{compactKRV}) and the action space $\mathbb{A}$ as follows
\[
\mathcal C
:= \{(x_1,x_2)\in\R^2;\ x_1\in K-\{0\},\ x_2/x_1\in\mathbb A\}.
\]
\begin{proposition}\label{completecase1}
Suppose Assumptions (C1-C2) are in force, consider $\lambda \in \mathcal{G}(\mathbb{A})$ and $\mu_K = q(x)dx \in \mathcal{M}_K$. Let $\xi_K$ be the probability measure on the cone $\mathcal{C}$ defined by the pushforward operation 
\[
\xi_K = \mathcal{Z}_\# \beta_K,
\]
where $\mathcal{Z}(a,z) := (z,az)^\top$ and $\beta_K(dadr):= G(a,r)\lambda(da) \mu_K(dr)$ for a disintegration kernel $G:\mathbb{A}\times K\rightarrow \mathbb{R}_+$ satisfying 

\begin{equation}\label{Gfunc}
\inf_{a\in \mathbb{A},x \in K}G(a,x)>0.
\end{equation}

Then, Assumption R1 is fulfilled for the set of probability kernels $\mathcal{H}$ described in (\ref{rsetd3}) and (\ref{rsetd4}), where  
$$\mu^{\pi_j}_{j}(dx| \mathbf{o}_j) = \rho^{\pi_j}_j(\mathbf{o}_j,x)\xi_K(dx),$$

\begin{equation}\label{rhoderC}
\rho^{\pi_j}_{j}(\mathbf{o}_j,x) := \frac{h_{j}\big(\mathbf{o}_j, \frac{x_2}{x_1}\big)\overline{r}_j(x_1, \mathbf{o}_{j})}{G\big(\frac{x_2}{x_1},x_1\big)}, 
\end{equation}
for $x=(x_1,x_2) \in \mathcal{C}$, $\mathbf{o}_j \in \mathbb{H}^j$ for $0\le j \le m-1$.  
\end{proposition}
\begin{proof}
Fix $j=m-1, \ldots, 0$. We start from the definition
\[
\mu^{\pi_j}_{j}(dx \mid \mathbf{o}_j)
=
\int_{\mathbb A}
\mathbb P\!\left[\Delta X^a_{j+1} \in dx \mid \Xi_j=\mathbf{o}_j\right]
\, \pi_j(da \mid \mathbf{o}_j).
\]

Recall that

$$
\text{Law}~[\Delta X^a_{j+1}| \Xi_j=\mathbf{o}_j]
\stackrel{(d)}{=}
\mathcal Z(a,R_{c^{(1)}(\mathbf{o}_{j}),v^{(1)}(\mathbf{o}_{j})}).
$$
Hence, for any bounded measurable test function $\eta$,
\[
\int_{\mathbb R^2}
\eta(x)\,
\mathbb P\!\left[\Delta X^a_{j+1}\in dx \mid \Xi_j=\mathbf{o}_j\right]
=
\int_K
\eta(\mathcal Z(a,z))\,
\bar{r}_j(z,\mathbf{o}_j)\,
\mu_K(dz).
\]
Assume that
\[
\pi_j(da \mid \mathbf{o}_j)
=
h_{j}(\mathbf{o}_j, a)\,\lambda(da).
\]
Then
\[
\int_{\mathbb R^2}
\eta(x)\,
\mu^{\pi_j}_{j}(dx \mid \mathbf{o}_j)
=
\int_{\mathbb A}
\int_K
\eta(\mathcal Z(a,z))\,
\bar{r}_j(z,\mathbf{o}_j)\,
\mu_K(dz)\,
h_{j}(\mathbf{o}_j,a)\,
\lambda(da).
\]

Recall that
\[
\beta_K(dadz)
=
G(a,z)\,\lambda(da)\,\mu_K(dz).
\]
Therefore,
\[
\lambda(da)\,\mu_K(dz)
=
\frac{1}{G(a,z)}\,\beta_K(dadz),
\]
which is well-defined since $\inf_{a,z} G(a,z)>0$. Substituting this identity yields

$$
\int_{\mathbb R^2}
\eta(x)\,
\mu^{\pi_j}_{j}(dx \mid \mathbf o_j)
=
\int_{\mathbb A\times K}
\eta(\mathcal Z(a,z))\,
\frac{
\bar{r}_j(z,\mathbf o_j)\,
h_{j}(\mathbf{o}_j, a)
}{
G(a,z)
}\,
\beta_K(dadz).
$$
We now claim that

\begin{equation}\label{qi1}
\int_{\mathbb A\times K}
\eta(\mathcal Z(a,z))\,
\frac{
\bar{r}_j(z,\mathbf o_j)\,
h_{j}(\mathbf{o}_j,a)
}{
G(a,z)
}\,
\beta_K(dadz)= \int_{\mathcal{C}}\eta(x)\frac{
\bar{r}_j(x_1,\mathbf o_j)\,
h_{j}(\mathbf{o}_j,\frac{x_2}{x_1})
}{
G(\frac{x_2}{x_1},x_1)
}\mathcal{Z}_{\#}\beta_K(dx)
\end{equation}
Let us denote $W_j(a,z,\mathbf{o}_j) = \frac{\bar{r}_j(z,\mathbf{o}_j)h_{j}(\mathbf{o}_j, a)}{G(a,z)}$. Indeed, by the definition of the pushforward 

$$\int_{\mathcal{C}}\psi(x) \xi_K(dx) = \int_{(\mathbb{A}\times K)=\mathcal{Z}^{-1}(\mathcal{C})}\psi (\mathcal{Z}(a,z))\beta_K(dadz)$$
for every test function $\psi$. Now, we choose

$$\psi(x) = \eta(x) W_j \Big(\frac{x_2}{x_1},x_1\Big),$$
for $x = (x_1,x_2)\in \mathcal{C}$ with $x_1 \neq 0$. Recall $
\mathcal Z(a,z) = (z,az)^\top$ and observe that 

$$\psi (\mathcal{Z}(a,z)) = \eta(\mathcal{Z}(a,z))W_j (a,z)$$
for every $(a,z) \in \mathbb{A}\times K$ with $z\neq 0$. Since $\beta_K$ is absolutely continuous, then

$$\int_{(\mathbb{A}\times K)=\mathcal{Z}^{-1}(\mathcal{C})}\psi (\mathcal{Z}(a,z))\beta_K(dadz) = \int_{(\mathbb{A}\times K)=\mathcal{Z}^{-1}(\mathcal{C})}\eta(\mathcal{Z}(a,z))W_j (a,z)\beta_K(dadz).$$
This shows (\ref{qi1}). By assumption, 
$$|\rho^{\pi_j}_j(\mathbf{o}_j,x)|= \Bigg|\frac{h_{j}\big(\mathbf{o}_j,\frac{x_2}{x_1}\big)\overline{r}_j(x_1, \mathbf{o}_{j})}{G\big(\frac{x_2}{x_1},x_1\big)}\Bigg| \le \frac{C(\mathcal{H})}{\inf_{a\in \mathbb{A},y \in K}G(a,y)} \frac{\| f_J\|_\infty}{Q_K c_{\min}},$$
uniformly in $x=(x_1,x_2)\in \mathcal{C}$, $\mathbf{o}_{j} \in \mathbb{H}^j$ and $j \in \{0, \ldots, m-1\}$. This shows that $\|\rho\|_\infty < \infty.$ By Lemma \ref{lippr}, the fact that $\mathfrak{T}_n$ is $\sqrt{n}$-Lipschitz , (\ref{rsetd4}) and (\ref{Gfunc}), we observe that $\|\rho\| < \infty$. This concludes the proof.
\end{proof}

\subsection{Rough Volatility model: Incomplete market}\label{incCASE}
This section treats  $-1 < \rho < 1$. In this case, the market is incomplete and there is an underlying 2-dimensional Brownian motion generating the filtration. Recall that in the one-dimensional case, the distribution of $\Delta A_1$ follows a $\frac{1}{2}$-Bernoulli distribution concentrated at $\{\pm \varepsilon\}$. In the 2-dimensional case, this is not the case and one has to work with the conditional distribution of $\Delta A_1$ given $\Delta T_1$ for $d=2$. 

The dynamics of $\Delta X^a_n$ given $\Xi_{n-1} = (w_0, y_0, \ldots, w_{n-1},y_{n-1})$ can be written as    

\begin{eqnarray*}
 \Delta X^a_n &=& \left( 
                  \begin{array}{c}
                    \mu_{\text{drift}} \mathfrak{T}_{n-1}(y_0, \ldots, y_{n-1} ) \\
                    a \mu_{\text{drift}} \mathfrak{T}_{n-1}(y_0, \ldots, y_{n-1} ) \\
                  \end{array}
                \right)J\\
&+& \left(
  \begin{array}{cc}
\mathfrak{T}_{n-1}(y_0, \ldots, y_{n-1}) \vartheta \circ \mathbf{z}(\bar{w}_{n-1}) & 0 \\
    0 & a \mathfrak{T}_{n-1}(y_0, \ldots, y_{n-1})  \vartheta \circ \mathbf{z}(\bar{w}_{n-1}) \\
  \end{array}
\right)
\left(
  \begin{array}{c}
    \Delta A^{(1)}_1 \\
    \Delta A^{(1)}_1 \\
  \end{array}
\right).
\end{eqnarray*}
The idea is to explore the following decomposition

\begin{eqnarray}
\nonumber\mathbb{P}\big[\Delta X_n^a \in dx' | \Xi_{n-1} = b\big]&=& \mathbb{P}\big[\Delta X_n^a \in dx' | \Xi_{n-1} = b, \Delta T_n = \Delta^2_n \big]\mathbb{P}[\Delta T_n=\Delta^2_n| \Xi_{n-1}=b] \\
\label{2rv}&+& \mathbb{P}\big[\Delta X_n^a \in dx' | \Xi_{n-1} = b, \Delta T_n = \Delta^1_n \big]\mathbb{P}[\Delta T_n=\Delta^1_n| \Xi_{n-1}=b]\\
\nonumber&=&\frac{1}{2} \mathbb{P}\big[\Delta X_n^a \in dx' | \Xi_{n-1} = b, \Delta T_n = \Delta^2_n \big]\\
\nonumber&+& \frac{1}{2}  \mathbb{P}\big[\Delta X_n^a \in dx' | \Xi_{n-1} = b, \Delta T_n = \Delta^1_n \big],
\end{eqnarray}
where we use the fact that $\Delta T_n$ is independent of $\Xi_{n-1}$ and  

$$\mathbb{P}[\Delta T_n=\Delta^1_n| \Xi_{n-1}=b] = \mathbb{P}[\Delta T_n=\Delta^1_n]=\frac{1}{2} = \mathbb{P}[\Delta T_n=\Delta^2_n| \Xi_{n-1}=b] = \mathbb{P}[\Delta T_n=\Delta^2_n].$$

The conditional law of $\Delta X^a_n$ knowing $\Xi_{n-1}$ and $\Delta T_n=\Delta^2_{n}$ is given by   

\begin{eqnarray}
\nonumber \Delta X^a_n &=& \left( 
                  \begin{array}{c}
                    \mu_{\text{drift}} \mathfrak{T}_{n-1}(y_0, \ldots, y_{n-1} ) \\
                    a \mu_{\text{drift}} \mathfrak{T}_{n-1}(y_0, \ldots, y_{n-1}) \\
                  \end{array}
                \right)J \\
\label{condlxad}&+& \left(
  \begin{array}{cc}
\mathfrak{T}_{n-1}(y_0, \ldots, y_{n-1} )\vartheta\circ \mathbf{z}(\bar{w}_{n-1}) & 0 \\
    0 & a \mathfrak{T}_{n-1}(y_0, \ldots, y_{n-1} ) \vartheta \circ \mathbf{z}(\bar{w}_{n-1}) \\
  \end{array}
\right)
\left(
  \begin{array}{c}
     L \\
     L\\
  \end{array}
\right),
\end{eqnarray}
where $L\stackrel{d}{=}\Delta A_1| \{\Delta T_1 = \Delta^2_1\}$.  
The conditional law of $\Delta X^a_n$ knowing $\Xi_{n-1}$ and $\Delta T_n=\Delta^1_{n}$ is given by   

\begin{eqnarray}
\nonumber \Delta X^a_n &=& \left( 
                  \begin{array}{c}
                    \mu_{\text{drift}} \mathfrak{T}_{n-1}(y_0, \ldots, y_{n-1} ) \\
                    a \mu_{\text{drift}} \mathfrak{T}_{n-1}(y_0, \ldots, y_{n-1}) \\
                  \end{array}
                \right)J\\
\label{condlxa}&+& \left(
  \begin{array}{cc}
\mathfrak{T}_{n-1}(y_0, \ldots, y_{n-1} )\vartheta\circ \mathbf{z}(\bar{w}_{n-1})  & 0 \\
    0 & a \mathfrak{T}_{n-1}(y_0, \ldots, y_{n-1} )\vartheta\circ \mathbf{z}(\bar{w}_{n-1})  \\
  \end{array}
\right)
\left(
  \begin{array}{c}
     \mathbf{B} \\
     \mathbf{B}\\
  \end{array}
\right). 
\end{eqnarray}
The analysis of (\ref{condlxa}) follows the previous section. We only need to study (\ref{condlxad}). We will assume the following assumptions:

\begin{itemize}


\item[(I1)] $\vartheta:\mathbb{R}_+\rightarrow \mathbb{R}_+$ is a bounded Lipschitz function: There exist $ \theta_{\min}$ and $\theta_{\max}$  such that 

$$0 < \theta_{\min} \le |\vartheta(z)|\le \theta_{\max},$$
for every $z\ge 0$. 
  
\item[(I2)]  We assume $0 < \underline{M}\le J \le \overline{M}$ a.s.\ for some $0 < \underline{M} < \overline{M}$; equivalently, the law of $J$ is the truncation of $\Delta T_1$ to $[\underline{M},\overline{M}]$.
\end{itemize}

For this purpose, inspired by the complete market case, let us now consider the wealth process subject to uncertain parameters  
$$R^L_{c,v}:=c J + v L,$$
for $(c,v)$ satisfying 

\begin{equation}\label{RLdomain}
0 < s_{\min}|\mu_{\text{drift}}| \le|c| \le s_{\max}|\mu_{\text{drift}}|,\quad 0 <\theta_{\min} s_{\min}\le  |v|\le \theta_{\max} s_{\max}.
\end{equation}
We denote 
$$\mathcal{R}^L(c,v;\cdot):=\text{density function of}~R^L_{c,v}$$ 
for $(c,v) \in \mathbb{R}^2-\{0,0\}$ and  $L\stackrel{d}{=}\Delta A_1| \{\Delta T_1 = \Delta^2_1\}$. Since $0 < \underline{M} \le J \le \overline{M}$ and $|L|< \varepsilon$ a.s., we observe we shall fix a compact subset 

\begin{equation}\label{Ksetinc}
K= \Big[-v_{\max} \varepsilon - c_{\max}\overline{M}, v_{\max} \varepsilon + c_{\max} \overline{M}\Big]
\end{equation}
with constants $0 < c_{\min} < c_{\max}$, $0 < v_{\min} < v_{\max}>0$ defined by 

\begin{equation}\label{cvparIN}
0 < c_{\min}= s_{\min}|\mu_{\text{drift}}| < s_{\max}|\mu_{\text{drift}}|= c_{\max},\quad 0< v_{\min} = s_{\min}\theta_{\min} <  s_{\max}\theta_{\max}= v_{\max} <\infty.
\end{equation}
By construction, 
$$\text{Range}~(R^L_{c,v}) \subset K\subset \mathbb{R}~a.s.,$$
whenever $0  < c_{\min}\le |c|\le c_{\max}$ and $0 < v_{\min}\le |v|\le v_{\max}$. See Remark \ref{rangeR}. In the sequel, the compact set $K$ is fixed and Assumptions (I1-I2) are in force. In the sequel, the Gaussian kernel is denoted by 
$$\phi_t(x):= \frac{1}{\sqrt{2\pi t}}e^{\frac{-x^2}{2t}}$$ 
for $t>0$ and $x \in \mathbb{R}$. The conditional law of $L$ given $\{\Delta T_1 = \Delta^2_1=t\}$ is expressed in terms of 

\begin{eqnarray*}
\label{cond0}\mathbb{P}\Big[ B^{1}(T_1) \in E\big| \Delta T_1= \Delta^2_1=t\Big] &=& \mathbb{P}\Big[B^1(t) \in E\big| \sup_{0\le s\le t}|B^1(s)| < \varepsilon\Big]\\ 
&=& \frac{\mathbb{P}[B^1(t) \in E, \sup_{0\le s\le t}|B^1(s)|  < \varepsilon]}{\mathbb{P}[\sup_{0\le s\le t}|B^1(s)|  < \varepsilon]},
\end{eqnarray*}
for every Borel set $E$. The law of $\mathbb{P}[B^1(t) \in E, \sup_{0\le s\le t}|B^1(s)|  < \varepsilon]$ is based on the trivariate distribution of the final, minimal and maximal value (see e.g. \cite{borodin} p. 174)

\begin{equation}\label{stv1}
\mathbb{P}\Big[B^1(t) \in dx, \sup_{0\le s\le t}|B^1(s)|  < \varepsilon\Big]=\sum_{n=-\infty}^{+\infty} \Big[ \phi_t (x-4n\varepsilon) -  \phi_t(x-2\varepsilon -4n\varepsilon)\Big]dx,
\end{equation}
for $-\varepsilon < x < \varepsilon$. Observe we can write the density in (\ref{stv1}) as follows (see e.g Th 2.2 in \cite{riedel})

$$
\sum_{n=-\infty}^{+\infty} \Big[ \phi_t (x-4n\varepsilon) -  \phi_t(x-2\varepsilon -4n\varepsilon) \Big]
$$
\begin{equation}\label{stv2}
= \phi_t(x) + Y(t,x,\varepsilon), 
\end{equation}
where 
\begin{eqnarray}\label{Yseries} 
Y(t,x,\varepsilon)&:=& \sum_{m=1}^\infty [\phi_t(x+4m\varepsilon) - \phi_t(x+2\varepsilon+4(m-1)\varepsilon) ]\\
\nonumber& +& \sum_{m=1}^\infty [\phi_t(x-4m\varepsilon) - \phi_t(x-2\varepsilon-4(m-1)\varepsilon) ],
\end{eqnarray}
for $-\varepsilon < x < \varepsilon$. Therefore the conditional density of $L$ given $J= \Delta^2_1 = t$ is
\begin{equation}\label{eq:cond_density_G}
f_{L\mid J=\Delta^2_1=t}(x)
= \frac{\phi_t(x) + Y(t,x,\varepsilon)}
       {p(t,\varepsilon)}\,\mathds{1}_{\{|x|<\varepsilon\}},
\end{equation}
where we set $p(t,\varepsilon) = \int_{-\varepsilon}^{\varepsilon}\big[\phi_t(y) + Y(t,y,\varepsilon)\big]dy$. By disintegrating the joint law $(J,L)$ onto $J$, we observe that for any bounded Borel function $U:\mathbb{R}\to\mathbb{R}$, we have 
\[
\mathbb{E}\big[U(R^L_{c,v})\big]
= \int_0^\infty \mathbb{E}\big[U(ct + vL)\mid J=\Delta^2_1 = t\big]\,f_J\mathds{1}_{[\underline{M},\overline{M}]}(t)dt,
\]
for every $(c,v)$ satisfying (\ref{cvparIN}). Using \eqref{eq:cond_density_G}, this becomes
\[
\mathbb{E}\big[U(R^L_{c,v})\big]
= \int_0^\infty\!\int_{-\varepsilon}^{\varepsilon}
   U(ct + vx)\,
   \frac{\phi_t(x)+Y(t,x,\varepsilon)}{p(t,\varepsilon)}\,dx\,
   f_J\mathds{1}_{[\underline{M},\overline{M}]}(t)dt.
\]

Hence, the law of $R^L_{c,v}$ is absolutely continuous, for every $(c,v)$ satisfying (\ref{cvparIN}). Its density $\mathcal{R}^L(c,v; \cdot)$ is given by
\begin{eqnarray}
\label{densRL}\mathcal{R}^L(c,v;z) &=& \int_0^\infty \frac{1}{|v|}f_{L|J=\Delta^2_1=t}\Bigg(\frac{z-ct}{v} \Bigg)\mathds{1}_{\left\{\left|\frac{z-ct}{v}\right|<\varepsilon\right\}} f_J\mathds{1}_{[\underline{M},\overline{M}]}(t)dt\\
\nonumber&=& \int_0^\infty
     \frac{1}{|v|}\,
     \frac{
       \phi_t\!\left(\frac{z-ct}{v}\right)
       + Y\!\left(t,\frac{z-ct}{v},\varepsilon\right)
     }{
       p(t,\varepsilon)
     }\,
     \mathds{1}_{\left\{\left|\frac{z-ct}{v}\right|<\varepsilon\right\}}\, 
    f_J\mathds{1}_{[\underline{M},\overline{M}]}(t)dt,
\end{eqnarray}
for $z\in K \subset \mathbb{R}$ and $(c,v)$ satisfying (\ref{RLdomain}).

The attentive reader may wonder why the introduction of the variable $L$ in the incomplete market case implies the term $\frac{1}{v}$ in (\ref{densRL}) for $v$ as in (\ref{RLdomain}). 
\begin{remark}\label{discDENS}
In the one-dimensional complete market setting,
\[
R_{c,v} = cJ + v\mathbf{B},
\qquad 
\mathbf{B}\in\{\pm \varepsilon\},
\qquad
\mathbb{P}(\mathbf{B}=\pm\varepsilon)=\tfrac12.
\]
Conditioning on $\mathbf{B}$, we obtain
\[
R_{c,v}\mid(\mathbf{B}=\pm\varepsilon)
= cJ \pm v\varepsilon.
\]
Therefore, 
\begin{equation}\label{FUNDden}
\mathcal{R}(c,v;z) 
= \frac{1}{|c|}
\left[
\frac12\, f_J\!\left(\frac{z - v\varepsilon}{c}\right)
+
\frac12\, f_J\!\left(\frac{z + v\varepsilon}{c}\right)
\right].
\end{equation}

The Jacobian factor $\frac{1}{|c|}$ appears because: The only continuous variable is $J$ where the change of variables is performed. Moreover, $\mathbf{B}$ is a finite range discrete random variable and therefore does not contribute any Jacobian. In the two-dimensional incomplete market setting, we consider
\[
R^L_{c,v} = cJ + vL,
\]
where:
\begin{itemize}
\item $J\stackrel{(d)}{=}\Delta T_1\stackrel{(d)}{=}\Delta^2_1$ is the first time the \emph{second} coordinate hits $\{\pm\varepsilon\}$;
\item $L=B^1(J)$ is the value of the first coordinate at that time;
\item Conditionally on $J=t$, the law of $L$ is absolutely continuous with support
      $(-\varepsilon,\varepsilon)$.
\end{itemize}
In this case, we have an explicit expression for the conditional density
$f_{L\mid J=t}$. Conditioning on $J=t$ yields
\[
R^L_{c,v}\mid(J=t) = ct + vL.
\]

Thus, for fixed $t$, the change of variables is performed with respect to $L$, and $z = ct + vx$. As a consequence, the density $\mathcal{R}^L(c,v;\cdot) $ of $R^L_{c,v}$ admits the representation (\ref{densRL}). Here the Jacobian factor $\frac{1}{|v|}$ appears because: The continuous variable being transformed is $L$. The variable $J$ acts only as a parameter in the conditional law. The change of variables is therefore performed in the $L$-direction. The difference between the factors $\frac{1}{|c|}$ in the one-dimensional case
and $\frac{1}{|v|}$ in the two-dimensional case reflects
the fact that: In one dimension, the continuous randomness comes from the hitting time $J$. In two dimensions, the continuous randomness comes from the spatial increment $L$.
\end{remark}

Factoring out the common direction $(1,a)^{\!\top}$ gives
\begin{equation}\label{rldeIN}
\text{Law}[\Delta X^a_j|\Xi_{j-1}=\mathbf{o}_{j-1}, \Delta T_j=\Delta^2_j]\stackrel{d}{=} R^L_{c^{(1)}(\mathbf{o}_{j-1}),v^{(1)}(\mathbf{o}_{j-1})}\left(
                                                                                   \begin{array}{c}
                                                                                     1 \\
                                                                                     a \\
                                                                                   \end{array}
                                                                                 \right),
\end{equation}
for $j=m,\ldots, 1$. Then, the standard trick (for $c=c^{(1)}_{\mathbf{o}_j}$ and $v=v^{(1)}_{\mathbf{o}_j}$)

$$\frac{d \text{Law}~R^L_{c^{(1)}(\mathbf{o}_{j}),v^{(1)}(\mathbf{o}_{j})}}{d\mu_{K}}(z) =  \frac{\mathcal{R}^L(c^{(1)}(\mathbf{o}_{j}),v^{(1)}(\mathbf{o}_{j});z) }{q(z)}, $$
yields $\text{Law}~R^L_{c^{(1)}(\mathbf{o}_{j}),v^{(1)}(\mathbf{o}_{j})}<< \mu_{K}$ for every $\mu_{K} \in \mathcal{M}_K$ with a Radon-Nikodym derivative $q$. In order to shorten notation, we set 

\begin{equation}\label{densINC}
\overline{m}_j (z,\mathbf{o}_j):= \frac{\mathcal{R}^L(c^{(1)}(\mathbf{o}_j), v^{(1)}(\mathbf{o}_j);z)}{q(z)},
\end{equation}
for $j=m-1,\ldots, 0$, $z \in K$ and $0  < c_{\min}\le |c|\le c_{\max}$ and $0 < v_{\min}\le |v|\le v_{\max}$.

\begin{lemma}\label{Ylemma}
Let $(\underline{M}, \overline{M})$ be the constants defined in Assumption I2, $\varepsilon \in (0,1)$ and let $Y$ be defined as in (\ref{Yseries}). Then,  
\label{lem:L-uniform-bound}
\[
\sup_{\underline M \le t \le \overline M}\ 
\sup_{|x| < \varepsilon}\ |Y(t,x,\varepsilon)| < \infty,
\]
$Y$ is continuously differentiable in $x$ and 
\[
\sup_{\underline{M}\le t\le \overline{M}} \ \sup_{|x|<\varepsilon}
\big|\partial_x Y(t,x,\varepsilon)\big| < \infty.
\]
\end{lemma}

By Lemma \ref{Ylemma}, we observe

\begin{equation}\label{mbarbound}
|\overline{m}_j (z,\mathbf{o}_j)|\le \frac{1}{v_{\min}Q_K\inf_{\underline{M}\le t\le \overline{M}}p(t,\varepsilon)}  \Bigg\{\frac{1}{\sqrt{2\pi\underline{M}}}+\sup_{\underline M \le t \le \overline M}\ 
\sup_{|x| < \varepsilon}\ |Y(t,x,\varepsilon)|\Bigg\}, 
\end{equation}
for $0\le j\le m-1$ and $z \in K$. 

\begin{lemma}\label{glemma}
Under assumption I2, there exists a constant $C>0$, depending on the constants  
$(K,\varepsilon,c_{\min},c_{\max},v_{\min},v_{\max},\underline{M},\overline{M})$, such that
\[
|\mathcal{R}^L(c,v;x)|+
\big|\partial_v \mathcal{R}^L(c,v;x)
\big|
+
\big|
\partial_c \mathcal{R}^L(c,v;x)
\big|
\;\le\; C,
\]
for every $x \in K, 0 < v_{\min}\le |v|\le v_{\max}, 0 < c_{\min}\le |c|\le c_{\max}$. 
\end{lemma}

We postpone the proofs of Lemmas \ref{Ylemma} and \ref{glemma} to section \ref{appendix4}.

\begin{theorem}\label{rvoltheorem}
Suppose Assumptions (I1-I2) are in force, consider $\lambda \in \mathcal{G}(\mathbb{A})$ and $\mu_{K} = q(x)dx \in \mathcal{M}_K$, where $K$ is defined by (\ref{Ksetinc}). Let $\xi_K$ be the probability measure on the cone $\mathcal{C}$ defined by the pushforward operation 
\[
\xi_K = \mathcal{Z}_\# \beta_K,
\]
where $\mathcal{Z}(a,z) = (z,az)^\top$ and $\beta_K(dadr)= G(a,r)\lambda(da) \mu_{K}(dr)$ for a disintegration kernel $G:\mathbb{A}\times K\rightarrow \mathbb{R}_+$ satisfying 

$$\inf_{a\in \mathbb{A},x \in K}G(a,x)>0.$$

Then, Assumption R1 is fulfilled for the set of probability kernels $\mathcal{H}$ described in (\ref{rsetd3}) and (\ref{rsetd4}), where  
$$\mu^{\pi_j}_{j}(dx| \mathbf{o}_j) = \rho^{\pi_j}_j(\mathbf{o}_j,x)\xi_K(dx)$$
and with importance sampling weight given by  
$$\rho^{\pi_j}_{j}(\mathbf{o}_j,x) = \frac{1}{2}\frac{h_{j}\big(\mathbf{o}_j, \frac{x_2}{x_1}\big)\{\overline{r}_j(x_1, \mathbf{o}_{j}) + \overline{m}_j(x_1,\mathbf{o}_j)\}}{G\big(\frac{x_2}{x_1},x_1\big)}, $$
for $x=(x_1,x_2) \in \mathcal{C}$, $\mathbf{o}_j \in \mathbb{H}^j$ for $0\le j \le m-1$.  
\end{theorem}
\begin{proof}
Write 
\begin{eqnarray*}
\nonumber\mathbb{P}\big[\Delta X_n^a \in dx' | \Xi_{n-1} = \mathbf{o}_{n-1}\big]&=& \frac{1}{2} \mathbb{P}\big[\Delta X_n^a \in dx' | \Xi_{n-1} =  \mathbf{o}_{n-1}, \Delta T_n = \Delta^2_n \big]\\
\nonumber&+& \frac{1}{2}  \mathbb{P}\big[\Delta X_n^a \in dx' | \Xi_{n-1} =  \mathbf{o}_{n-1}, \Delta T_n = \Delta^1_n \big].
\end{eqnarray*}
By Proposition \ref{completecase1}, we only need to treat the conditional law given the information set $\Delta T_n=\Delta^2_n$. By using (\ref{densINC}), (\ref{mbarbound}) and Lemma \ref{glemma}, the proof is entirely similar to the one written in Proposition \ref{completecase1}. For sake of conciseness, we omit the details.  
\end{proof}
\section{Monte Carlo numerical scheme and adaptive stochastic control}\label{mainressec}
This section presents the Monte Carlo schemes for solving the stochastic control problem (\ref{INTROpr3}) via the dynamic programming equation (\ref{valuefunc}) as considered in Theorem \ref{thp}. More importantly, we will provide a Monte Carlo scheme for a stochastic control problem

$$\inf_{u \in U^T_0} \mathbb{E}[\varphi(X^{\theta^\star,u}(T))],$$
for an underlying controlled state $X^{\theta^\star,u}$ which depends on a fixed unknown deterministic parameter $\theta^\star \in \Theta$ for a compact subset $\Theta \subset \mathbb{R}^p$. Based on the result of Section \ref{MSECTION}, the agent generates a model-independent training dataset from a dominating reference measure and, via importance sampling, reinterprets this dataset under successive parameter estimates to learn and update value functionals and optimal controls, yielding a scalable Monte Carlo procedure robust to model uncertainty.

In order to estimate optimal controls $C_{j}:\mathbb{H}^{j}\rightarrow\mathbb{A}$ in (\ref{explicitcontrolTEXT}), we will make use of two Feedforward Neural Networks that we describe as follows. In order to establish the rate of convergence, we need to impose boundedness and Lipschitz property on the terminal condition $\varphi:\mathbb{R}^q\rightarrow \mathbb{R}$

\

\noindent \textbf{Assumption T1:} $\varphi:\mathbb{R}^q\rightarrow \mathbb{R}$ is bounded and globally Lipschitz.

\

The numerical scheme is based on approximating the controls and value functionals via two Neural Networks that we briefly described as follows. For a given $1\le \ell \le m-1$,

\begin{equation}\label{controlNN}
\mathcal{C}^{\eta,\delta}_{N}(\ell):= \Big\{ \mathbf{o}_\ell \mapsto A(\mathbf{o}_\ell; \beta) = \big( A_1(\mathbf{o}_\ell; \beta), \ldots, A_p(\mathbf{o}_\ell; \beta) \big)\in \mathbb{A}\Big\},
\end{equation}
where 

\begin{equation}\label{controlf}
A_i (\mathbf{o}_\ell; \beta):= \textbf{a}\Bigg(\sum_{j=1}^N c_{ij}\Big(\langle a_{ij}, \mathbf{o}_\ell\rangle + b_{ij} \Big)^+ + c_{0j} \Bigg); 1\le i\le p
\end{equation}
$$\beta = \big(a_{ij}, b_{ij},c_{ij}\big)_{1\le i\le p, 1\le j\le N},~a_{ij} \in \mathbb{R}^\ell,~\|a_{ij}\|\le \eta,~b_{ij},~c_{ij} \in \mathbb{R},~\sum_{j=1}^N c_{ij}\le \delta.$$
Note that the considered Neural Networks have
one hidden layer, $N$ neurons, Relu activation functions, and $\textbf{a}: \mathbb{R}\rightarrow \mathbb{R}$ as an output layer chosen case-by-case. The constants $\delta$ and $\eta$ are often referred as, respectively, total variation and kernel.

We shall consider the following set of Neural Networks for the value function approximation: For a given $1\le \ell \le m-1$,

\begin{equation}\label{valueNN}
\mathcal{V}^{\eta,\delta}_{N}(\ell):= \Big\{ \mathbf{o}_\ell \mapsto \Phi(\mathbf{o}_\ell; \theta)\Big\}, 
\end{equation}
where
$$\Phi(\mathbf{o}_\ell;\theta):= \sum_{i=1}^N c_{i}\psi\big(\langle a_{i}, \mathbf{o}_\ell\rangle + b_{i} \big) + c_{0},$$

where 

$$\theta = (a_i,b_i,c_i),\quad \|a_i\|\le \eta,\quad b_i \in \mathbb{R},\quad \sum_{i=1}^N|c_i|\le \delta.$$
We suppose $\psi$ to be a globally Lipschitz function with the seminorm $\|\psi\|$. Let $N_M, \eta_M$ and $\delta_M$ be sequences of integers such that

\begin{equation}\label{setcCV}
N_M, \eta_M, \delta_M\rightarrow +\infty \quad \text{and}\quad \delta^4_M N_M \frac{\log(M)}{M} + \frac{\rho^2_M\delta^8_M\eta^{6}_M }{M}\longrightarrow 0, 
\end{equation}
as $M\rightarrow +\infty$. For each $1\le \ell \le m-1$ and $M\ge 1$, we set  

\begin{equation}\label{NNet}
\mathcal{C}_{M}(\ell):=\mathcal{C}^{\eta_M,\delta_M}_{N_M}(\ell)\quad \text{and}\quad \mathcal{V}_{M}(\ell):=\mathcal{V}^{\eta_M,\delta_M}_{N_M}(\ell).
\end{equation}

\begin{remark}
Throughout this paper, in order to simplify the implementation of the algorithm, we will fix a projection $\textbf{pr}_\ell: \mathbb{H}^{m-1}\rightarrow \mathbb{H}^{\ell}$ and from $(\mathcal{C}_{M}(m-1), \mathcal{V}_{M}(m-1))$ as above, we choose

\begin{equation}\label{projNN}
\mathcal{C}_{M}(\ell)= \mathcal{C}_{M}(m-1)|_{\ell}, \quad \mathcal{V}_{M}(\ell)= \mathcal{V}_{M}(m-1)|_{\ell}, 
\end{equation}
where $\mathcal{C}_{M}(\ell)|_{\ell}=\{\mathbf{o}_{\ell} \mapsto A(\textbf{o}_\ell; \beta); \textbf{pr}_\ell(\mathbf{o}_{m-1})=\textbf{o}_\ell, A\in \mathcal{C}_M(m-1)\}$ for $1\le \ell\le m-1$. Similarly, $\mathcal{V}_{M}(\ell)|_{\ell}=\{\mathbf{o}_{\ell} \mapsto \Phi(\textbf{o}_\ell; \theta); \textbf{pr}_\ell(\mathbf{o}_{m-1})=\textbf{o}_\ell, \Phi\in \mathcal{V}_M(m-1)\}$ for  $1\le \ell\le m-1$. 
\end{remark}

\begin{remark}
In our discrete-time dynamic programming recursion, the control at each step is defined through infimum over the admissible action space. Approximating only the value function would therefore require solving a nested optimization problem at each state and may amplify regression errors through the infimum operator, which is generally not stable under perturbations. By parameterizing both the value functions and the control policies, we avoid this additional inner optimization and obtain a numerically stable scheme in which policy and value approximation errors can be analyzed jointly.
\end{remark}

The main results of this section, Theorem \ref{mainresult} and Theorem \ref{mainresultRS}, distinguish between non-randomized and randomized strategies. 
\subsection{Rate of Convergence for non-randomized strategies}\label{detTH}

 We are now in position to describe the Monte Carlo numerical scheme associated with the dynamic programming equation (\ref{valuefunc}). In order to estimate $(\mathbb{V}_j)_{j=0}^{m-1}$, we will make use of a training data denoted by

\begin{equation}\label{Ondef}
O_n= \big(\mathcal{W}_1, Y_1, \ldots, \mathcal{W}_n, Y_n\big); 1\le n\le m.
\end{equation}
For each $C \in \mathbb{A}^{\mathbb{H}^{m-1}}$, let us denote

\begin{equation}\label{Xc}
\mathcal{X}^C_n:=  (\mathcal{W}_n, \Delta X^C_n),
\end{equation}
where, for a given past information $O_{n-1}$, we set 

\begin{equation}\label{deltacxn} 
\Delta X^C_n:= \blacktriangle x_n \big(\pi_2(O_{n-1}),C(O_{n-1}) , \bar{W}^{(1)} ,\ell_n(\pi_3(O_{n-1}),\bar{W})\big),
\end{equation}
for $n=m, \ldots, 1$. Here, $\bar{W} = \big(\bar{W}^{(1)},\bar{W}^{(2)}\big)\stackrel{d}{=} \mathcal{W}_n\stackrel{d}{=}\mathcal{W}_1$ and independent of $O_{n-1}$, for every $n\ge 1$. 



\begin{remark}
We make a slight abuse of notation in (\ref{Xc}) by interpreting $C(O_{n-1})$ as a restriction of $C \in \mathbb{A}^{\mathbb{H}^{m-1}}$ onto his first $n-1$ variables.
\end{remark}
\begin{remark}
Whenever we write $\Delta X^C_n$, it will be implicit that we are computing it following (\ref{deltacxn}) with a given transition function $\blacktriangle x_n$ and based on an underlying past information $O_{n-1}$ which is usually totally clear from the context. This will keep notation simple and we employ (\ref{deltacxn}) with a slight abuse of notation.  
\end{remark}
Recall 

$$\mathbb{V}_{m}(\mathbf{o}_{m}) := \varphi\Big(x_0+ \sum_{i=1}^{m} y_i\Big),$$ 
for $\mathbf{o}_m=(w_1,y_1, \ldots, w_m,y_m)$.

\

\noindent Terminal condition: $\widehat{\mathbb{V}}^M_m:= \mathbb{V}_m$.  

\

\begin{enumerate}
  \item Compute the approximated control at time $n$

\begin{equation}\label{ERMoptc3}
\hat{a}^M_{n}\in \argmin_{C \in \mathcal{C}_M(n)}\mathbb{E}\Bigg[ \widehat{\mathbb{V}}^M_{n+1}\Big(O_{n}, \mathcal{X}^{C}_{n+1}\Big)\Bigg].
\end{equation}

  \item compute the estimation of the value function at time $n$
  
\begin{equation}\label{ERMoptc4}
\widetilde{\mathbb{V}}^M_{n} \in \argmin_{\Phi \in \mathcal{V}_M(n)} \mathbb{E} \Big[ \widehat{\mathbb{V}}^M_{n+1} \big(O_{n}, \mathcal{X}^{\hat{a}^M_{n}}_{n+1}\big) - \Phi (O_{n})  \Big]^2,
\end{equation}
\end{enumerate} 
where $$\widehat{\mathbb{V}}^M_n:= \max \Big\{ \min (\widetilde{\mathbb{V}}^M_n; \| \mathbb{V}_n\|_\infty); -\| \mathbb{V}_n\|_\infty\Big\},$$
for $n=m-1, \ldots, 1, 0$. 

\

\begin{remark}\label{Ondisc}
The data $O_n$ in (\ref{ERMoptc3}) and (\ref{ERMoptc4}) is generated by the law given by the $n$-fold product measure of $\nu \otimes\mu$ and $(O_n,\mathcal{X}^C_{n+1})$ follows (\ref{Ondef}), (\ref{Xc}) and (\ref{deltacxn}). The Monte Carlo value functions are computed by using samples $O^{(p)}_n$ of $O_n$ and $\mathcal{X}^{\hat{a}^M_n,(p)}_{n+1}$ of $\mathcal{X}^{\hat{a}^M_n}_{n+1}$, for $1\le p\le M$. The approximated policy $\hat{a}_n$ is estimated by using a training sample $O^{(p)}_n, \mathcal{W}^{(p)}_{n+1}; p=1, \ldots, M$ of $O_n,\mathcal{W}_{n+1}$ to simulate $O_n,\mathcal{X}^C_{n+1}$ for $C \in \mathcal{C}_M(n)$. The optimization in steps (1) and (2) above can be accomplished by stochastic gradient descent methods. This produces a Monte Carlo estimator $\hat{a}^M_n$.   
\end{remark}

\begin{theorem}\label{mainresult}
Assume hypotheses E1, H0-H1-H2 and T1. Assume there exists an optimal feedback control $(\phi^{\text{opt}}_\ell)_{\ell=n}^{m-1}$ for the control problem with value functionals $\mathbb{V}_n$, for $n=0,\ldots, m-1$. Then  

\begin{eqnarray*}
\mathbb{E}_M|\widehat{\mathbb{V}}^M_n(O_n) - \mathbb{V}_n(O_n)|&=& \mathcal{O}_{\mathbb{P}}\Bigg( \Big(\delta^4_M N_M \frac{\log(M)}{M}\Big)^{2^{-(m-n)}}+ \Big( \rho_M^2 \delta_M^8 \eta^6_M \| \psi\|^2 M^{-1} \Big)^{2^{-(m-n+1)}}\\
&+& \Big(\max_{n\le \ell \le m} \inf_{\Phi \in \mathcal{V}_M(\ell)}\| \Phi(O_\ell) - \mathbb{V}_\ell(O_\ell)\|_{M,2}\Big)^{2^{-(m-n+1)}}\\
&+& \Big(\max_{n\le \ell \le m} \inf_{G \in \mathcal{C}_M(\ell)}\sqrt{\| G(O_\ell) - \phi^{\text{opt}}_\ell(O_\ell)\|_{M,1}}~\Big)^{2^{-(m-n-1)}}\Bigg),
\end{eqnarray*}
for $n=m-1, \ldots, 0$. Here, $\mathbb{E}_M[\cdot]$ stands for the expectation conditioned by the training set used to estimate the optimal policies $(\hat{a}^M_j); 0\le j\le m-1$ in (\ref{ERMoptc4}) and $\|\cdot\|^r_{M,r}:= \mathbb{E}_M[|\cdot|^r]$ for $r=1,2$. 
\end{theorem}

\begin{remark}
By using universal approximation theorems (see e.g. Chapter 9 in \cite{carlin} and Prop. 4.8 in \cite{hure}), one can check that if $\phi^{\text{opt}}_\ell \in L^2(\bigotimes_{i=1}^\ell (\nu\otimes \mu))$, then 

$$\max_{n\le \ell \le m} \inf_{G \in \mathcal{C}_M(\ell)}\mathbb{E}| G(O_\ell) - \phi^{\text{opt}}_\ell(O_\ell)|\rightarrow 0,$$
as $M\rightarrow +\infty$. Similar discussion applies to the approximation error term $\max_{n\le \ell \le m} \inf_{\Phi \in \mathcal{V}_M(\ell)}\| \Phi(O_\ell) - \mathbb{V}_\ell(O_\ell)\|_{M,2}$ used to estimate the value functionals.  
\end{remark}

We postpone the proof of Theorem \ref{mainresult} to Sections \ref{appendix1} and \ref{appendix2}.

\subsubsection{Rate of Convergence for Randomized Strategies} \label{ranTH}

Motivated by the problem of partial hedging in rough volatility models, we are now in position to state the convergence rates of the Monte Carlo algorithm for the backward dynamic programming equations (\ref{randomDP}). Let us fix a set of randomized strategy $\mathcal{H}$ of the form (\ref{rsetd0}) and (\ref{rsetd1}). We first observe that since the elements of $\mathcal{H}_j$ are given by absolutely continuous probability kernels, any $\kappa: \mathbb{H}^{j}\rightarrow \mathcal{P}(\mathbb{A}) \in \mathcal{H}_j$ is uniquely identified by a density $g: \mathbb{H}^{j}\times \mathbb{A}\rightarrow \mathbb{A}$ and the natural estimator for the optimal control at step $m-1$ is given by 

$$a^{\mathcal{H}}_{m-1} \in \argmin_{g \in \mathcal{H}_{m-1}}\mathbb{E} \big\langle \varphi(O_{m-1},\mathcal{X}^{\cdot}_{m}), g(O_{m-1},\cdot) \big \rangle.$$
The algorithm is then given by the following backward scheme:  Recall

$$\mathbb{V}^{\mathcal{H}}_{m}(\mathbf{o}_{m}) = \varphi\Big(x_0+ \sum_{i=1}^{m} y_i\Big),$$ 
for $\mathbf{o}_m=(w_0,y_0, \ldots, w_m,y_m)$.

\

\noindent Terminal condition: $\widehat{\mathbb{V}}^{\mathcal{H},M}_m:= \mathbb{V}^{\mathcal{H}}_m$.  

\

\begin{enumerate}
  \item Compute the approximated control at time $n$

\begin{equation}\label{ERMoptc5}
\hat{a}^{\mathcal{H},M}_{n}\in \argmin_{g \in \mathcal{H}_n}\mathbb{E}\Big\langle \widehat{\mathbb{V}}^{\mathcal{H},M}_{n+1}\Big(O_{n}, \mathcal{X}^{\cdot}_{n+1}\Big), g(O_n,\cdot)\Big\rangle.
\end{equation}

  \item compute the estimation of the value function at time $n$
  
\begin{equation}\label{ERMoptc6}
\widetilde{\mathbb{V}}^{\mathcal{H},M}_{n} \in \argmin_{\Phi \in \mathcal{V}_M(n)} \mathbb{E} \Big[ \Big| \Big\langle \widehat{\mathbb{V}}^{\mathcal{H},M}_{n+1} \big(O_{n}, \mathcal{X}^{\cdot}_{n+1}\big), \hat{a}^{\mathcal{H},M}_{n}(O_n,\cdot)\Big\rangle - \Phi (O_{n})  \Big|^2\Big],
\end{equation}
\end{enumerate} 
where $$\widehat{\mathbb{V}}^{\mathcal{H},M}_n:= \max \Big\{ \min (\widetilde{\mathbb{V}}^{\mathcal{H},M}_n; \| \mathbb{V}^{\mathcal{H}}_n\|_\infty); -\| \mathbb{V}^{\mathcal{H}}_n\|_\infty\Big\},$$
for $n=m-1, \ldots, 1, 0$. 

\

The same discussion about the training data $O_n$ and the Monte Carlo sampling given in Remark \ref{Ondisc} applies to (\ref{ERMoptc5}) and (\ref{ERMoptc6}). 

\begin{theorem}\label{mainresultRS}

Let $\mathcal{H}$ be the set of probability kernels of the form (\ref{rsetd1}) and (\ref{rsetd0}). Assume hypotheses (E1, H0-H2-R1 and T1). Assume there exists a family of optimal controls $(\pi^{\text{opt}}_j)_{\ell=n}^{m-1}$ in $\mathcal{H}$ for the control problem with value functionals $\mathbb{V}^{\mathcal{H}}_n$ (see (\ref{randomDP})), for $n=0,\ldots, m-1$. Then  

\begin{eqnarray}
\nonumber\mathbb{E}_M|\widehat{\mathbb{V}}^{\mathcal{H}, M}_n(O_n) - \mathbb{V}^{\mathcal{H}}_n(O_n)| &=& \mathcal{O}_{\mathbb{P}}\Bigg( \Big(\delta^4_M N_M \frac{\log(M)}{M}\Big)^{2^{-(m-n)}}+ \Big( \rho_M^2 \delta_M^8 \eta^6_M \| \psi\|^2 M^{-1} \Big)^{2^{-(m-n+1)}}\\
\label{errrs}&+& \Big(\max_{n\le \ell \le m} \inf_{\Phi \in \mathcal{V}_M(\ell)}\| \Phi(O_\ell) - \mathbb{V}^{\mathcal{H}}_\ell(O_\ell)\|_{M,2}\Big)^{2^{-(m-n+1)}}\Bigg), 
\end{eqnarray}
for $n=m-1, \ldots, 0$. 
\end{theorem}
In Theorem \ref{mainresultRS}, it is important to stress, for simplicity, we deliberately assume the optimization problem is written over $\mathcal{H}$ rather than the set of all probability kernels $\varphi_j: \mathbb{H}^j\rightarrow \mathcal{P}(\mathbb{A})$. Hence, the result depends on the choice of $\mathcal{H}$. This is the reason that there is no component in the right hand side of (\ref{errrs}) measuring the error that one can possibly make by approximating optimal randomized controls via a dense architecture of functions. We postpone the proof of Theorem \ref{mainresultRS} to Section \ref{appendix3}.

\subsection{Adaptive model learning integrated into Dynamic Programming.}\label{adaptivesec}
This section explains how dominating measures $\xi_K \in \mathcal{M}_K$ can be effectively used in the context of adaptive stochastic control. We assume there is a true parameter $\theta^\star \in \Theta$ which describes the environment of the examples of Section \ref{EXsection}. The goal of this section is to present a feasible and scalable numerical scheme based on the dominating class of training measures that is robust w.r.t. estimates $\hat{\theta}$ of the true parameter model $\theta^\star$. Then $\theta^\star$ is a fixed but unknown parameter (in the sense of classical statistical framework) and $\hat{\theta}_n\rightarrow \theta$ in probability as $n\rightarrow +\infty$ for $\hat{\theta}_n \in \Theta$. 


As a warming up, we start with the simplest case of path-dependent SDEs. 

\subsubsection{Path-dependent case}\label{adaptiveSDE}
Suppose we have a family of controlled path-dependent SDEs of the form 
\begin{equation}\label{robustpdsde}
dX^\theta_t = \alpha_\theta(t, X^{\theta,u}|_t, u_t)dt + \sigma_{\theta}(t, X^{\theta,u}|_t,u_t)dB_t
\end{equation}
where $\theta$ lies in a compact set $\Theta \subset \mathbb{R}^q$. For simplicity, we treat the one-dimensional case. Let $\mathbb{V}^\theta_j$ be the family of value functionals parameterized by $\theta$: $\mathbb{V}^\theta_m(\mathbf{o}_m) = \mathbb{V}_m(\mathbf{o}_m) = \varphi(x_0 + \sum_{i=1}^m y_i)$, 

$$\mathbb{V}^\theta_j (\mathbf{o}_j):= \min_{a \in \mathbb{A}} \mathbf{U}^\theta_j(\mathbf{o}_j,a),\quad \mathbf{U}^\theta_j(\mathbf{o}_j,a):= \int_{\mathbb{H}} \mathbb{V}^\theta_{j+1}(\pi_2(\mathbf{o}_j),x,y)r^\theta_j(a,y,\mathbf{o}_j)q(y)dy \nu(dx)$$
for $j=m-1, \ldots, 0$. Here, $\xi_K(dy) = g(y)dy \in \mathcal{M}_K$ for a fixed compact set $K$ of the form (\ref{compactKPDSDE}) and the parameterized importance sampling density

$$r^\theta_j(a,y,\mathbf{o}_j):= \frac{\mathcal{R}(c^{\theta}(\mathbf{o}_j,a),v^\theta(\mathbf{o}_j,a);y)}{q(y)}; y \in K,$$
where $c^\theta(\mathbf{o}_j,a):= \alpha_\theta(t_j,\{y_i\}_{i=0}^{j-1},a),v^\theta(\mathbf{o}_j,a):= \sigma_\theta(t_j,\{y_i\}_{i=0}^{j-1},a)$ and the parameterized functionals $(\alpha_\theta,\sigma_\theta)$ of (\ref{robustpdsde}) satisfy the following standing assumptions:

\

\begin{itemize}
\item[(TB1)]The coefficients $(\alpha,\sigma)$ are globally Lipschitz. That is, there exists a constant $K_{Lip}$ such that  

$$|\alpha_{\theta}(t,w,a) - \alpha_{\theta'}(t',w',b)| + |\sigma_{\theta}(t,w,a) - \sigma_{\theta'}(t',w',b)|\le K_{\Theta, Lip} \Big\{ d\big((t,w); (t',w') \big)+ |a-b| + |\theta -\theta'|\Big\},$$
for every $(t,w), (t',w') \in \Lambda$, $a,a' \in \mathbb{A}$ and $\theta,\theta' \in \Theta$.

\item[(TB2)] There exist $0<c_{\min} <  c_{\max}<\infty$ and $v_{\max} $such that
$$
c_{\min}\ \le\ |\,\alpha_{\theta}(t,w,a)\,|\ \le\ c_{\max}\quad \text{and}\quad |\sigma_{\theta}(t,w,a)|\le v_{\max},
$$
for every $(t,w,a,\theta)\in\Lambda\times\mathbb A\times \Theta$. 
\end{itemize} 
We also assume that assumption (B3) holds. 
\begin{remark}
Observe that all training samples do not depend on $\theta$ because they are generated by $\xi_K$.   
\end{remark}

Let us denote $a^\theta_j(\mathbf{o}_j) \in \argmin_a \mathbf{U}^\theta_j(\mathbf{o}_j,a)$ for $j=m-1, \ldots, 0$ and the optimal control $\text{opt}_{\theta}(\mathbf{o}_{m-1}):= \big( a^\theta_0(\mathbf{o}_0), a^\theta_1(\mathbf{o}_1), \ldots, a^\theta_{m-1}(\mathbf{o}_{m-1}) \big)$ for a given state $\mathbf{o}_{m-1} \in \mathbb{H}^{m}$. Dynamic programming equation applied for each $\theta \in \Theta$ yields 

$$\mathbb{E}\big[ g(X^{\theta, \text{opt}_{\theta}}_m) \big] = \min_{u \in U^m_0} \mathbb{E}[g(X^{\theta,u}_m)] = \mathbb{V}^{\theta}_0.$$
Before we describe the algorithm, we need to show a stability result.



\begin{proposition}\label{LipvaluePAR}
Under Assumptions (TB1-TB2, B3), there exists a constant $C$ depending on $\text{Leb}(K)$, $c_{\min}$, $\overline{M}$, $\epsilon$, $\|f_J\|_\infty$, $K_{\Theta,Lip}$ and $\|f_J'\|_\infty$, such that 

$$\max_{0\le j \le m-1}\sup_{\mathbf{o}_j \in \mathbb{H}^j}|\mathbb{V}^{\theta}_j(\mathbf{o}_j) -\mathbb{V}^{\theta'}_j(\mathbf{o}_j)|\le C\|\varphi\| |\theta - \theta'|,$$
for every $\theta,\theta' \in \Theta$.  
\end{proposition}
\begin{proof}
Throughout the proof, we denote $C_1 = C_1(\bar{M},\epsilon,K_{\Theta,Lip}, Q_K, c_{\min},\|f_J\|_\infty, \|f'_J\|_\infty)$ as a constant which may differ from line to line and which depends on the parameters $\bar{M},\epsilon,K_{\Theta,Lip}, c_{\min},\|f_J\|_\infty, \|f'_J\|_\infty$. Let us denote 
\[
\Delta(c,v; c',v') := \int_K |\mathcal{R}(c,v;y) - \mathcal{R}(c',v'; y)| \, dy
\]
for $0 < c_{\min} \le |c|,|c'| \le c_{\max}$ and $0 < |v|,|v'| \le v_{\max}$. For $j=m-1,\ldots, 0$, we fix $\mathbf{o}_j \in \mathbb{H}^{j}, a, a' \in \mathbb{A}$ and we write 
\begin{small}
\begin{eqnarray}\label{iter1}
\nonumber\mathbf{U}^{\theta}_j (\mathbf{o}_j,a) - \mathbf{U}^{\theta'}_j(\mathbf{o}_j,a') &=& \int_{[\underline{M},\overline{M}]\times K} \mathbb{V}^\theta_{j+1}(\pi_{2}(\mathbf{o}_j),x,y)\big[ r^\theta_j(a,y,\mathbf{o}_j) -  r^{\theta'}_j(a',y,\mathbf{o}_j)   \big]q(y)dy f_J(x)dx\\
& + & \int_{[\underline{M},\overline{M}]\times K} \Big[\mathbb{V}^\theta_{j+1}(\pi_{2}(\mathbf{o}_j),x,y) - \mathbb{V}^{\theta'}_{j+1}(\pi_{2}(\mathbf{o}_j),x,y) \Big] r^{\theta'}_j(a',y,\mathbf{o}_j) q(y)dy f_J(x)dx.  
\end{eqnarray} 
\end{small}
Our argument is fully based on (\ref{iter1}) and (\ref{LipRden}). For $j=m-1$, by using the fact $\mathbb{V}^\theta_m = \mathbb{V}_m$, (\ref{iter1}) is reduced to 

$$
\mathbf{U}^{\theta}_j (\mathbf{o}_j,a) - \mathbf{U}^{\theta'}_j(\mathbf{o}_j,a') = \int_{[\underline{M},\overline{M}]\times K} \mathbb{V}^\theta_{j+1}(\pi_{2}(\mathbf{o}_j),x,y)\big[ r^\theta_j(a,y,\mathbf{o}_j) -  r^{\theta'}_j(a',y,\mathbf{o}_j)   \big]q(y)dy f_J(x)dx.
$$
By using (\ref{LipRden}) and Assumption TB1, we get 

\begin{eqnarray}\label{iter2}
|\mathbf{U}^{\theta}_j (\mathbf{o}_j,a) - \mathbf{U}^{\theta'}_j(\mathbf{o}_j,a')|&\le& \|\varphi\|_\infty \Delta (c^{\theta}(\mathbf{o}_j,a), c^{\theta'}(\mathbf{o}_j,a'); v^{\theta}(\mathbf{o}_j,a), v^{\theta'}(\mathbf{o}_j,a') )\\
\nonumber&\le& \| \varphi\|_\infty C_1 \Big\{ |\alpha_{\theta}(t_j,\{y_i\}_{i=0}^j,a) - \alpha_{\theta'}(t_j,\{y_i\}_{i=0}^j,a') |\\
\nonumber&+& |\sigma_{\theta}(t_j,\{y_i\}_{i=0}^j,a) - \sigma_{\theta'}(t_j,\{y_i\}_{i=0}^j,a') |      \Big\}\\
\nonumber&\le& \| \varphi\|_\infty C_1 \Big\{ |a-a'| + |\theta - \theta'|\Big\}. 
 \end{eqnarray}
From (\ref{iter2}), we infer 

\begin{eqnarray}\label{iter3}
|\mathbb{V}^{\theta}_{m-1}(\mathbf{o}_{m-1}) - \mathbb{V}^{\theta'}_{m-1}(\mathbf{o}_{m-1}) |&=& |\min_{a \in \mathbb{A}}\mathbf{U}^{\theta}_j (\mathbf{o}_j,a) - \min_{a \in \mathbb{A}}\mathbf{U}^{\theta'}_j(\mathbf{o}_j,a)|\\
\nonumber&\le& \sup_{a \in \mathbb{A}} |\mathbf{U}^{\theta}_j (\mathbf{o}_j,a) - \mathbf{U}^{\theta'}_j(\mathbf{o}_j,a)|\\
\nonumber&\le&\| \varphi\|_\infty C_1 |\theta - \theta'|. 
\end{eqnarray}
For $j=m-2$, by using (\ref{iter1}), (\ref{iter3}) and again (\ref{LipRden}) and Assumption TB1, we get  

\begin{eqnarray}\label{iter4}
|\mathbf{U}^{\theta}_j (\mathbf{o}_j,a) - \mathbf{U}^{\theta'}_j(\mathbf{o}_j,a')|&\le& \|\varphi\|_\infty \Delta (c^{\theta}(\mathbf{o}_j,a), c^{\theta'}(\mathbf{o}_j,a'); v^{\theta}(\mathbf{o}_j,a), v^{\theta'}(\mathbf{o}_j,a') )\\
\nonumber&+& 2 C_1\|f_J\|_\infty |\theta - \theta'|\\
\nonumber&\le& \| \varphi\|_\infty C_1 \Big\{ |a-a'| + |\theta - \theta'|\Big\}\\
\nonumber&+&  2 C_1\|\varphi\|_\infty \|f_J\|_\infty |\theta - \theta'|.
\end{eqnarray}
From (\ref{iter4}), we infer

\begin{equation}\label{iter5}
|\mathbb{V}^{\theta}_{m-1}(\mathbf{o}_{m-2}) - \mathbb{V}^{\theta'}_{m-2}(\mathbf{o}_{m-2}) |\le  \| \varphi\|_\infty C_1 |\theta - \theta'|. 
\end{equation}
By iterating this argument down to $j=0$, we conclude that $|\mathbb{V}^\theta_0 - \mathbb{V}^{\theta'}_0|\le \| \varphi\|_\infty C_1 |\theta - \theta'|$ for every $\theta,\theta'\in \Theta$. 
\end{proof}

\noindent \textbf{Adaptive learning scheme into dynamic programming}. Let us now introduce the following algorithm:

\

\noindent \textbf{Step 0}. Let $\theta_{n-1} \in \Theta$ be a consistent estimator for $\theta^\star \in \Theta$ at level $n\ge 2$. For a given iid sample $\{O^{(p)}_j;1\le j\le m, 1\le p\le M_{n-1} \} \stackrel{(d)}{=} \nu \otimes \xi_K$, follow (\ref{ERMoptc3}) and (\ref{ERMoptc4}) and train the Neural Networks 

$$\big\{\hat{a}^{\theta_{n-1}}_{j,M_{n-1}}, \widehat{\mathbb{V}}^{\theta_{n-1}}_{j,M_{n-1}}\big\}_{j=0}^{m-1}$$ 
associated with the model (\ref{robustpdsde}). Here, $M_{\ell}$ is the number of Monte Carlo samples associated with the model parameter $\theta_\ell$. 

\

\noindent \textbf{Step 1}. Parameter update of the importance sampling density from $\theta_{n-1}$ to $\theta_n$:

$$r^{\theta_{n-1}}_j(a,y,\mathbf{o}_j) \quad \text{to} \quad r^{\theta_n}_j(a,y,\mathbf{o}_j)$$
for $1\le j\le m-1 $

\

\noindent \textbf{Step 2}. Warm start (Initialization). The initialization rule for the Neural Networks are set as

\begin{equation}\label{rule1}
\text{NN-params}~(\hat{a}^{\theta_n}_{j,M_n})|_{\text{init}}:= \text{NN-params}~(\hat{a}^{\theta_{n-1}}_{j,M_{n-1}}),
\end{equation}

\begin{equation}\label{rule2}
\text{NN-params}~(\widehat{\mathbb{V}}^{\theta_n}_{j,M_n})|_{\text{init}}:= \text{NN-params}~(\widehat{\mathbb{V}}^{\theta_{n-1}}_{j,M_{n-1}}),
\end{equation}
for $1\le j\le m-1 $.  Here, the right-hand side of (\ref{rule1}) and (\ref{rule2}) are the Neural Network parameters associated with $\hat{a}^{\theta_{n-1}}_{j,M_{n-1}}$ and $\widehat{\mathbb{V}}^{\theta_{n-1}}_{j,M_{n-1}}$, respectively.

\

\noindent \textbf{Step 3}. Run the Dynamic Programming equation: With the objects of Steps 0, 1 and 2 at hand,  we set $\widehat{\mathbb{V}}^{\theta_n}_{m,M_n}:= \mathbb{V}_m$ and  

\

\begin{enumerate}
  \item Control update at time $j$

$$
\hat{a}^{\theta_n}_{j,M_{n}}\in \argmin_{C \in \mathcal{C}_{M_n}(j)}\widehat{\mathcal{J}}^{\theta_n}_{j,M_n}(C)
$$
where

$$\widehat{\mathcal{J}}^{\theta_n}_{j,M_n}(C):= \frac{1}{M}\sum_{p=1}^{M_n}\widehat{\mathbb{V}}^{\theta_n}_{j+1,M_n}\big( \pi_2(O^{(p)}_j), \mathcal{W}^{(p)}_{j+1},(Y_{j+1})^{(p)} \big) r^{\theta_n}_j(C(O^{(p)}_j), (Y_{j+1})^{(p)},O^{(p)}_j ).$$

  \item Value function update at time $j$

  
$$
\widetilde{\mathbb{V}}^{\theta_n}_{j,M_n} \in \argmin_{\Phi \in \mathcal{V}_{M_n}(j)}\widehat{\mathcal{L}}^{\theta_n}_{j,M_n}(\Phi)
$$
where 

$$\widehat{\mathcal{L}}^{\theta_n}_{j,M_n}(\Phi):= \frac{1}{M}\sum_{p=1}^{M_n}r^{\theta_n}_j( \hat{a}^{\theta_n}_{j,M_{n}}, (Y_{j+1})^{(p)},O^{(p)}_j)\Big| \widehat{\mathbb{V}}^{\theta_n}_{j+1,M_n}(O^{(p)}_{j+1}) - \Phi(O^{(p)}_j)\Big|^2$$
and we set 
$$\widehat{\mathbb{V}}^{\theta_n}_{j,M_n}:= \max \Big\{ \min (\widetilde{\mathbb{V}}^{\theta_n}_{j,M_n}; \| \mathbb{V}_n\|_\infty); -\| \mathbb{V}_n\|_\infty\Big\},$$
for $j=m-1, \ldots, 1, 0$. 
\end{enumerate} 
The above scheme admits a two-scale consistency. Theorem \ref{mainresult} and Proposition \ref{LipvaluePAR} yield the following result that separates Monte Carlo approximation error from model-risk error.

\begin{proposition}\label{adapresSDE}
Assume hypotheses TB1-TB2-B3 and T1. If $\theta_n, \theta^\star \in \Theta$, then 
\begin{equation}\label{twosc}
\mathbb{E}_{M_n}\big|\widehat{\mathbb{V}}^{\theta_n}_{j,M_n}(O_j) - \mathbb{V}^{\theta^\star}_j(O_j)\big|
\lesssim_{\mathbb{P}} 
\underbrace{\overbrace{\mathbb{E}_{M_n}|\widehat{\mathbb{V}}^{\theta_n}_{j,M_n}(O_j) - \mathbb{V}^{\theta_n}_j(O_j)|}^{\mathclap{\text{sampling error controlled by }M_n}}}_{\mathclap{\text{Th~\ref{mainresult} at }\theta_n}}
 \quad + \quad  
\mathbb{E}_{M_n}\underbrace{\overbrace{|\mathbb{V}^{\theta_n}_{j}(O_j) - \mathbb{V}^{\theta^\star}_j(O_j)|}^{\mathclap{\text{model mismatch controlled by }|\theta_n-\theta^\star|}}}_{\mathclap{\text{Prop~\ref{LipvaluePAR}}}},
\end{equation}
for $j=m-1, \ldots, 0$. 
\end{proposition}

\begin{remark} 
Since $\{\mathbb{V}_j\}_{j=0}^{m-1}$ is not known in practice, we do not expect to find $M=M_n$ independently from $n$ such that $\mathbb{E}_{M}|\widehat{\mathbb{V}}^{\theta_n}_{j,M}(O_j) - \mathbb{V}^{\theta^\star}_j(O_j)|\rightarrow 0$ in probability as $M\rightarrow +\infty$ uniformly in $n$. Nevertheless, observe that Proposition \ref{adapresSDE} yields local uniform stability of (1) and (2) in Step 3 as long as $\theta_n$ lies in a neighborhood of $\theta^\star$. The algorithm is feasible because: expectations under all models are expressed through a single reference measure $\nu\otimes \xi_K$, value functions depend smoothly on parameters, Bellman recursion is finite and locally stable. The algorithm is scalable because: Monte Carlo samples are reused, only importance sampling weights are updated, Neural Network retraining is incremental via (\ref{rule1}) and (\ref{rule2}), sample size can be grown adaptively.
\end{remark}
\subsubsection{Partial hedging with rough volatility under uncertainty}
Fix $0 < H < \frac{1}{2}$. Let $\Theta \subset \mathbb{R}^q$ be a compact subset of parameters. The rough stochastic volatility is subject to uncertainty on the parameters,  
namely: 

$$\zeta_\theta,\beta_\theta>0, \mu_{\theta,\text{drift}} \neq 0, \varkappa_\theta\in \mathbb{R}.$$. 

\begin{remark}
Since the kernels $H \mapsto K_{H,1}, K_{H,2}$ are not differentiable along $(0, 0.5)$, we cannot allow parameter dependence of $H$ w.r.t. $\theta$. 
\end{remark}
We fix the structural parameters $ 0 < s_{\min} < s_{\max}$ and $0 < v_{\min} < v_{\max}$ as in Section \ref{trrv}. We assume that $\theta\mapsto \Gamma(\theta):=(\zeta_\theta, \beta_\theta, \mu_{\theta,\text{drift}}, \varkappa_{\theta})$ has bounded derivatives. Moreover,

\begin{equation}
0 < c_{\min}=s_{\min}|\mu_{\theta,\text{drift}}| < s_{\max}|\mu_{\theta,\text{drift}}|=c_{\max},\quad 0 < \theta_{\min}\le \vartheta(\tau^\theta(\cdot))\le \theta_{\max}
\end{equation}
for every $\theta \in \Theta$. For a control value $a\in\mathbb A$ and a history $O_{j} =\mathbf{o}_{j}=(w_0, y_0, \ldots, w_{j},y_{j})$ (with $O_0 = (0,0,x_0)$), we define
\[
c^{\theta}(\mathbf{o}_{j},a) = \left(
            \begin{array}{c}
              c^{\theta,(1)}(\mathbf{o}_{j}) \\
              c^{\theta, (2)}(\mathbf{o}_{j},a) \\
            \end{array}
          \right)
:=\begin{pmatrix}
\mu_{\theta,\mathrm{drift}} \mathfrak{T}_j(y_0, \ldots, y_{j})\\
a\,\mu_{\theta, \mathrm{drift}}\mathfrak{T}_j(y_0, \ldots, y_{j})
\end{pmatrix},
\]
\[v^{\theta}(\mathbf{o}_{j},a) = \left(
            \begin{array}{c}
              v^{\theta, (1)}(\mathbf{o}_{j}) \\
              v^{\theta, (2)}(\mathbf{o}_{j},a) \\
            \end{array}
          \right)
:=\begin{pmatrix}
\mathfrak{T}_j(y_0, \ldots, y_{j})\,\vartheta(\tau^\theta(\bar w_{j}))\\
a\,\mathfrak{T}_j(y_0, \ldots, y_{j})\,\vartheta(\tau^\theta(\bar w_{j}))
\end{pmatrix},
\]
for $0\le j\le m-1$. Here, for $\mathcal{A}_n = \bar{w}_n$ and $\theta \in \Theta$, we define 

$$\ln \tau^\theta (\mathcal{A}_n) = \varkappa_\theta + e^{-\beta_\theta T_n}(z_0-\varkappa_\theta) + \zeta_\theta W^H_n -\beta_\theta \zeta_\theta e^{-\beta_\theta T^k_n}\sum_{j=1}^n W^H_{j-1} e^{\beta_\theta T_{j-1} }\Delta T_{j},$$  
for $1\le n\le m-1$. 
\begin{lemma}\label{bderVOL}
Assume that $\theta\mapsto (\zeta_\theta, \beta_\theta, \mu_{\theta,\text{drift}}, \varkappa_{\theta})$ has bounded derivatives. Then, the mappings $\theta \mapsto c^{\theta,(1)}(\mathbf{o}_j)$ and $\theta\mapsto v^{\theta,(1)}(\mathbf{o}_j)$ are globally Lispchitz uniformly w.r.t. $\mathbf{
o}_j$, for $0\le j\le m-1$. 
\end{lemma}
\begin{proof} 
Fix $0\le j\le m-1$ and $\mathbf{o}_j = (w_0, y_0, \ldots, w_j, y_j)$ and $\bar{w}_j = (w_0, \ldots, w_j)$. Of course, 

$$\|\nabla_\theta c^{\theta,(1)}(\mathbf{o}_j)\|\le \|\nabla_\theta \mu_{\theta,\text{drift}}\||\mathfrak{T}_j(y_0, \ldots, y_{j})|\le s_{\max} \|\nabla_\theta \mu_{\theta,\text{drift}}\|,$$ 
for every $\mathbf{o}_j \in \mathbb{H}^{j}$ and $\theta \in \Theta$. The subtle point is related to the second component. Observe that

$$\nabla_\theta (\tau^\theta(\bar{w}_j)) = \exp (\ln \tau^\theta(\bar{w}_j))\nabla_\theta (\ln \tau^\theta(\bar{w}_j)) $$ 
By assumption C2, we sample $\mathcal{W} \stackrel{(d)}{=} (T_1, \Delta A_1)$  where $T_1\ge \underline{M} >0$ a.s. We also recall, by construction of the imbedding scheme, that $\max_{0\le \ell\le m}\|A_\ell\|\le m\epsilon$ a.s. From (\ref{dkh1}), (\ref{dkh2}) and Lemma \ref{repWH}, we can safely state that $\max_{0\le n\le m-1}|W^H_n| \in L^{\infty}(\mathbb{P})$. Therefore, $\exp (\ln \tau^\theta(\bar{w}_j))$ is bounded for every realization $\bar{w}_j$. Therefore, there exists a constant $C$ such that

$$\max_{0\le j\le m-1}\sup_{\theta \in \Theta, \bar{w}_j \in \mathbb{W}^j}\|\nabla_\theta (\tau^\theta(\bar{w}_j))\|\le C.$$
This shows that 
$$\max_{0\le j\le m-1}\sup_{\mathbf{o}_j \in \mathbb{H}^j}|v^{\theta,(1)}(\mathbf{o}_j) - v^{\theta',(1)}(\mathbf{o}_j) |\le s_{\max}\|\vartheta\|_{\text{Lip}}C\|\theta-\theta'\|,$$
for every $\theta,\theta' \in \Theta$. This concludes the proof.      
\end{proof}

We fix $\lambda \in \mathcal{G}(\mathbb{A})$ and $\mu_K = q(r)dr \in \mathcal{M}_K$. Let $\xi_K$ be the probability measure on the cone $\mathcal{C}$ defined by the pushforward operation 
\[
\xi_K = \mathcal{Z}_\# \beta_K,
\]
where $\beta_K(dadr)= G(a,r)\lambda(da) \mu_K(dr)$ for a disintegration kernel $G:\mathbb{A}\times K\rightarrow \mathbb{R}_+$ satisfying 
$$
\inf_{a\in \mathbb{A},r \in K}G(a,r)>0.
$$
Let $\mathcal{H}$ be a family of controls of the form (\ref{rsetd3}) and (\ref{rsetd4}). For a given parameter $\theta$, we denote   
$$\overline{r}^\theta_j(z, \mathbf{o}_{j}):= \frac{\mathcal{R}(c^{\theta,(1)}(\mathbf{o}_j),v^{\theta,(1)}(\mathbf{o}_j);z)}{q(z)},\quad \overline{m}^\theta_j(z, \mathbf{o}_{j}):=\frac{\mathcal{R}^L(c^{\theta,(1)}(\mathbf{o}_j),v^{\theta,(1)}(\mathbf{o}_j);z)}{q(z)};~z\in K,$$

$$\mu^{\theta, \pi_j}_{j}(dx| \mathbf{o}_j) := \rho^{\theta, \pi_j}_j(\mathbf{o}_j,x)\xi_K(dx),\quad \rho^{\theta, \pi_j}_{j}(\mathbf{o}_j,x) := \frac{1}{2} \frac{h_{j}\big(\mathbf{o}_j, \frac{x_2}{x_1}\big)\{\overline{r}^\theta_j(x_1, \mathbf{o}_{j}) + \overline{m}^{\theta}_j(x_1)\big\}}{G\big(\frac{x_2}{x_1},x_1\big)}, 
$$
for $x=(x_1,x_2) \in \mathcal{C}$, $\mathbf{o}_j \in \mathbb{H}^j$ for $0\le j \le m-1$. In the sequel, we denote $\{\mathbb{V}^{\theta,\mathcal{H}}_j; 0\le j\le m\}$ as the sequence of value functionals defined by: $\mathbb{V}^{\theta, \mathcal H}_m(\mathbf{o}_m)
:= \varphi\big(x_0+\sum_{i=1}^m \Delta x_i\big)$, we set

\begin{align}\label{randomDPadap}
\nonumber \mathbf{U}^{\theta, \mathcal{H}}_j(\mathbf{o}_j,\pi_j)
&:= \int_{\mathbb{W}}\int_{\mathbb{H}} \mathbb{V}^{\theta, \mathcal{H}}_{j+1}(\mathbf{o}_j,x,x')\mu_j^{\theta, \pi_j}(dx'|\mathbf{o}_j)\nu(dx)\\
\mathbb{V}^{\theta, \mathcal H}_j(\mathbf{o}_j)
&:= \inf_{\pi_j\in \mathcal{H}_j} \mathbf{U}^{\theta, \mathcal{H}}_j(\mathbf{o}_j,\pi_j)
\end{align}
for $j=m-1, \ldots, 0$.   
\begin{proposition}\label{LipvaluePARVOL}
There exists a constant $C$ depending on $K$, $H$, $c_{\min}, c_{\max},v_{\min}, v_{\max}$, $\underline{M},\overline{M}$, $\epsilon$, $\sup_{\theta \in \Theta}\nabla_\theta\Gamma(\theta)$, $\|f_J\|_\infty$, and $\|f_J'\|_\infty$, such that 

$$\max_{0\le j \le m-1}\sup_{\mathbf{o}_j \in \mathbb{H}^j}|\mathbb{V}^{\theta, \mathcal{H}}_j(\mathbf{o}_j) -\mathbb{V}^{\theta',\mathcal
{H}}_j(\mathbf{o}_j)|\le C\|\varphi\| |\theta - \theta'|,$$
for every $\theta,\theta' \in \Theta$.  
\end{proposition}
\begin{proof}
The proof is entirely similar to the one of Proposition \ref{LipvaluePAR}. We just need to apply Lemmas \ref{glemma} an \ref{bderVOL}. We omit the details.  
\end{proof}
Based on Proposition \ref{LipvaluePARVOL}, we can construct an adaptive learning scheme into dynamic programming associated with partial hedging under stochastic volatility. The description is entirely similar to the one described in Section \ref{adaptiveSDE} with the obvious modification, then we omit details.   

\begin{remark}\label{LipvaluePARfbm}
The case of controlled SDEs driven by fractional Brownian motion with $\frac{1}{2} < H < 1$ can be analyzed similar to the previous examples. Under boundedness assumption on the gradient of the coefficient $\theta \mapsto \varrho_\theta$, we can evaluate the Radon derivative 

$$r^\theta_j(a,y,\mathbf{o}_j) = \frac{f^\theta(\mathbf{o}_j,a,y)}{q_\beta(y)}$$
where $q_\beta$ is the two-sided Laplace density with $0<\beta\le \frac{\gamma_\Delta}{c_{\text{max}}}$, $f^\theta(\mathbf{o}_j,a,y)$ is the density of $\phi_{b_{j-1},c^\theta_{j-1}}(J)$, $c^\theta_{j-1} = \varrho_\theta(\mathbf{y}_{j-1},a); \theta \in \Theta$ and $b_{j-1} = O_{j-1} = (w_0,y_0, \ldots, w_{j-1},y_{j-1})$. See (\ref{FBMpathwise}). We then can prove that the associated value function $\theta\mapsto \mathbb{V}^{\theta}_j$ satisfies the Lipschitz property as described in Propositions \ref{LipvaluePAR} and \ref{LipvaluePARVOL}. We omit the details. By combining Theorem \ref{mainresultRS} and Proposition \ref{LipvaluePARVOL}, a similar estimate (\ref{twosc}) also holds for the SDE controlled state driven by fractional Brownian motion.  
\end{remark}


\section{Numerical Experiments}\label{sec:numerical}

We illustrate the embedded neural dynamic programming algorithm of Section~4 by means of two experiments: (a) hedging in a rough stochastic volatility model and (b) mitigating model risk via importance sampling in a simple Markovian toy example. The purpose of Section \ref{offroughsec} is not to investigate the global exploration problem for rough-volatility models, but to provide the simplest possible numerical illustration of the skeleton-based deep neural dynamic programming method on a realistic fully non-Markovian problem. Thus, the experiment should be viewed as a controlled empirical illustration of the algorithm on the finite subset of state--action transitions effectively visited by the Monte Carlo sample, where we implement the algorithm of deterministic feedback strategies as described in Section \ref{detTH}. In this limited sense, it is consistent with the local finite-sample logic underlying Theorem \ref{mainresult}, even though Proposition \ref{nonePCONE} rules out a global dominating measure over the full control class. We therefore do not claim that Section \ref{offroughsec} provides a numerically complete treatment of exploration under rough volatility. The empirical investigation of randomized strategies, which constitute the mathematically appropriate framework for global exploration in view of Proposition \ref{nonePCONE} and Theorem \ref{mainresultRS}, is left for future work. The theoretically consistent numerical validation of the adaptive importance‑sampling scheme is presented in Section \ref{subsec:structured_is_experiment} in an illustrative Markovian example.

The discrete skeleton $\mathscr{D}$ is generated as in Section \ref{discretesec} with accuracy parameter $2^{-k}$ and number of steps for the dynamic programming equations is $m$. In order to keep simplicity for the numerical experiment, we choose quadratic loss function. Hence, we will illustrate the method in the classical linear-quadratic control problem.


\subsection{Mean--Variance Hedging}\label{mvhsec}
For simplicity, the underlying market interest rate is zero. For a terminal payoff $\Phi(\cdot)$ and admissible strategy $u\in U^{m}_0$, the embedded wealth process satisfies
\begin{equation}\label{yn}
Y^{u}_{n+1}=Y^{u}_n+u_n \Delta S_{n+1},\quad 0\le n\le m-1
\end{equation}
and the MVH objective reads
\begin{equation}
\min_{u\in U_0^{m}} 
\mathbb{E}\big[(Y^{u}_{m}-\Phi(S_{m}))^2\big],
\end{equation}
where $U^{m}_0$ is the set of controls described in (\ref{controlform0}). This is a Linear-Quadratic stochastic control problem and the control and value functions are parameterized, respectively, as



\begin{equation}
u_n=a(S_n, \theta)+b(S_n, \theta)Y_n,
\end{equation}
and 
\begin{equation}
V(S_n,Y_n;\eta)=c(S_n;\eta)+d(S_n; \eta)Y_n+e(S_n;\eta )Y^2_n,\quad e\ge 0,
\end{equation}
for parameters $(\theta,\eta)$. The scalar functions $a,b,c,d,e$ are implemented by standard Feed Forward Neural Networks and the risky asset price process $S$ will follow a fully non-Markovian rough volatility model $H \approx 0$ with a standard put option with short maturity $T=\frac{1}{12}$ (approximately one month).  




\subsubsection{MVH with rough volatility} The 2-dimensional embedded skeleton along $\mathscr{D}$ reads
\begin{eqnarray}\label{expRVm}
S_{n+1} &=& S_n + \mu_{\text{drift}} S_n\,\Delta T_{n+1} + \vartheta(V_n)\,S_n\,\Delta A^1_{n+1}\\
\nonumber V_{n} &=& \exp(Z_{n}),
\end{eqnarray}
where $Z_n$ follows a discrete fractional Ornstein-Uhlenbeck (\ref{OUdiscrete}) driven by a fractional Brownian motion $W^H_n$ with $0 < H < \frac{1}{2}$ and generated by $\rho \Delta A^1_n + \bar{\rho}\Delta A^2_n$. Here, $\Delta A^1_n$ and $\Delta A^2_n$ are the skeleton increments approximating the driving Brownian motion of the risky asset price and the rough volatility. The fractional Brownian motion is generated by the representation

\begin{equation}\label{discWH1}
W^H_ n = \rho B^{H,1}_n + \bar{\rho}B^{H,2}_n
\end{equation}
where $B^{H,i}_n=0$, for $n=0,1$ and $i=1,2$ and 
\begin{equation}\label{discWH2}
B^{H,i}_n= \sum_{j=2}^n \Delta A^{i}_j\,K_{H,1}(T_n,T_{j-1})
   + \sum_{j=1}^{n-1} \Delta A^{i}_j\,K_{H,2}(T_n,T_j),
\end{equation}
for $n\ge 2$ and $i=1,2$. Recall that $(K_{H,1}, K_{H,2})$ are the original Volterra kernels which describe fractional Brownian motion. See (\ref{RKHker}). Here, $\vartheta:\mathbb{R}_+\to\mathbb{R}_+$ is the identity. There are $4$ parameters $(z_0, \varkappa, \beta,\zeta)$ which fully describe 

\begin{equation}\label{simfou}
Z_n=\varkappa + e^{-\beta T_n}(z_0-\varkappa) + \zeta W^H_n -\beta \zeta e^{-\beta T_n}\sum_{\ell=1}^n W^H_\ell\exp(\beta T_\ell)\Delta T_\ell.
\end{equation}
The log-volatility process is simulated through
\[
Z_{i+1}=Z_i+\beta(\varkappa-Z_i)\,\Delta T_i+\zeta\,\Delta W^H_i,
\qquad
V_i=e^{Z_i},
\]
with the additional clipping
\[
Z_i \leftarrow \min\{\max\{Z_i,Z_{\min}\},Z_{\max}\},
\]
where $Z_{\min}=-3\times 10^3$ and $Z_{\max}=3\times 10^3$. Empirical fits often find $H \approx [0.05,0.2]$. A common benchmark is $H=0.1$. $\zeta>0$ is the vol-of-vol scale for the fractional driver. $\beta$ is mean-reversion speed of log-vol. $z_0$ is the initial log-vol. The parameter $\varkappa \in \mathbb{R}$ is the ong-run mean of $Z$ (hence of $V$) and it is very important to choose properly.

\begin{remark}
We must $\varkappa$ to target a desired long-run average level of instantaneous vol $\mathbb{E}[V_T]$. In the experiment, we choose $\varkappa$ such that
\begin{equation}\label{varkappap}
\mathbb{E}[V_T] = 0.2
\qquad \Longleftrightarrow \qquad
\varkappa = \frac{\ln(0.2) - \tfrac12\,\sigma_T^2 - e^{-\beta T} z_0}
{1 - e^{-\beta T}},
\end{equation}
where $V_T = \exp(Z_T)$. We can compute $\sigma^2_T$ via 

\[
\sigma_T^2
=
\zeta^2 \,
\mathrm{Var}\!\left(
X(T)
-
\beta e^{-\beta T}
\int_0^T e^{\beta u} X(u)\,du
\right),
\]
for a fractional Brownian motion $X$.

\end{remark}
Figure \ref{sampleFOU} illustrates a path of the simulated fractional Ornstein-Uhlenbeck described in (\ref{simfou}) with parameters described in Table \ref{tab:rough OU}. By Monte Carlo, we can compute $Y_0 = \mathbb{E}_Q[(K-S_{m})^+]$ in (\ref{yn}), where $Q$ is a martingale measure one has to choose. In order to check the accuracy of Neural Networks in estimating value functions and optimal controls in a complete market, we deliberately set the correlation $\rho=1$. Then, in our experiment, we are in a complete market setting where $Y_0$ is the unique risk neutral price of the put option.

We consider a short - maturity $T=\frac{1}{12}$ and at-the-money $S_0 = K=100$ case. The short-maturity regime $T=\frac{1}{12}$ and $\rho=1$ impose

\begin{equation}\label{mnumberRV}
m = \Big \lceil \frac{\epsilon^{-2}}{\chi_d}\frac{1}{12}\Big\rceil
\end{equation}
number of periods in the dynamic programming equation, where $\chi_d=1$. The parameters are summarized below.

\begin{table}[h]
\centering
\caption{Rough volatility model}
\begin{tabular}{ccccccccccc}
\hline
\text{Maturity} & $K$ & $\mu_{\text{drift}}$ & $\beta$ & $H$ & $S_0$ &  $\varkappa$ & $\zeta$ & $z_0$ & \text{long-run level of vol} & $\rho$\\
\hline
$\frac{1}{12}$ & 100 & 0.08 & 3 & 0.1& 100& (\ref{varkappap}) & 1.5 & $\log$ 0.2 & 0.2 & 1\\
\hline
\end{tabular}
\label{tab:rough OU}
\end{table} 



\begin{figure}[!t] 
\centering
\vspace*{-7cm}
\includegraphics[width=\linewidth, keepaspectratio, trim=0 11cm 0 0,clip]{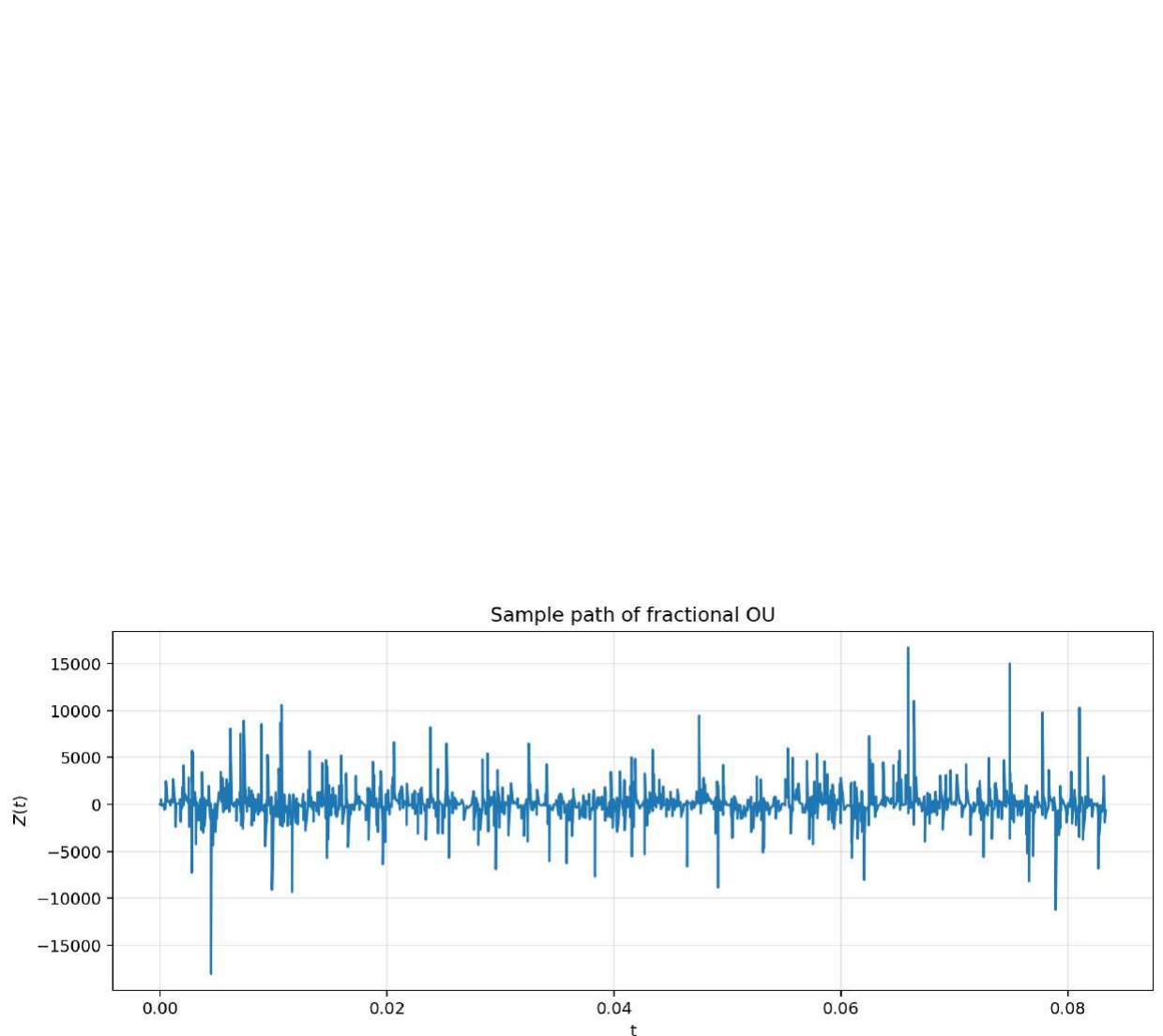}
\caption{Sample path of fractional Ornstein-Uhlenbeck on the interval $[0,\frac{1}{12}]$ with parameters of Table 2, kernels (\ref{RKHker}) and discretization level $\epsilon=2^{-7}$}
\label{sampleFOU}
\end{figure}

%

\subsubsection{Off-model exploration for rough volatility}

The training measure $\mu$ used in the experiments are given by 

\begin{equation}\label{realexpst}
\mu=\mathcal{Z}_{\#}\beta, \quad \beta= U[-r_{\text{train}},r_{\text{train}}]\times U[-dS_{\max}, +dS_{\max}],
\end{equation}
where $U[-r_{\text{train}},r_{\text{train}}]$ denotes the uniform distribution over a compact set $[-r_{\text{train}},+r_{\text{train}}]$ for $r_{\text{train}} \in (0,1]$. Then, we make the exploration training scheme outside the model. If you choose $dS_{\max}$ arbitrarily, then it might be too large implying unrealistic transitions. If it is too small, then we would face poor exploration. The code first generates a small number of rough volatility paths
$$\Delta S_{i,n}= \mu_{\text{drift}} S_{i,n-1}\Delta T_1 + S_{i,n-1} \vartheta(V_{i,n-1})\Delta A_1$$
$$V_{i,n-1} = \exp(Z_{i,n-1})$$
using the same discretization as the real experiment and compute $\max_{i,n}|\Delta S_{i,n}|$. To avoid truncating the support too aggressively, the code multiplies by a safety factor $dS_{\max}=1.10\times \max_{i,n}|\Delta S_{i,n}|$. This gives the largest increment observed in the pre-simulation.

\subsubsection{The training method for Rough volatility}

Skeleton increments are generated by the Burq--Jones hitting-time scheme \cite{Burq_Jones2008}. Asset paths follow the embedded Euler dynamics as described in Section \ref{EXsection}. Fractional Brownian motion is generated via $\mathscr{D}$ and the mappings $(K_{H,1}, H_{H,2})$ as described in (\ref{discWH1}) and (\ref{discWH2}). See also Lemma \ref{repWH}. 

Optimization is performed by stochastic gradient descent of AdamW-type with a time-adaptive learning-rate schedule designed to mitigate Monte Carlo noise in the backward recursion:
\[
\mathrm{lr}_n = 10^{-4}\Big(1-\frac{n}{m}\Big)+10^{-6};~n=m-1, \ldots,0,
\]
where the number of steps $m$ follows (\ref{mnumberRV}). Gradient clipping is applied for stability. Table \ref{NNp} presents the configuration of the Neural Networks used in solving the dynamic programming equations for the rough volatility model (Rough SV model). 

\begin{small}
\begin{table}[h]\label{NNp}
\centering
\caption{Neural network and training parameters for the rough SV model}
\begin{tabular}{cccccccc}
\hline
Experiment & Control iters & Value iters & Learning rate & Width & RB & Dropout \\
\hline
Rough & 200 & 150 & $[5\cdot10^{-5},10^{-4}]$ adaptive & 256 & 6 & 0.2 \\
\hline
\end{tabular}
\label{NNp}
\end{table}

\end{small}
The quantities reported in Table~\ref{NNp} have the following interpretations.

\begin{itemize}

\item \emph{Control iters} is the number of stochastic gradient steps used to optimize the control network at each stage $n$, corresponding to the numerical minimization of the conditional variance functional defining the optimal strategy.

\item \emph{Value iters} is the number of stochastic gradient steps used to fit the value network at each stage $n$, approximating the continuation value appearing in the dynamic programming recursion.

\item \emph{Learning rate} refers to the step size of the stochastic gradient descent (AdamW variant). A stage-dependent adaptive schedule is employed, decreasing along the backward recursion to mitigate the accumulation of Monte Carlo noise.

\item \emph{Width} denotes the number of neurons in each hidden layer of the fully connected feed-forward networks and controls the expressive capacity of the approximation space.

\item \emph{RB (Residual blocks)} refers to the number of skip-connected blocks used to construct the deep feed-forward architectures, each block consisting of two affine layers with nonlinear activation and normalization.

\item \emph{Dropout} is the probability of random neuron deactivation during training and acts as a regularization mechanism preventing overfitting to Monte Carlo sampling noise.
\end{itemize}

\subsection{Off-policy training in the rough-volatility model}\label{offroughsec}

We now discuss the numerical behavior of the off-policy training procedure in the rough-volatility setting. The purpose of the following experiments is twofold. First, we investigate how the quality of the learned mean-variance hedging strategy evolves as the embedded discretization is refined. Second, we analyze the sensitivity of the method w.r.t. the exploratory training radius \(r_{\mathrm{train}}\) that defines the off-policy control distribution.

In all the experiments reported below, the final P\&L corresponds to the hedging error of an at-the-money European put evaluated under the rough-volatility model, while the training data are generated off-policy through the exploratory law described previously. More precisely, the off-model training wealth states are built from exploratory controls sampled from a uniform law on $[-r_{\mathrm{train}},r_{\mathrm{train}}]$, together with exploratory increments \(\Delta S^{\mathrm{train}}\) sampled from a bounded interval calibrated from the underlying numerical scheme. The resulting tables therefore provide direct information on how the choice of discretization level and the width of the exploratory distribution affect the quality of the learned hedge.

We first fix the exploratory parameter at $r_{\mathrm{train}} = 0.5$ and examine the behavior of the terminal P\&L as the discretization level \(k\) increases. The corresponding statistics are reported in Table \ref{tab:offpolicy_pnl_quantiles_k}.

\begin{small}
\begin{table}[htbp]
\centering
\caption{Diagnostic statistics for the P\&L of an ATM European put. Off-policy rough-volatility experiment with \(r_{\mathrm{train}}=0.5\).}
\begin{tabular}{c c c c c c}
\hline
$k$ & $m$ & Mean P\&L & $\mathrm{Var}(\text{P\&L})$ & $q_{5\%}$ & $q_{1\%}$ \\
3 & 6 & 0.036333 & 12.305737 & -4.622572 & -5.780918 \\
4 & 22 & 0.071766 & 6.220595 & -4.087614 & -5.981536 \\
5 & 86 & -0.010717 & 2.463002 & -2.620666 & -3.595436 \\
6 & 342 & 0.019073 & 0.537016 & -1.107734 & -1.628304 \\
\hline
\end{tabular}
\label{tab:offpolicy_pnl_quantiles_k}
\end{table}
\end{small}

The first striking feature of Table \ref{tab:offpolicy_pnl_quantiles_k} is the strong and systematic monotone reduction in the variance of the terminal P\&L as \(k\) increases. The monotone decrease is the most important quantitative signal in the table. It shows that the off-policy training procedure becomes substantially more stable as the random skeleton is refined. In particular, the monotone drop from \(6.220595\) at \(k=4\) to \(0.537016\) at \(k=6\) is especially significant and indicates that the method enters a much more accurate numerical regime once the discretization is sufficiently fine.

The lower-tail diagnostics exhibit the same behavior. The empirical \(5\%\) and \(1\%\) quantiles improve sharply with \(k\). Refining the discretization not only reduces the global dispersion of the P\&L, but also makes the extreme downside scenarios considerably less severe. This is financially important, since the quantiles capture the large-loss events of the hedging strategy. In particular, the improvement from \(k=3\) to \(k=6\) is substantial, with the \(1\%\) quantile moving from approximately \(-5.78\) to approximately \(-1.62\). The mean P\&L remains very close to zero throughout the table confirming that we are in a complete market setting with $\rho=1$. In the present rough-volatility setting, variance reduction and left-tail improvement are the dominant consequences of refinement.



\begin{figure}[htbp]
\vspace{-4cm}
\centering
\captionsetup{skip=0pt}
\includegraphics[width=0.85\textwidth,trim={0 8cm 0 0},clip]{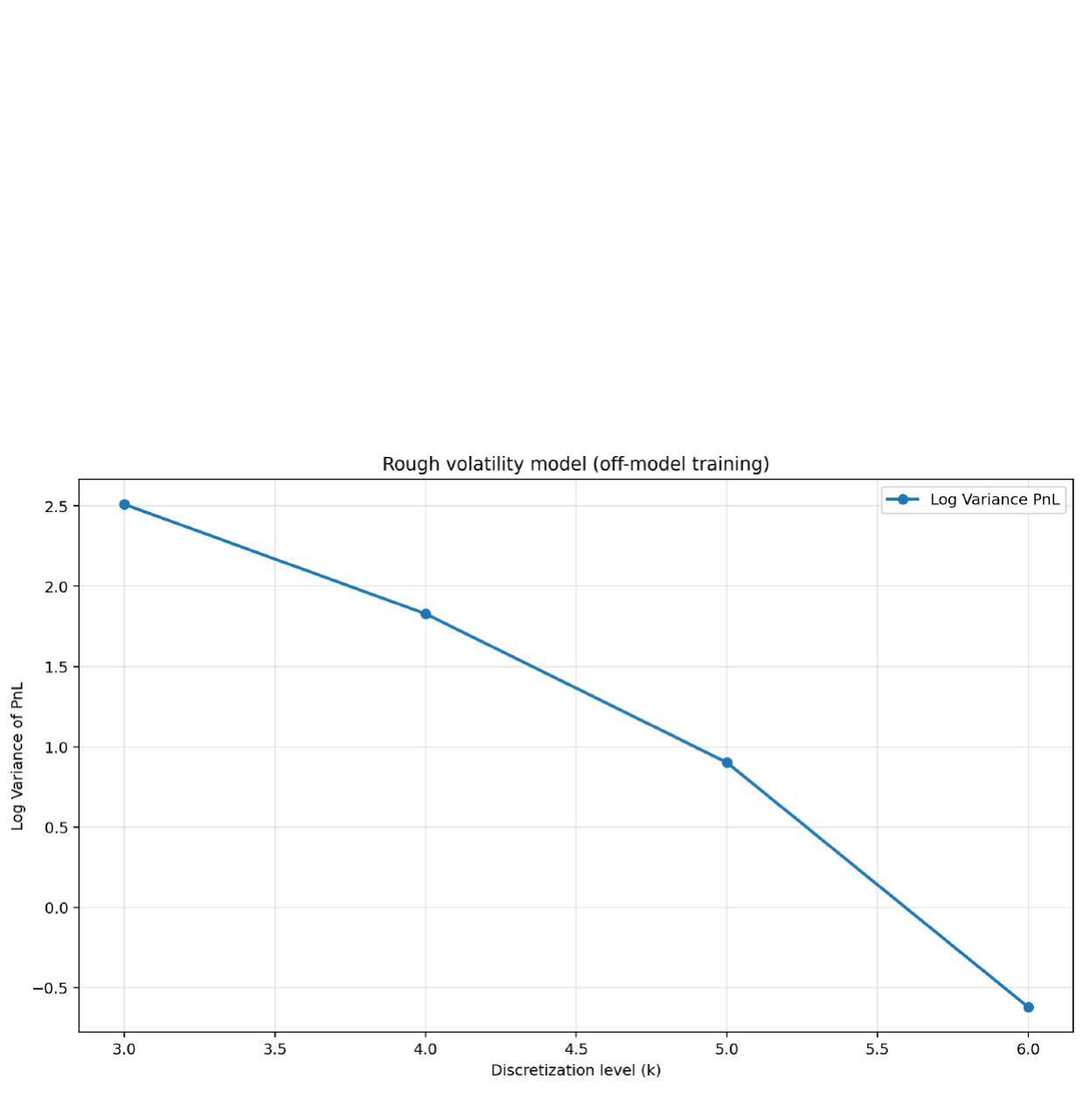}
\caption{Empirical $\mathrm{Var}(\mathrm{P\&L})$ as a function of the discretization level (Rough SV model with $\mathbb{E}[V_T]=0.2$, $r_{\mathrm{train}}=0.5$).}
\label{sweeping}
\end{figure}

Figure \ref{sweeping} investigates the empirical variance of the $\text{P\&L}_k$ as the discretization level $\epsilon_k\downarrow 0$. 
The approximately linear decrease of $\log(\mathrm{Var}(\mathrm{P\&L}))$ indicates an exponential decay of the hedging error,
$$
\mathrm{Var}(\mathrm{P\&L}_k)\approx C2^{-\alpha k},
$$
for $k\ge 3$ confirming convergence of the numerical scheme. Table \ref{tab:offpolicy_rtrain_rough} presents the effect of the exploration radius \(r_{\mathrm{train}}\) in the off-policy training law, for fixed $k=6$. The results show that the dependence on \(r_{\mathrm{train}}\) is clearly non-monotone. Among the four tested values, the choice $r_{\mathrm{train}} = 1.0$ produces the best overall performance. The case \(r_{\mathrm{train}}=1.0\) also gives the most favorable empirical quantiles: $q_{5\%}=-0.540655$ and $q_{1\%}=-0.611763$, whereas all the other tested values lead to more negative tail outcomes. In particular, the deterioration is substantial when \(r_{\mathrm{train}}\) is 0.5 and 0.75. For example, when \(r_{\mathrm{train}}=0.5\), the \(1\%\) quantile is \(-1.628304\), and when \(r_{\mathrm{train}}=0.75\), it remains at \(-1.655443\). Hence, the value \(r_{\mathrm{train}}=1.0\) is again the most favorable choice from the viewpoint of protection against large losses. 

From a practical viewpoint, Table \ref{tab:offpolicy_rtrain_rough} shows that the exploration radius is not a secondary numerical detail, but a genuine tuning parameter of the methodology. A similar pattern is observed for several other seeds, where \(r_{\mathrm{train}}=1.00\) often yields either the smallest variance or one of the best overall trade-offs between variance and lower-tail risk. This suggests that, in the present rough-volatility setting, a wider exploration range may help the backward dynamic programming algorithm learn a more robust control across the relevant region of the state space. Intuitively, when \(r_{\mathrm{train}}\) is too small, the off-policy wealth states used during training may fail to sufficiently explore the part of the state space visited by the optimal policy, which can lead to a poorer approximation of the continuation values. On the other hand, the results also show that this effect is not completely uniform across seeds. We therefore interpret \(r_{\mathrm{train}}=1.00\) not as a universally optimal choice, but as the most consistently favorable exploration level among those tested.

\begin{small}
\begin{table}[htbp]
\centering
\caption{Diagnostic statistics for the P\&L of an ATM European put. Off-policy rough-volatility experiment with varying \(r_{\mathrm{train}}\).}
\begin{tabular}{c c c c c}
\hline
$r_{\mathrm{train}}$ & Mean P\&L & $\mathrm{Var}(\text{P\&L})$ & $q_{5\%}$ & $q_{1\%}$ \\
\hline
0.25 & 0.036721 & 0.296224 & -0.681063 & -0.897237 \\
0.50 & 0.019073 & 0.537016 & -1.107734 & -1.628304 \\
0.75 & 0.020382 & 0.524207 & -1.116670 & -1.655443 \\
1.00 & 0.023102 & 0.255494 & -0.540655 & -0.611763 \\
\hline
\end{tabular}
\label{tab:offpolicy_rtrain_rough}
\end{table}
\end{small}

Figure~\ref{histogram} indicates that the terminal P\&L is tightly centered around zero and exhibits an approximately bell-shaped distribution. The empirical mean is small compared with the standard deviation, so the distribution does not reveal a substantial systematic bias. At the same time, the right tail is slightly more pronounced than the left one, suggesting a mild positive skewness.



\begin{figure}[htbp] 
\vspace*{-4cm}
\centering
    \includegraphics[width=\linewidth, height=11cm, trim=0 9.1cm 0 0,clip]{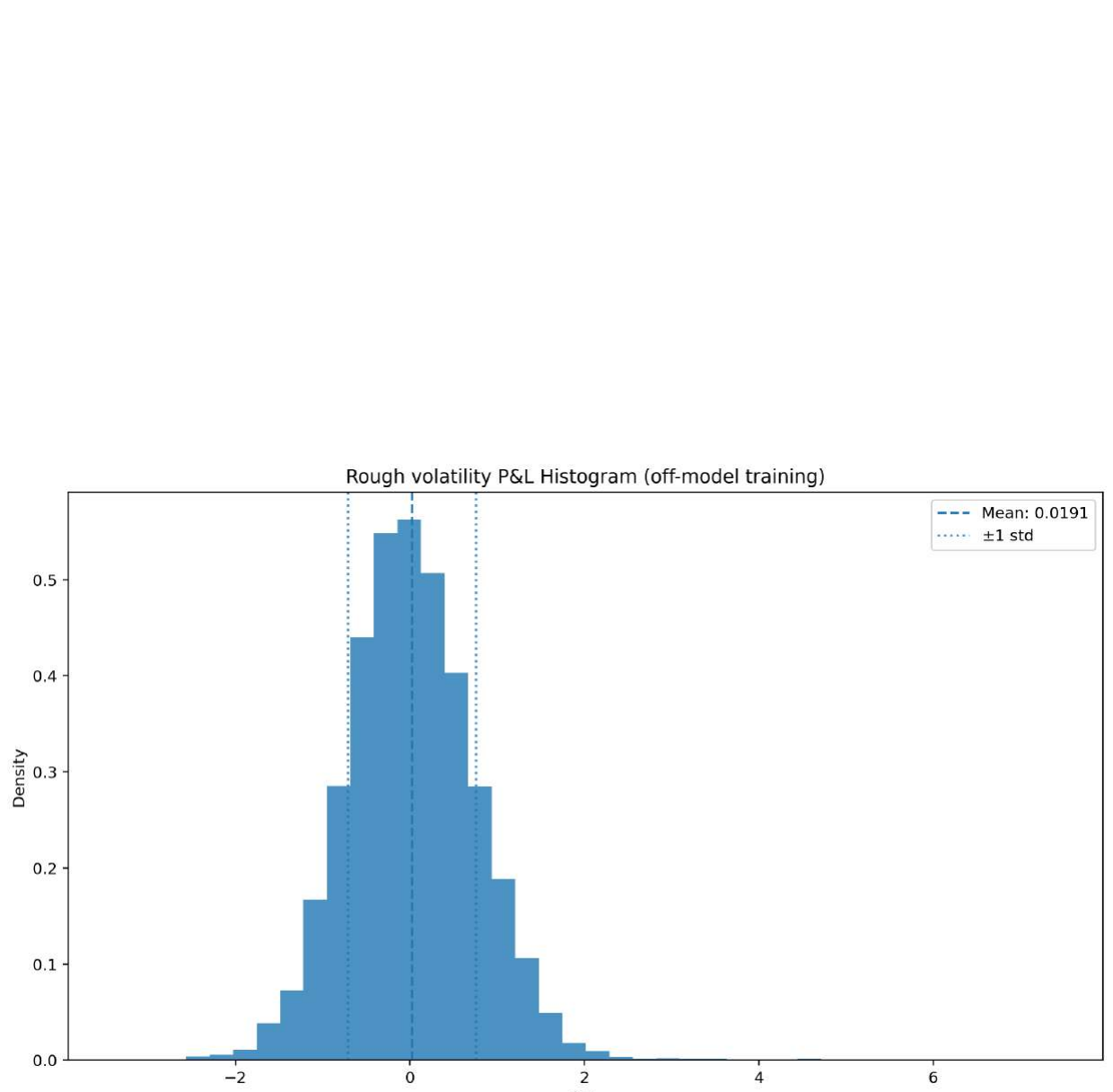}
    \caption{Histogram of P\&L. Rough volatility. ATM Put with $S_0=100=K$. Number of Monte Carlo samples = 7000. $r_{\text{train}} = 0.5$}
    \label{histogram}
\end{figure}

\subsection{A structured random-skeleton importance-sampling experiment under model risk}
\label{subsec:structured_is_experiment}

The goal of this experiment is to illustrate, in a simple Markovian setting, the model-risk adaptation mechanism developed in Section \ref{adaptivesec}. More precisely, we aim to test the adaptive importance-sampling implementation in a setting where the Bellman structure is transparent and the effect of parameter misspecification can be studied explicitly. The experiment is designed to remain faithful to the philosophy of the paper: one fixes a single proposal law, one generates a single Monte Carlo training dataset once and for all, and a change in the model parameter is handled by reweighting the Bellman regressions rather than regenerating the training sample.

At the continuous-time level, this toy problem corresponds to the linear controlled diffusion
\begin{equation}\label{toysde}
dX^u(t)=(u(t)-\theta)\,dt+\sigma\,dB(t),
\qquad 0\le t\le T,
\qquad T=1,
\end{equation}
where $\sigma =0.5, T=1, \theta \in \Theta \in \{1.20,1.35,1.50,1.65,1.80\}, \mathbb{A}=[-1,1]$ and we aim to solve
\[
\inf_{u\in U_0^T}\;\mathbb E\big[(X^u(T))^2\big].
\]

The discretization scheme associated with (\ref{toysde}) is described as follows:  The dynamics of $\Delta X^a_n$ conditioned on $\Xi_{n-1}$ is 

$$\Delta X^a_n|\Xi_{n-1}\stackrel{(d)}{=}(a-\theta)J + \sigma \mathbf{B}$$
whose density is 
\[
\mathcal R^\theta(a;y)
=
\frac{1}{2|a-\theta|}
\sum_{s\in\{+1,-1\}}
f_J\!\left(\frac{y-s\sigma\varepsilon}{a-\theta}\right)
\mathbf 1_{\left\{\frac{y-s\sigma\varepsilon}{a-\theta}\in[\underline M,\overline M]\right\}},
\]
where \(f_J\) denotes the truncated density of \(J\), $\mathbf{B}\stackrel{(d)}{=}\text{Bernoulli}$ taking values  $\pm 2^{-4}$, $\underline{M}=10^{-4}$ and $\overline{M}=5$. For a fixed proposal density $q$ used to generate the training sample for the increments, the corresponding importance-sampling weight is
\[
r^\theta(a,y)=\frac{\mathcal R^\theta(a;y)}{q(y)},
\]
for $\theta \in \Theta, a \in [-1,1]$ and $y \in K$ for a compact set $K$.  

A central numerical issue is the choice of the support $K$ of the proposal law. A broad worst-case support leads to an excessively diffuse proposal and poor effective sample size, while a local support better matches the region actually explored by the training procedure. In the numerical experiment, we therefore proceed as follows. First, before any Bellman training is performed, we generate an exploratory cloud of one-step increments under the reference model \(\theta_{\mathrm{ref}}\). More precisely, we sample
\[
A^{\mathrm{explore}}\stackrel{(d)}{=} \mathrm{Unif}[-1,1],
\qquad
J\stackrel{(d)}{=} f_J,
\qquad
\mathbf B\in\{\pm 2^{-4}\},
\]
independently, and define
\[
Y^{\mathrm{explore}}
=
\bigl(A^{\mathrm{explore}}-\theta_{\mathrm{ref}}\bigr)J+\sigma\mathbf B.
\]
From this exploratory sample we compute empirical lower and upper quantiles, and we then enlarge them by a small safety margin. This produces a compact interval
\[
K=[K_{\min},K_{\max}]
\]
which is local to the region effectively visited by the exploratory dynamics under the reference model. In the implementation used for the final experiment, the lower and upper quantiles were chosen as \(0.5\%\) and \(99.5\%\), respectively, and a relative safety margin equal to \(5\%\) of the empirical quantile width was added on both sides. Thus \(K\) is not a worst-case support over the full action-parameter space, but rather a fixed local support determined by pre-sampling.

Once \(K\) is fixed, the proposal density is taken to be the uniform law on \(K\):
\[
q(y)=\frac{1}{|K|}\mathbf 1_K(y).
\]
This proposal is generated only once and is kept unchanged throughout the whole experiment. All adaptive information is then carried by the Radon--Nikodym factors \(r^\theta(a,y)\).

After the proposal support \(K\) has been fixed, the Monte Carlo training dataset is generated once and for all. First, a cloud of training states \((X_j^{(p)})\) is produced recursively under exploratory actions sampled uniformly from \([-1,1]\) and under a reference parameter $\theta_{\mathrm{ref}}=1.50 \in \Theta$. More precisely,
\[
X_{j+1}^{(p)}
=
X_j^{(p)}
+
\bigl(A_j^{(p)}-\theta_{\mathrm{ref}}\bigr)J_{j+1}^{(p)}
+
\sigma \mathbf B_{j+1}^{(p)},
\qquad
A_j^{(p)}\stackrel{(d)}{=} \mathrm{Unif}[-1,1].
\]
Second, independently of these states, a fixed collection of proposal increments
\[
Y_j^{(p)}\stackrel{(d)}{=} q=\mathrm{Unif}(K)
\]
is sampled for every time step \(j\) and every Monte Carlo index \(p\). The pair
\[
\bigl(X_j^{(p)},Y_j^{(p)}\bigr),
\qquad
1\le p\le M,\quad 0\le j\le m-1,
\]
is then reused for all parameter values \(\theta_{\mathrm{true}} \in \Theta_{\mathrm{true}}\) considered in the experiment, where 

$$\Theta_{\mathrm{true}}:= \{1.20,1.35, 1.65,1.80\}.$$

At the reference parameter \(\theta_{\mathrm{ref}}=1.50\), we solve the backward dynamic programming recursion once on this fixed training dataset. This yields a family of approximate reference controls and value functions
\[
\bigl(\widehat a_j^{\,\mathrm{ref}},\widehat{\mathbb V}_j^{\,\mathrm{ref}}\bigr)_{0\le j\le m-1}.
\]
Next, we define what we call Frozen and adaptive IS modes. The Frozen mode is the non-adaptive benchmark. When the true parameter is changed from \(\theta_{\mathrm{ref}}\) to some \(\theta_{\mathrm{true}}\in\Theta_{\mathrm{true}}\), no Bellman update is performed. One simply keeps using the already-trained reference policy. Thus
\[
X_n^{\mathrm{Frozen}}
:=
X_{n-1}^{\mathrm{Frozen}}
+
\Bigl(
\widehat a_{n-1}^{\,\mathrm{ref}}(X_{n-1}^{\mathrm{Frozen}})
-\theta_{\mathrm{true}}
\Bigr)J
+
\sigma \mathbf B.
\]
Frozen therefore measures the cost of ignoring the model update and serves as the baseline against which the adaptive IS policy is compared.

The adaptive IS mode implements the importance-sampling update on the \emph{same fixed training dataset}. No new training cloud is generated when \(\theta\) is changed. Instead, only the Bellman regressions are reweighted by
\[
r^{\theta_{\mathrm{true}}}(a,y)
=
\frac{\mathcal R^{\theta_{\mathrm{true}}}(a;y)}{q(y)}, \theta_{\mathrm{true}} \in \Theta_{\mathrm{true}}. 
\]
Thus, at the control step, the empirical Bellman criterion uses the weighted objective
\begin{equation}\label{isq1}
\frac1M\sum_{p=1}^M
r^{\theta_{\mathrm{true}}}\!\big(a_j(X_j^{(p)}),Y_j^{(p)}\big)\,
\widehat{\mathbb V}^{\,\theta_{\mathrm{true}}}_{j+1}\bigl(X_j^{(p)}+Y_j^{(p)}\bigr),
\end{equation}
while the value step uses the weighted regression
\begin{equation}\label{isq2}
\frac1M\sum_{p=1}^M
r^{\theta_{\mathrm{true}}}\!\big(a_j(X_j^{(p)}),Y_j^{(p)}\big)\,
\Bigl(
\Phi_j(X_j^{(p)})
-
\widehat{\mathbb V}^{\,\theta_{\mathrm{true}}}_{j+1}\bigl(X_j^{(p)}+Y_j^{(p)}\bigr)
\Bigr)^2,
\end{equation}
for $\theta_{\mathrm{true}} \in \Theta_{\mathrm{true}}$. Hence the adaptive content of the experiment lies entirely in the change of the Radon--Nikodym factors. The training sample itself is left untouched.


After the backward training step has produced the Frozen and adaptive IS policies, the two policies are compared out of sample on a common bank of forward random-skeleton trajectories. For each \(\theta_{\mathrm{true}}\in\Theta_{\mathrm{true}}\), we generate a collection of independent forward noises
\[
\bigl(J_1^{(p)},\dots,J_m^{(p)},\mathbf B_1^{(p)},\dots,\mathbf B_m^{(p)}\bigr)_{1\le p\le N_{\mathrm{eval}}},
\]
with \(N_{\mathrm{eval}}=8000\), and we use \emph{the same} bank of trajectories for Frozen and IS. Thus, for each mode \(\pi\in\{\mathrm{Frozen},\mathrm{IS}\}\),
\[
X_n^{\pi,(p)}
=
X_{n-1}^{\pi,(p)}
+
\Bigl(
a_{n-1}^{\pi}(X_{n-1}^{\pi,(p)})
-\theta_{\mathrm{true}}
\Bigr)J_n^{(p)}
+
\sigma \mathbf B_n^{(p)},
\qquad
X_0^{\pi,(p)}=x_0.
\]
Here, $a_{n-1}^{\mathrm{ref}}:= \hat{a}^{\text{ref}}_{n-1}$ is computed under $\theta_{\mathrm{ref}}$ and $a_{n-1}^{\mathrm{true}}:= \hat{a}^{\text{IS}}_{n-1} $ is computed via (\ref{isq1}) and (\ref{isq2}), for $\theta_{\mathrm{true}} \in \Theta_{\mathrm{true}}$.  

The reported performance statistic is
\[
\mathrm{L}(\pi,\theta_{\mathrm{true}})
=
\frac1{N_{\mathrm{eval}}}
\sum_{p=1}^{N_{\mathrm{eval}}}
\bigl(X_m^{\pi,(p)}\bigr)^2,
\]
together with the corresponding Monte Carlo standard error.


The experiment shows how a parameter update can be incorporated through importance-sampling. Table \ref{tab:frozen_vs_is_localK} shows that $\text{IS L} = \mathrm{L}(\text{IS},\theta_{\mathrm{true}})$  is consistently smaller than $\text{Frozen L}= \mathrm{L}(\text{Frozen},\theta_{\mathrm{true}})$ across distinct $\theta_{\mathrm{true}}\neq \theta_{\mathrm{ref}}$ which confirms the efficiency of the importance sampling scheme in mitigating model risk.  

\begin{table}[ht!]
\centering
\begin{tabular}{ccccc}
\toprule
$\theta_{\mathrm{true}}$ & Frozen L & IS L & Frozen s.e. & IS s.e. \\
\midrule
1.20 & 0.347363 & 0.270243 & 0.005343 & 0.004279 \\
1.35 & 0.459942 & 0.390208 & 0.006429 & 0.005683 \\
1.65 & 0.834852 & 0.693630 & 0.009634 & 0.008512 \\
1.80 & 1.071468 & 0.958361 & 0.011176 & 0.010458 \\
\bottomrule
\end{tabular}
\caption{Out-of-sample comparison between the Frozen policy and the adaptive IS policy under local pre-sampled proposal support.}
\label{tab:frozen_vs_is_localK}
\end{table}

\section{Proof of Theorem \ref{mainresult}}\label{appendix1}
In this section we present the complete proof of Theorem \ref{mainresult}. For this purpose, we need to introduce the following objects. For a given $C \in \mathbb{A}^{\mathbb{H}^{n-1}}$, we set 

$$O^C_n:= \big(O_{n-1}, \mathcal{W}_n, \Delta X^C_n\big); 1\le n\le m.$$

For a given $C \in \mathcal{C}_M(n-1)$, the Monte Carlo samples of these objects are naturally defined by

$$O^{(p),C}_n:= \big(O^{(p)}_{n-1}, \mathcal{W}^{(p)}_n, \Delta X^{(p),C}_n\big),$$
where 

$$O^{(p)}_{n-1}:= \big(\mathcal{W}^{(p)}_1, Y^{(p)}_1, \ldots, \mathcal{W}^{(p)}_{n-1}, Y^{(p)}_{n-1}  \big),$$

$$\Delta X^{(p),C}_n := \blacktriangle x_n \big(\pi_2(O^{(p)}_{n-1}),C(O^{(p)}_{n-1}) , \Delta T^{(p)}_n,\ell_n(\pi_3(O^{(p)}_{n-1}),\mathcal{W}^{(p)}_n)\big) $$
for $n=m,\ldots, 1$ and $1\le p\le M$. For $C \in \mathcal{C}_M(n)$, we set

\begin{equation}\label{Zpc}
\hat{Z}^{(p),C}_{n+1}:= \hat{\mathbb{V}}^M_{n+1} \big( O^{(p),C}_{n+1}\big)
\end{equation}
and 
\begin{equation}\label{Zc}
\hat{Z}^{C}_{n+1}:= \hat{\mathbb{V}}^M_{n+1} \big( O^{C}_{n+1}\big),
\end{equation}
for $n=m-1, \ldots, 0$. Starting with $\bar{ \mathbb{V}}^M_m(\mathbf{o}_m):= g(x_0 + \sum_{j=1}^m y_j)$, we define recursively

\begin{equation}\label{vbar}
\bar{\mathbb{V}}^M_n(\mathbf{o}_n):= \mathbb{E}\big[ \widehat{\mathbb{V}}^M_{n+1}(\mathbf{o}_n, \mathcal{X}^{\hat{a}^M_n}_{n+1}(\mathbf{o}_n  )) \big],
\end{equation}
where, guided by (\ref{fXc}), (\ref{Xc}) and (\ref{deltacxn}), we set 

$$\mathcal{X}^{\hat{a}^M_{n}}_{n+1}(\mathbf{o}_n):= \Big(\mathcal{W}_{n+1}, \blacktriangle x_{n+1} \big(\pi_2(\mathbf{o}_{n}), \hat{a}^M_{n}(\mathbf{o}_n), \Delta T_{n+1},\ell_{n+1}(\pi_3(\mathbf{o}_{n}),\mathcal{W}_{n+1})\big) \Big),$$
for $\mathbf{o}_n= (w_1,y_1, \ldots, w_n,y_n) \in \mathbb{H}^n$, $\hat{a}^M_n \in \mathcal{C}_M(n)$, with $n=m-1, \ldots, 0$.

 Recall that $\nu$ and $\mu$ are the law of $\mathcal{W}_1$ and $\Delta X_1$, respectively. For a given training probability measure $\mu \in \mathcal{P}(\mathbb{R}^n)$, let us define 

$$\lambda_j:= \text{j-fold product measure of}~\nu \otimes \mu,$$
for $1\le j\le m-1$. 

\begin{remark}
We observe that
$$\bar{\mathbb{V}}^M_n(\cdot) = \mathbb{E}\big[ \widehat{\mathbb{V}}^M_{n+1}(\cdot, \mathcal{X}^{\hat{a}^M_n}_{n+1}(\cdot  ))| O_n=\cdot \big],$$
is the quantity estimated by $\widehat{\mathbb{V}}^M_n (\cdot)$ via the Feedforward Neural Networks $\mathcal{V}_M(n)$ for $n=m-1, \ldots, 0$ in (\ref{ERMoptc4}).  
\end{remark}

For a given sequence $\{C_{\ell}; n\le \ell\le m-1 \}\subset \mathbb{A}^{\mathbb{H}^{m-1}}\times \mathbb{A}^{\mathbb{H}^{m-2}}\times \ldots \times \mathbb{A}^{\mathbb{H}^{n}}$, we want to define $J^{\{C_\ell\}_{\ell=n}^{m-1}}_n \in \mathcal{V}_M(n)=\mathcal{V}_M(m-1)|_n$ in such way that (see page 12 in \cite{hure})

\begin{equation}\label{apcoin}
\mathbb{E}_M[J^{\{\hat{a}^M_\ell\}_{\ell=n}^{m-1}}_n(O_n)]= \mathbb{E}_M[\widehat{\mathbb{V}}^M_n(O_n)]
\end{equation}
for $n=m-1, \ldots, 0$. For this purpose, we set $J_m := \mathbb{V}_m$ and we define recursively

$$J^{(C_\ell)_{\ell=n}^{m-1}}_{n}(\cdot):= \text{empirical}~L^2(\lambda_n)~\text{estimator of}~\mathbb{E}[J^{C_{n+1}}_{n+1}(O_n,\mathcal{X}^{C_n}_{n+1})| O_n = \cdot] $$
for $n=m-1, \ldots, 0$ via $\mathcal{V}_M(n)$. By construction, (\ref{apcoin}) holds true.

Now, by the very definition,  

\begin{equation}\label{s1}
\bar{\mathbb{V}}^M_n(O_n) \ge \inf_{a \in \mathbb{A}}\mathbb{E}_{M}[\widehat{\mathbb{V}}^M_{n+1}(O^a_{n+1})]
\end{equation}
so that 

$$
\mathbb{E}_M\Bigg[\bar{\mathbb{V}}^M_n(O_n) - \inf_{a \in \mathbb{A}}\mathbb{E}_{M}[\widehat{\mathbb{V}}^M_{n+1}(O^a_{n+1})]\Bigg]\ge 0,
$$
for $n=m-1,\ldots, 0$. Let us denote the \textbf{estimation error} 

\begin{equation}\label{errorEST}
\varepsilon^{\text{esti}}_n:=\sup_{C \in \mathcal{C}_M(n)} \Bigg| \frac{1}{M} \sum_{p=1}^M \hat{Z}^{(p),C}_{n+1}  - \mathbb{E}_{M}[\widehat{\mathbb{V}}^M_{n+1}(O^C_{n+1})] \Bigg|,
\end{equation}
for $n=m-1,\ldots, 0$.
\begin{lemma}\label{p1lemma}
\begin{eqnarray}
\label{s2}\mathbb{E}_M[\bar{\mathbb{V}}^M_n(O_n)]&\le& \mathbb{E}_M[\widehat{\mathbb{V}}^M_{n+1}(O^{\hat{a}^M_n}_{n+1})]\\
\label{s3}&\le& \inf_{C \in \mathcal{C}_M (n)} \mathbb{E}_M[\widehat{\mathbb{V}}^M_{n+1}(O^{C}_{n+1})] + 2 \varepsilon^{\text{esti}}_n
\end{eqnarray}
for $n=m-1, \ldots, 0$. 
\end{lemma}
\begin{proof}
Inequality (\ref{s2}) follows from (\ref{vbar}) and the fact that $O^{\hat{a}^M_n}_{n+1} = (O_n, \mathcal{W}_{n+1}, \Delta X^{\hat{a}^M_n}_{n+1}\big)$. The analysis of (\ref{s3}) is more involved. Here, we make use of the function $J^{(\hat{a}^M_\ell)^{m-1}_{\ell=n+1}}_{n+1}$ as follows:

$$\mathbb{E}_M\big[\widehat{\mathbb{V}}^M_{n+1}(O^{\hat{a}^M_n}_{n+1})\big]=\mathbb{E}_M\big[J^{(\hat{a}^M_\ell)^{m-1}_{\ell=n+1}}_{n+1}(O^{\hat{a}^M_n}_{n+1})\big]+ \hat{J}^{(\hat{a}^M_\ell)^{m-1}_{\ell=n}}_{n,M} - \hat{J}^{(\hat{a}^M_\ell)^{m-1}_{\ell=n}}_{n,M},  $$
where

$$\hat{J}^{(\hat{a}^M_\ell)^{m-1}_{\ell=n}}_{n,M}:= \frac{1}{M} \sum_{p=1}^M \hat{Z}^{(p),\hat{a}^M_n}_{n+1}.
$$
It is also important to define 

$$\hat{J}^{C, (\hat{a}^M_\ell)^{m-1}_{\ell=n+1}}_{n,M}:= \frac{1}{M} \sum_{p=1}^M \hat{Z}^{(p),C}_{n+1},$$
for $C \in \mathcal{C}_M(n)$, where we recall $\hat{Z}^{(p),C}_{n+1}=\widehat{\mathbb{V}}^M_{n+1}\big( O^{(p),C}_{n+1}\big)$. Of course, 

$$\hat{J}^{(\hat{a}^M_\ell)^{m-1}_{\ell=n}}_{n,M} = \hat{J}^{C, (\hat{a}^M_\ell)^{m-1}_{\ell=n+1}}_{n,M}$$ 
for $C = \hat{a}^M_n$. Then, 

$$\mathbb{E}_M[\widehat{\mathbb{V}}^M_{n+1}(O^{\hat{a}^M_n}_{n+1})]\le \varepsilon^{\text{esti}}_n + \hat{J}^{(\hat{a}^M_\ell)^{m-1}_{\ell=n}}_{n,M},$$
for $n=m-1, \ldots, 0$. We again sum and subtract

\begin{equation}\label{s4}
\hat{J}^{C, (\hat{a}^M_\ell)^{m-1}_{\ell=n+1}}_{n,M} = \hat{J}^{C, (\hat{a}^M_\ell)^{m-1}_{\ell=n+1}}_{n,M} - \mathbb{E}_M[\widehat{\mathbb{V}}^M_{n+1}(O^C_{n+1})] + \mathbb{E}_M[\widehat{\mathbb{V}}^M_{n+1}(O^C_{n+1})],
\end{equation}
for $C \in \mathcal{V}_M(n)$ and $n=m-1, \ldots, 0$. For fixed $n=m-1, \ldots, 0$, we use again $\varepsilon^{\text{esti}}_n$ and (\ref{s4}) produces

\begin{equation}\label{s5}
\hat{J}^{C, (\hat{a}^M_\ell)^{m-1}_{\ell=n+1}}_{n,M} \le \varepsilon^{\text{esti}}_n + \mathbb{E}_M[\widehat{\mathbb{V}}^M_{n+1}(O^C_{n+1})],
\end{equation}
for every $C \in \mathcal{V}_M(n)$. Observe that

\begin{equation}\label{s6}
\hat{a}^M_n\in \argmin_{C \in \mathcal{C}_M(m)}\hat{J}^{C, (\hat{a}^M_\ell)^{m-1}_{\ell=n+1}}_{n,M},
\end{equation}
so that (\ref{s5}) and (\ref{s6}) yield

$$\hat{J}^{(\hat{a}^M_\ell)^{m-1}_{\ell=n}}_{n,M}\le  \varepsilon^{\text{esti}}_n + \mathbb{E}_M[\widehat{\mathbb{V}}^M_{n+1}(O^C_{n+1})],$$
for every $C \in \mathcal{C}_M(n)$. Then, for fixed $n=m-1, \ldots, 0$, we have  

$$\mathbb{E}_M[\widehat{\mathbb{V}}^M_{n+1}(O^{\hat{a}^M_n}_{n+1})]\le 2\varepsilon^{\text{esti}}_n + \mathbb{E}_M[\widehat{\mathbb{V}}^M_{n+1}(O^C_{n+1})],$$
for every $C \in \mathcal{C}_M(n)$. This concludes the proof. 
\end{proof}

Next, we introduce the \textbf{approximation error} 

\begin{equation}\label{errorAPP}
\varepsilon^{\text{approx}}_n:= \inf_{C \in \mathcal{C}_M(n)} \mathbb{E}_M[\widehat{\mathbb{V}}^M_{n+1} (O^C_{n+1})] - \inf_{L \in \mathbb{A}^{\mathbb{H}^n}}\mathbb{E}_M[\widehat{\mathbb{V}}^M_{n+1} (O^L_{n+1})],
\end{equation}
for $n=m-1, \ldots, 0$. Summing up (\ref{s1}) and Lemma \ref{p1lemma}, we get 

\begin{eqnarray} 
\nonumber 0&\le& \mathbb{E}_M\Big|\bar{\mathbb{V}}^M_n(O_n) - \inf_{a \in \mathbb{A}}\mathbb{E}_{M}[\widehat{\mathbb{V}}^M_{n+1}(O^a_{n+1})]\Big|\\
\nonumber &\le& 2 \varepsilon^{\text{esti}}_n + \inf_{C \in \mathcal{C}_M(n)} \mathbb{E}_M[\widehat{\mathbb{V}}^M_{n+1}(O^C_{n+1})] -  \inf_{a \in \mathbb{A}} \nonumber\mathbb{E}_M[\widehat{\mathbb{V}}^M_{n+1}(O^a_{n+1})]\\
\label{int1}&\le& 2 \varepsilon^{\text{esti}}_n + \varepsilon^{\text{approx}}_n,
\end{eqnarray}
for every $n=m-1, \ldots, 0$. By the dynamic programming principle, 
\begin{equation}\label{int2}
\sup_{M\ge 1,0\le n\le m } \{\| \bar{\mathbb{V}}^M_n\|_\infty + \| \widehat{\mathbb{V}}^M_{n}\|_\infty\}\le 2 \|g\|_\infty.
\end{equation}
Then, Lemma \ref{p2lemma} below is a straightforward application of Lemma \ref{p1lemma}, (\ref{int1}) and (\ref{int2}).  In the sequel, we denote $\| \cdot\|^2_{M,2}:= \mathbb{E}_M[|\cdot|^2 ]$.

\begin{lemma}\label{p2lemma}
$$\Big\| \bar{\mathbb{V}}^M_n(O_n) - \inf_{a \in \mathbb{A}} \mathbb{E}_{M}\big[\widehat{\mathbb{V}}^M_{n+1}(O^a_{n+1})\big]\Big\|^2_{M,2}\le 2 \|g\|_\infty \{ 2\varepsilon^{\text{esti}}_n + \varepsilon^{\text{approx}}_n  \},$$
for every $n=m-1, \ldots, 0$ and $M >0$. 
\end{lemma}
We are now ready to state the following important intermediate step towards the proof of Theorem \ref{mainresult}. 
\begin{lemma}\label{p3lemma}
$$\Big\|\widehat{\mathbb{V}}^M_n(O_n) - \bar{\mathbb{V}}^M_n(O_n)\Big\|^2_{M,2} = O _\mathbb{P} \Big( \delta^4_M N_M \frac{\log(M)}{M} + \inf_{\Phi \in \mathcal{V}_M(n)} \mathbb{E}\big| \Phi(O_n) - \bar{\mathbb{V}}^M_n (O_n) \big|^2 \Big),
$$
for $n=m,\ldots, 0$ and $M\ge 1$.  
\end{lemma}
\begin{proof}
Observe that 

\begin{equation}\label{s7}
\widehat{\mathbb{V}}^M_n  =\left\{
\begin{array}{rl}
\widetilde{\mathbb{V}}^M_n; & \hbox{if} \ -\| \mathbb{V}_n\|_\infty < \widetilde{\mathbb{V}}^M_n < \| \mathbb{V}_n\|_\infty   \\
\|\mathbb{V}_n\|_\infty;& \hbox{if} \ \widetilde{\mathbb{V}}^M_n  > \|\mathbb{V}_n\|_\infty\\
- \|\mathbb{V}_n\|_\infty;& \hbox{if} \ \widetilde{\mathbb{V}}^M_n  < - \|\mathbb{V}_n\|_\infty < \| \mathbb{V}_n\|_\infty,
\end{array}
\right\},
\end{equation}
for $n=m-1, \ldots, 0$. By the dynamic programming principle, we than have $\max_{0\le n\le m, M\ge 1}\| \widehat{\mathbb{V}}^M_n\|_\infty \le \| g\|_\infty$. Therefore, 

$$
\mathbb{E}\big[ \widehat{\mathbb{V}}^M_n \big(O_{n-1}, \mathcal{X}^{\hat{a}^M_{n-1}}_n(O_{n-1})\big)| O_{n-1}\big] - \widehat{\mathbb{V}}^M_n \big(O_{n-1}, \mathcal{X}^{\hat{a}^M_{n-1}}_n(O_{n-1})\big),
$$
is subgaussian for each $1\le n \le m$ and $M\ge 1$. By applying the same argument outlined in the proof of Step 2 of Theorem 4.13 in \cite{hure}, we shall apply Theorem 1 in \cite{kholer1} (see also Lemma 5.1 in \cite{kholer2} or Lemma H2 in \cite{hure}) to infer that 

\begin{equation}\label{s8}
\mathbb{E}_M \Big[ |\widetilde{\mathbb{V}}^M_n(O_n) - \bar{\mathbb{V}}^M_n(O_n)|^2 \Big] = O _\mathbb{P} \Big( \delta^4_M N_M \frac{\log(M)}{M} + \inf_{\Phi \in \mathcal{V}_M(n)} \mathbb{E}\big| \Phi(O_n) - \bar{\mathbb{V}}^M_N (O_n) \big|^2 \Big),
\end{equation}
for $n=m,\ldots, 0$ and $M\ge 1$. We omit the details of the proof of (\ref{s8}) and we refer the reader to Step 2 of Theorem 4.13 in \cite{hure} for all details. By using (\ref{s7}) and (\ref{s8}), we get 
$$\mathbb{E}_M \Big[ |\widehat{\mathbb{V}}^M_n(O_n) - \bar{\mathbb{V}}^M_n(O_n)|^2 \Big]\le \mathbb{E}_M \Big[ |\widetilde{\mathbb{V}}^M_n(O_n) - \bar{\mathbb{V}}^M_n(O_n)|^2 \Big]$$ 
for  $n=m,\ldots, 0$ and $M\ge 1$. This concludes the proof.  
\end{proof}

\begin{lemma}\label{p4lemma}
\begin{eqnarray*}
\|\mathbb{V}_n(O_n) -  \bar{\mathbb{V}}^M_n(O_n)\|_{M,2}&\lesssim_{\|r\|_\infty,\|g\|_\infty}& 
\sqrt{\|\mathbb{V}_{n+1}(O_{n+1}) -  \widehat{\mathbb{V}}^M_{n+1}(O_{n+1})\|_{M,1}}\\
&+&  \sqrt{\|g\|_\infty \{ 2\varepsilon^{\text{esti}}_n + \varepsilon^{\text{approx}}_n  \}},
\end{eqnarray*}
foe $n=m-1, \ldots, 0$. 
\end{lemma}
\begin{proof}
Fix $n \in \{0, \ldots, m-1\}$. By Theorem 4.2 in \cite{leaoohashi}, we know that

$$\mathbb{V}_n(O_n)= \inf_{a \in \mathbb{A}}\mathbb{E}\big[\mathbb{V}_{n+1}(O^a_{n+1})|O_n\big] = \inf_{a \in \mathbb{A}}\mathbb{E}\big[\mathbb{V}_{n+1}(O_n, \mathcal{X}^a_{n+1})|O_n\big],$$
where $\mathbb{E}\big[\mathbb{V}_{n+1}(O^a_{n+1})|O_n\big]$ denotes the conditional expectation w.r.t. $\mathbb{P}$ knowing that $\Delta X^a_{n+1}$ is controlled by $a \in \mathbb{A}$ on a given history $\Xi_{n}=O_{n}$. Triangle inequality yields

\begin{eqnarray*}
\|\mathbb{V}_n(O_n) -  \bar{\mathbb{V}}^M_n(O_n)\|_{M,2}&\le& \Big\|\mathbb{V}_n(O_n) - \inf_{a\in \mathbb{A}} \mathbb{E}_{M} \big[\widehat{\mathbb{V}}^{M}_{n+1}(O^a_{n+1})|O_n\big] \Big\|_{M,2}\\
&+& \Big\| \inf_{a\in \mathbb{A}} \mathbb{E}_{M} \big[\widehat{\mathbb{V}}^{M}_{n+1}(O^a_{n+1})|O_n\big] - \bar{\mathbb{V}}^M_n(O_n) \Big\|_{M,2}, 
\end{eqnarray*}
for $n=m-1, \ldots, 0$. By Lemma \ref{p2lemma}, we only need to estimate the first term of the above inequality. Observe that

\begin{eqnarray*}
\Big| \mathbb{V}_n(O_n) - \inf_{a\in \mathbb{A}} \mathbb{E}_{M} \big[\widehat{\mathbb{V}}^{M}_{n+1}(O^a_{n+1})|O_n\big]   \Big| &=& \Big| \inf_{a \in \mathbb{A}}\mathbb{E}\big[\mathbb{V}_{n+1}(O^a_{n+1})|O_n\big] - \inf_{a\in \mathbb{A}} \mathbb{E}_{M} \big[\widehat{\mathbb{V}}^{M}_{n+1}(O^a_{n+1})|O_n\big]   \Big|\\
&=& \Big| \inf_{a \in \mathbb{A}}\mathbb{E}_a\big[\mathbb{V}_{n+1}(O_{n+1})|O_n\big] - \inf_{a\in \mathbb{A}} \mathbb{E} \big[\widehat{\mathbb{V}}^{M}_{n+1}(O^a_{n+1})|O_n\big]   \Big|\\
&\le& \sup_{a \in \mathbb{A}}\Big| \mathbb{E}\big[\mathbb{V}_{n+1}(O^a_{n+1})|O_n\big] - \mathbb{E} \big[\widehat{\mathbb{V}}^{M}_{n+1}(O^a_{n+1})|O_n\big]   \Big|\\
&\le& \sup_{a \in \mathbb{A}}\mathbb{E}\Big[ |\mathbb{V}_{n+1}(O^a_{n+1}) - \widehat{\mathbb{V}}^{M}_{n+1}(O^a_{n+1})|  \big| O_n\Big],
\end{eqnarray*}
for  $n=m-1, \ldots, 0$. Fix $0\le j \le m-1$ and by using Assumption H1, observe that 

$$\sup_{a \in \mathbb{A}}\mathbb{E}\Big[ |\mathbb{V}_{j+1}(O^a_{j+1}) - \widehat{\mathbb{V}}^{M}_{j+1}(O^a_{j+1})|  \big| O_j=y\Big]
$$
$$ = \sup_{a \in \mathbb{A}}\int |\mathbb{V}_{j+1}(y, x, x') - \widehat{\mathbb{V}}^{M}_{j+1}(y,x,x')|  r_{j+1}(x,a,x'; y)\mu(dx')\mathbb{P}_{\mathcal{W}_1}(dx)$$
$$\lesssim_{\|r\|_\infty} \int |\mathbb{V}_{j+1}(y, x, x') - \widehat{\mathbb{V}}^{M}_{j+1}(y,x,x')|  \mu(dx')\mathbb{P}_{\mathcal{W}_1}(dx) $$
\begin{equation}\label{arg1h1}
 = \mathbb{E}\big[ |\mathbb{V}_{j+1}(y, \mathcal{W}_{j+1},Y_{j+1}) - \widehat{\mathbb{V}}^{M}_{j+1}(y, \mathcal{W}_{j+1},Y_{j+1})| \big| O_j=y \big]
\end{equation}
for every $y \in \mathbb{H}^j$. Then, by Jensen's inequality and taking the conditional expectation on both sides w.r.t. the Monte Carlo sample data, we have 

$$\mathbb{E}_M\Big|\sup_{a \in \mathbb{A}}\mathbb{E}\Big[ |\mathbb{V}_{j+1}(O^a_{j+1}) - \widehat{\mathbb{V}}^{M}_{j+1}(O^a_{j+1})|  \big| O_j\Big]  \Big|^2$$
$$\lesssim_{\|r\|_\infty} \mathbb{E}_M  \Big|\mathbb{V}_{j+1}(O_j, \mathcal{W}_{j+1},Y_{j+1}) - \widehat{\mathbb{V}}^{M}_{j+1}(O_j, \mathcal{W}_{j+1},Y_{j+1})\Big|^2 $$
\begin{equation}\label{arg2h1} 
= \mathbb{E}_M  \Big|\mathbb{V}_{j+1}(O_{j+1}) - \widehat{\mathbb{V}}^{M}_{j+1}(O_{j+1})\Big|^2.
\end{equation}
Hence, 

\begin{eqnarray*}
\|\mathbb{V}_n(O_n) -  \bar{\mathbb{V}}^M_n(O_n)\|_{M,2}&\lesssim_{\|r\|_\infty}& 
\|\mathbb{V}_{n+1}(O_{n+1}) -  \widehat{\mathbb{V}}^M_{n+1}(O_{n+1})\|_{M,2}\\
&+&  \sqrt{\|g\|_\infty \{ 2\varepsilon^{\text{esti}}_n + \varepsilon^{\text{approx}}_n  \}},
\end{eqnarray*}
foe $n=m-1, \ldots, 0$. Recall,

$$|\mathbb{V}_{n+1}(O_{n+1}) -  \widehat{\mathbb{V}}^M_{n+1}(O_{n+1})|\le 2\|g\|_\infty,$$
for every $n,M$. Hence,

$$\|\mathbb{V}_{n+1}(O_{n+1}) -  \widehat{\mathbb{V}}^M_{n+1}(O_{n+1})\|_{M,2}\le 2\|g\|_\infty \sqrt{\|\mathbb{V}_{n+1}(O_{n+1}) -  \widehat{\mathbb{V}}^M_{n+1}(O_{n+1})\|_{M,1}}$$

\end{proof}

We now need two technical auxiliary results. 

\begin{lemma}\label{estilemma}
Assume that the output activation function $\mathbf{a}$ in (\ref{controlf}) is the identity and the training data lives in a compact set of $\mathbb{H}^{m-1}$ bounded by $R$. Then,   
$$\max_{0\le n\le m-1}\mathbb{E}\big[\varepsilon^{\text{esti}}_n\big]\lesssim_{R,T} \frac{\|g\|_\infty}{\sqrt{M}} + \frac{\rho_M\delta^4_M\eta^3_M\|\psi\|}{\sqrt{M}},$$
for every $M\ge 1$. 
\end{lemma}

\begin{lemma}\label{approxlemma}
For $n=0, \ldots, m-1$, we have 
\begin{eqnarray*}
\varepsilon^{\text{approx}}_n &\lesssim& \| \mathbb{V}_{n+1}\|_\infty \| r\| \rho_M \inf_{G \in \mathcal{C}_M(n)} \mathbb{E}_M \Big[ | G(O_n) - \phi^{\text{opt}}_n (O_n)|\Big]\\
&+& 2 \|r\|_\infty \mathbb{E}_M \Big[ \big| \mathbb{V}_{n+1} \big(O_{n+1}\big) - \widehat{\mathbb{V}}^M_{n+1} \big( O_{n+1}\big)\big| \Big].
\end{eqnarray*}
\end{lemma}

We postpone the proofs of Lemmas \ref{estilemma} and \ref{approxlemma} to section \ref{appendix2}. We are now able to present the proof of Theorem \ref{mainresult}. 

\

\noindent \textbf{Proof of Theorem \ref{mainresult}:} 
Let us denote $\Delta_{n,M}= \mathbb{E}_M|\widehat{\mathbb{V}}^M_n(O_n) - \mathbb{V}_n(O_n)|$, for $n=m, \ldots, 0$ and $M\ge 1$. Write 

$$\widehat{\mathbb{V}}^M_n(O_n) - \mathbb{V}_n(O_n) = \widehat{\mathbb{V}}^M_n(O_n) - \bar{\mathbb{V}}^M_n(O_n) + \bar{\mathbb{V}}^M_n(O_n) - \mathbb{V}_n(O_n).$$
Triangle inequality yields 
$$\Delta_{n,M}\le \|\widehat{\mathbb{V}}^M_n(O_n) - \bar{\mathbb{V}}^M_n(O_n)\|_{M,1} + \|\bar{\mathbb{V}}^M_n(O_n) - \mathbb{V}_n(O_n)\|_{M,1},$$
for $n=m,\ldots, 0$ and $M\ge 1$. By Lemma \ref{p3lemma}, we know that 

$$\Big\|\widehat{\mathbb{V}}^M_n(O_n) - \bar{\mathbb{V}}^M_n(O_n)\Big\|_{M,1} = \mathcal{O}_\mathbb{P} \Big( \delta^2_M \sqrt{N_M \frac{\log(M)}{M}} + \inf_{\Phi \in \mathcal{V}_M(n)} \big\| \Phi(O_n) - \bar{\mathbb{V}}^M_n (O_n) \big\|_{M,2} \Big),
$$
for $n=m,\ldots, 0$ and $M\ge 1$.  Triangle inequality yields  

\begin{eqnarray*}
\inf_{\Phi \in \mathcal{V}_M(n)} \big\| \Phi(O_n) - \bar{\mathbb{V}}^M_n (O_n) \big\|_{M,2} &\le& \inf_{\Phi \in \mathcal{V}_M(n)} \| \Phi(O_n)-\mathbb{V}_n(O_n)\|_{M,2}\\
&+& \| \mathbb{V}_n(O_n) - \bar{\mathbb{V}}^M_n(O_n)\|_{M,2}
\end{eqnarray*}
and Lemma \ref{p4lemma} yields  

\begin{eqnarray}\label{vvbarra}
\nonumber\| \mathbb{V}_n(O_n) - \bar{\mathbb{V}}^M_n(O_n)\|_{M,1}&\le& \| \mathbb{V}_n(O_n) - \bar{\mathbb{V}}^M_n(O_n)\|_{M,2}\lesssim \sqrt{\|\mathbb{V}_{n+1}(O_{n+1}) -  \widehat{\mathbb{V}}^M_{n+1}(O_{n+1})\|_{M,1}}\\
&+&  \sqrt{\|g\|_\infty \{ 2\varepsilon^{\text{esti}}_n + \varepsilon^{\text{approx}}_n  \}}.
\end{eqnarray}
Lemmas \ref{estilemma} and \ref{approxlemma} yield 

$$\sqrt{\epsilon^{\text{esti}}_{n,M}} = \mathcal{O}_{\mathbb{P}}\Bigg( \frac{\sqrt{\rho_M}\delta^2_M\eta^{\frac{3}{2}}_M\sqrt{\|\psi\|}}{M^{\frac{1}{4}}}\Bigg),$$

\begin{eqnarray*}
\sqrt{\varepsilon^{\text{approx}}_{n,M}} &=& \mathcal{O}_{\mathbb{P}}\Bigg\{\inf_{G \in \mathcal{C}_M(n)} \sqrt{\| G(O_n) - \phi^{\text{opt}}_n (O_n)\|_{M,1}}\\
&+& \sqrt{2 \|r\|_\infty} \sqrt{\big\| \mathbb{V}_{n+1} \big(O_{n+1}\big) - \widehat{\mathbb{V}}^M_{n+1} \big( O_{n+1}\big)\big\|_{M,1}}\Bigg\},
\end{eqnarray*}
for $n=0, \ldots, m-1, M\ge 1$. From (\ref{vvbarra}), we get

 $$\Delta_{n,m}  = \mathcal{O}_{\mathbb{P}}\Bigg(\sqrt{\Delta_{n+1,M}} + \sqrt{\epsilon^{\text{esti}}_{n,M}} + \sqrt{\epsilon^{\text{approx}}_{n,M}} +  \delta^2_M \sqrt{N_M \frac{\log(M)}{M}}  $$
$$+ \inf_{\Phi \in \mathcal{V}_M(n)} \big\| \Phi(O_n) - \bar{\mathbb{V}}^M_n (O_n) \big\|_{M,2}\Bigg) $$
$$=\mathcal{O}_\mathbb{P}\Big( \sqrt{\Delta_{n+1,M}} + A_{n,M}\Big)$$
where we set 

$$A_{n,M} = \frac{\sqrt{\rho_M}\delta^2_M\eta^{\frac{3}{2}}_M\sqrt{\|\psi\|}}{M^{\frac{1}{4}}} + \delta^2_M \sqrt{N_M \frac{\log(M)}{M}}$$

$$ \inf_{G \in \mathcal{C}_M(n)} \sqrt{\| G(O_n) - \phi^{\text{opt}}_n (O_n)\|_{M,1}} +  \inf_{\Phi \in \mathcal{V}_M(n)} \| \Phi(O_n)-\mathbb{V}_n(O_n)\|_{M,2} $$
for $n=m-1, \ldots, 0$ and $M\ge 1$. Observe that $\Delta_{m,M}=0$ and $\Delta_{m-1,M}\le A_{m-1,M}$. By using the inequality $\sqrt{x+y}\le \sqrt{x} + \sqrt
{y}$ and iterating we get 

$$\Delta_{n,M} = \mathcal{O}_\mathbb{P} \Bigg(A_{n,M} + A^{\frac{1}{2}}_{n,M} + A^{\frac{1}{4}}_{n,M} + \ldots + A^{2^{-(m-n-1)}}_{n,M} \Bigg)$$
for $n=m-1, \ldots,0$ and $M\ge 1$. By seting $A_{\star,n,M} = \max_{n\le \ell \le n}A_{\ell,M}$, we then get  

$$\Delta_{n,M} = \mathcal{O}_\mathbb{P} \Bigg(A_{\star,n,M} + A^{\frac{1}{2}}_{\star,n,M} + A^{\frac{1}{4}}_{\star,n,M} + \ldots + A^{2^{-(m-n-1)}}_{\star,n,M} \Bigg)$$
for $n=m-1, \ldots,0$ and $M\ge 1$. By taking $M$ large enough if necessary, we may suppose $A_{\star,n,M}\le 1$. On $(0,1]$, the mapping $p\mapsto A^{p}_{\star,n,M}$ is decreasing so that 

$$\Delta_{n,M} = \mathcal{O}_\mathbb{P}\Big( A^{2^{-(m-n-1)}}_{\star,n,M} \Big),$$
for $n=m-1, \ldots,0$ and $M\ge 1$. Lastly, by setting $E_{1,M}= \Big(\delta^4_M N_M \frac{\log(M)}{M}\Big)^{\frac{1}{2}}$, $E_{2,M} = \Big(\frac{\rho^2_M\delta^8_M\eta^{6}_M \|\psi\|^2}{M}\Big)^{\frac{1}{4}} $,  

$$B^{\text{val}}_{\star,n,M} = \max_{n\le \ell\le m}\inf_{\Phi \in \mathcal{V}_M(\ell)} \| \Phi(O_\ell)-\mathbb{V}_\ell(O_\ell)\|_{M,2}, $$
$$B^{\text{ctr}}_{\star,n,M} = \max_{n\le \ell\le m}\inf_{G \in \mathcal{C}_M(\ell)} \sqrt{\| G(O_\ell) - \phi^{\text{opt}}_\ell (O_\ell)\|_{M,1}}, $$
and raising to the power of $p=2^{-(m-n-1)} \in (0,1]$, we get 
$$\Delta_{n,M} = \mathcal{O}_\mathbb{P} \Big(  E^p_{1,M} + E^p_{1,M} + (B^{\text{val}}_{\star,n,M})^p + (B^{\text{ctr}}_{\star,n,M})^p \Big), $$
for $n=m-1, \ldots,0$ and $M\ge 1$. This concludes the proof. 

\section{Proofs of Lemmas \ref{estilemma} and \ref{approxlemma}}\label{appendix2}

\subsection{Proof of Lemma \ref{estilemma}} 
The proof is inspired by the proof of Lemma 4.15 in \cite{hure}. We give the details here. 

\begin{lemma}\label{LipNNV}
The elements of $\mathcal{V}_M(n)$ are globally $\delta_M\|\psi\|\eta_M$-Lipschitz functions, for every $n=m-1, \ldots, 1$. 
\end{lemma}
\begin{proof}
Let $\Phi (\cdot; \theta) \in \mathcal{V}_M(n)$ be an arbitrary function of the form 
$$\Phi(\mathbf{o}_\ell;\theta)= \sum_{j=1}^N c_{i}\psi\big(\langle a_{i}, \mathbf{o}_\ell\rangle + b_{i} \big) + c_{0},$$

where 

$$\theta = (a_i,b_i,c_i),\quad \|a_i\|\le \eta_M,\quad b_i \in \mathbb{R},\quad \sum_{i=1}^N|c_i|\le \delta_M.$$

Observe the Lipschitz property of $\psi$ and Cauchy - Schwartz inequality yield 

$$\big|\Phi(\mathbf{o}_n;\theta) - \Phi(\mathbf{o}_n';\theta)  \big|\le \sum_{i=1}^N|c_i|  \| \psi\| \langle a_i,\mathbf{o}_n-\mathbf{o}'_n \rangle|\le \delta_M\|\psi\|\eta_M |\mathbf{o}_n-\mathbf{o}'_n|,$$
for every $\mathbf{o}_n,\mathbf{o}'_n \in \mathbb{H}^n$ and a parameter $\theta$.   
\end{proof}
\begin{lemma}\label{leest1}
$$\mathbb{P}\Bigg\{ \sup_{G \in \mathcal{C}_M(n)}\Big| \frac{1}{M} \sum_{p=1}^M \widehat{Z}^{(p),G}_{n+1} - \mathbb{E}\big[\widehat{Z}^{(p),G}_{n+1}\big]\Big| > \varepsilon \Bigg\}$$
$$\le 2 \mathbb{P}\Bigg\{ \sup_{G \in \mathcal{C}_M(n)}\Big| \frac{1}{M} \sum_{p=1}^M \big[\widehat{Z}^{(p),G}_{n+1} - \widehat{Z}^{' (p),G}_{n+1}\big]\Big| > \frac{\varepsilon}{2} \Bigg\}.$$
\end{lemma}

\begin{lemma}\label{leest2}

$$\mathbb{E}\Bigg[ \sup_{G \in \mathcal{C}_M(n)}\Big| \frac{1}{M} \sum_{p=1}^M \widehat{Z}^{(p),G}_{n+1} - \mathbb{E}\Big[ \widehat{Z}^{(p),G}_{n+1}\Big]\Big| \Bigg]$$
$$\le 4 \mathbb{E}\Bigg[ \sup_{G \in \mathcal{C}_M(n)}\Bigg| \frac{1}{M} \sum_{p=1}^M \Big[\widehat{Z}^{(p),G}_{n+1} - \widehat{Z}^{' (p),G}_{n+1}\Big]\Bigg|  \Bigg]. $$
\end{lemma}

\begin{lemma}\label{leest3}
Let $(r_p)_{p=1}^\infty$ be an iid sequence such that $\mathbb{E}[r_1]=0$ and $r_1$ is a Bernoulli random variable taking values $\{\pm 1\}$ with probability $\frac{1}{2}$ (Rademacher family). Then,  
$$\mathbb{E}\Bigg[ \sup_{G \in \mathcal{C}_M(n)}\Big| \frac{1}{M} \sum_{p=1}^M \Big[\widehat{Z}^{(p),G}_{n+1} -\widehat{Z}^{'(p),G}_{n+1}\Big]\Big| \Bigg]$$
$$\le 4 \mathbb{E}\Bigg[ \sup_{G \in \mathcal{C}_M(n)}\Bigg| \frac{1}{M} \sum_{p=1}^M r_p\widehat{Z}^{(p),G}_{n+1}\Bigg|  \Bigg]. $$
\end{lemma}  
The proofs of Lemmas \ref{leest1}, \ref{leest2} and \ref{leest3} follow the same arguments given in Steps 1, 2 and 3, respectively, in the proof of Lemma 4.15 (Appendix F) of \cite{hure}. For this reason, we omit the details.   
We recall the important  result due to \cite{talagrand}. 

\begin{theorem}[Talagrand's Contraction Lemma. Th 4.12 in \cite{talagrand}]\label{talagrandlemma}
Let $(r_i)_{ 1\le i\le M}$ be an iid Rademacher sequence. Let $(\phi_i)_{ 1\le i\le M}$ be a family of Lipschitz functions with norms $L$ and assume $\phi_i(0)=0$. Let $\mathcal{G}$ be a class of functions and let $\{x_1, \ldots, x_m\}$ be an iid sequence independent of $(r_i)_{1\le i\le M}$. Then

$$\mathbb{E}_r \Bigg[\sup_{f \in \mathcal{G}}\Bigg| \sum_{p=1}^M r_p \phi_p (f(x_p))\Bigg| \Bigg]\le 2 L \mathbb{E}_r \Bigg[\sup_{f \in \mathcal{G}}\Bigg| \sum_{p=1}^M r_p f(x_p)\Bigg| \Bigg] $$

\end{theorem}

\begin{lemma}\label{leest4}
Assume that the output activation function in the Neural Network (\ref{controlNN}) is the identity. Furthermore, assume that the training data lives in a compact subset bounded by $R$. Then,  
$$\max_{0\le n\le m-1}\mathbb{E}\Bigg[ \sup_{G \in \mathcal{C}_M(n)}\Big| \frac{1}{M} \sum_{p=1}^M r_p \widehat{Z}^{(p),G}_{n+1}\Big| \Bigg]$$
$$\lesssim_{R,T} \frac{\|g\|_\infty}{\sqrt{M}} + \frac{\rho_M\delta^4_M\eta^3_M\|\psi\|}{\sqrt{M}},$$
for every $M\ge 1$. 
\end{lemma}  
\begin{proof}
For $G \in \mathcal{C}_M(n)$, we recall 

$$
\widehat{Z}^{(p),G}_{n+1}= \widehat{\mathbb{V}}^M_{n+1} \big( O^{(p),G}_{n+1}\big),\quad O^{(p),G}_{n+1}= \big(O^{(p)}_{n}, \mathcal{W}^{(p)}_{n+1}, \Delta X^{(p),G}_{n+1}\big),$$
where 

$$\Delta X^{(p),G}_{n+1} = \blacktriangle x_{n+1} \big(\pi_2(O^{(p)}_{n}),G(O^{(p)}_{n}) , \Delta T^{(p)}_{n+1},\ell_n(\pi_3(O^{(p)}_{n}),\mathcal{W}^{(p)}_{n+1})\big) $$
for $n=m-1,\ldots, 0$. Having this in mind, we set  

$$\phi_{n,G}(\mathbf{o}_n):= \widehat{\mathbb{V}}^M_{n+1} \Big(\mathbf{o}_{n}, \mathcal{W}^{(p)}_{n+1}, \blacktriangle x_{n+1} \big(\pi_2(\mathbf{o}_{n}),G(\mathbf{o}_{n}) , \Delta T^{(p)}_{n+1},\ell_n(\pi_3(\mathbf{o}_n)),\mathcal{W}^{(p)}_{n+1}\big) \Big).$$
Triangle inequality yields

\begin{eqnarray*}
\mathbb{E}\Bigg[ \sup_{G \in \mathcal{C}_M(n)}\Big| \frac{1}{M} \sum_{p=1}^M r_p \phi_{n,G}(O^{(p)}_n)\Big| \Bigg]&\le& \mathbb{E}\Bigg[ \sup_{G \in \mathcal{C}_M(n)}\Big| \frac{1}{M} \sum_{p=1}^M r_p \phi_{n,0}(O^{(p)}_n)\Big|\Bigg]\\
&+& \mathbb{E}\Bigg[ \sup_{G \in \mathcal{C}_M(n)}\Big| \frac{1}{M} \sum_{p=1}^M r_p \Big\{\phi_{n,G}(O^{(p)}_n) - \phi_{n,0}(O^{(p)}_n)\Big\} \Big|\Bigg]
\end{eqnarray*}
By applying Cauchy-Schwartz, observing that $(r_p)_{p=1}^\infty$ is iid with zero mean and $|r_p|^2=1$ a.s. and the fact that $\|\widehat{\mathbb{V}}^M_{n+1}\|_\infty \le \|g\|_\infty$, we get 

\begin{eqnarray*}
\mathbb{E}\Bigg[\Big| \frac{1}{M} \sum_{p=1}^M r_p \phi_{n,0}(O^{(p)}_n)\Big|\Bigg]&\le& \frac{1}{M}\sqrt{\mathbb{E}\Big| \sum_{p=1}^M r_p \phi_{n,0}(O^{(p)}_n)\Big|^2}\\
&=& \frac{1}{M}\sqrt{\mathbb{E}\sum_{p=1}^M |r_p|^2 |\phi_{n,0}(O^{(p)}_n)|^2}\\
&\le& \frac{\sqrt{M}}{M}\|g\|_\infty 
\end{eqnarray*}
By Lemma \ref{LipNNV}, we know that 

\begin{small}
$$\Big|\phi_{n,G}(O^{(p)}_n) - \phi_{n,0}(O^{(p)}_n)\Big| \le \delta_M\|\psi\|\eta_M 
$$

$$\times \Big|\blacktriangle x_{n+1} \big(\pi_2(O^{(p)}_{n}),G(O^{(p)}_{n}) , \Delta T^{(p)}_{n+1},\ell_n(\pi_3(O^{(p)}_{n}),\mathcal{W}^{(p)}_{n+1})\big) - \blacktriangle x_{n+1} \big(\pi_2(O^{(p)}_{n}),0 , \Delta T^{(p)}_{n+1},\ell_n(\pi_3(O^{(p)}_{n}),\mathcal{W}^{(p)}_{n+1})\big)\Big|$$
$$\lesssim_T\delta_M\|\psi\|\eta_M  \max_{1\le p \le M}C(\Delta A^{1,(p)}_{n+1}, \ldots, \Delta A^{d,(p)}_{n+1})|G(O^{(p)}_{n})|$$
$$\lesssim_{T,R}\delta_M\|\psi\|\eta_M  \max_{1\le p \le M}C(\Delta A^{1,(p)}_{n+1}, \ldots, \Delta A^{d,(p)}_{n+1})\delta_M \eta_M$$
\end{small} 
where $\max_{1\le p \le M}C(\Delta A^{1,(p)}_{n+1}, \ldots, \Delta A^{d,(p)}_{n+1})$ is independent of $O^{(p)}_n$. Here, we use the linear growth of the output layer $\mathbf{a}$ defined in (\ref{controlf}) and the bound $R$ on the state space $\mathbb{H} = \mathbb{W}\times S$ for a compact subset $S \in \mathbb{R}^n$. By applying Theorem \ref{talagrandlemma}, we get 

$$\mathbb{E}\Bigg[ \sup_{G \in \mathcal{C}_M(n)}\Big| \frac{1}{M} \sum_{p=1}^M r_p \Big\{\phi_{n,G}(O^{(p)}_n) - \phi_{n,0}(O^{(p)}_n)\Big\} \Big|\Bigg]$$
$$\lesssim_T  \delta^2_M\|\psi\|\eta^2_M \rho_M \mathbb{E}\Bigg[ \sup_{G \in \mathcal{C}_M(n)}\Bigg| \frac{1}{M} \sum_{p=1}^M r_p G(O^{(p)}_n)\Bigg|\Bigg].$$
Now, we want to make use of the so-called “Frank-Wolfe step” as discussed in \cite{bach}: By assumption $\mathbf{a}$ is linear. In this case, we can apply identity (2) of \cite{bach} jointly with the remark in page 6 of \cite{bach} and arrive at the pathwise identity 

$$
\sup_{G \in \mathcal{C}_M(n)}\Bigg|  \sum_{p=1}^M r_p G(O^{(p)}_n)\Bigg| = \delta_M \max_{\|v\|\le \delta_M\eta_M} \Bigg|  \sum_{p=1}^M r_p \big(v^\top O^{(p)}_n\big)^+\Bigg|
$$
so that 
\begin{eqnarray*}
\mathbb{E}\Bigg[\sup_{G \in \mathcal{C}_M(n)}\Bigg|  \sum_{p=1}^M r_p G(O^{(p)}_n)\Bigg|\Bigg] &=& \delta_M \mathbb{E}\Bigg[\max_{\|v\|\le \delta_M \eta_M} \Bigg|  \sum_{p=1}^M r_p \big(v^\top O^{(p)}_n\big)^+\Bigg|\Bigg]\\
&\le& \delta_M \mathbb{E}\Bigg[\max_{\|v\|\le \delta_M\eta_M} \Bigg|  \sum_{p=1}^M r_p v^\top O^{(p)}_n\Bigg|\Bigg]\\
&=& \delta_M \mathbb{E}\Bigg[\max_{\|v\|\le \delta_M\eta_M} \Bigg| \Big\langle v,  \sum_{p=1}^M r_p  O^{(p)}_n\Big\rangle_{\mathbb{H}^n} \Bigg|\Bigg]\\
&\le& \delta_M \delta_M \eta_M \mathbb{E} \Bigg[  \Bigg|\sum_{p=1}^M r_p  O^{(p)}_n \Bigg|\Bigg]\\
&\lesssim_{R}&\delta^2_M \eta_M\sqrt{N}
\end{eqnarray*}
\end{proof}
\textbf{Proof of Lemma \ref{estilemma}:} Recall that

$$\varepsilon^{\text{esti}}_n=\sup_{C \in \mathcal{C}_M(n)} \Bigg| \frac{1}{M} \sum_{p=1}^M \hat{Z}^{(p),C}_{n+1}  - \mathbb{E}_{M,O_n}[\widehat{\mathbb{V}}^M_{n+1}(O^C_{n+1})] \Bigg|,$$
for $n=0, \ldots, m-1$. Observe we shall write 

$$\mathbb{E}_{M,O_n}[\widehat{\mathbb{V}}^M_{n+1}(O^C_{n+1})]  = \frac{1}{M}\sum_{p=1}^M\big[\widehat{Z}^{(p),C}_{n+1}\big] = \frac{1}{M}\sum_{p=1}^M \mathbb{E}\Big[\widehat{\mathbb{V}}^M_{n+1} \big( O^{(p),C}_{n+1}\big)\Big],$$
for $n=0, \ldots, m-1$. By applying Lemmas \ref{leest2}, \ref{leest3} and \ref{leest4}, we conclude the proof of Lemma \ref{estilemma}.

\subsection{Proof of Lemma \ref{approxlemma}}

\begin{lemma}\label{Lipvalue}
Under assumptions (H0-H1-H2), the value functions satisfy

$$|\mathbb{V}_j(\mathbf{o}_j)) - \mathbb{V}_j(\mathbf{o}'_j)|\le \| \mathbb{V}_{j+1}\|_\infty\| r_j\| |\mathbf{o}_j - \mathbf{o}'_j|,$$
for $j=m-1, \ldots, 0$. 
\end{lemma}
\begin{proof}
Recall (see (\ref{valuefunc}))
$$
\mathbf{U}_j(\mathbf{o}_j,\theta) = \int_{\mathbb{W}}\mathbb{V}_{j+1}\Big(\mathbf{o}_{j}, \mathfrak{X}_{j+1}(\theta, \mathbf{o}_{j}, w)\Big)\nu(dw)
$$
for $\mathbf{o}_j = (w_1, y_1, \ldots, w_j,y_j)$ and $j=m-1, \ldots, 0$. Here,  
\begin{eqnarray*}
\mathfrak{X}_{j+1}(\theta, \mathbf{o}_{j}, \mathcal{W}_{j+1})&\stackrel{d}{=}&\Big(\mathcal{W}_{j+1}, \blacktriangle x_{j+1}\big(\pi_2(\textbf{o}_j),\theta, \Delta T_{j+1}, \ell_{j+1}(\pi_3(\mathbf{o}_j),\mathcal{W}_{j+1})\big)\Big)\\
&\stackrel{d}{=}& \Big(\mathcal{W}_{j+1}, \Delta X^\theta_{j+1}\Big)
\end{eqnarray*}
knowing that $\Xi_{j}= \mathbf{o}_j$, for $j=m-1, \ldots, 0$. 
Then, 
\begin{eqnarray*}
\mathbf{U}_j(\mathbf{o}_j,\theta) &=& \int_{\mathbb{W}\times \mathbb{R}^n}\mathbb{V}_{j+1}\Big(\mathbf{o}_{j}, x,x'\Big)\mathbb{P}[ (\mathcal{W}_j, \Delta X^\theta_j) \in dxdx' | O_{j}= \mathbf{o}_{j}]\\
&=& \int_{\mathbb{W}\times\mathbb{R}^n}\mathbb{V}_{j+1}\Big(\mathbf{o}_{j}, x,x'\Big)r_j(\theta,x',\mathbf{o}_j)\mu(dx')\nu(dx)
\end{eqnarray*}
Then, 

\begin{eqnarray*}
|\mathbf{U}_j(\mathbf{o}_j,\theta) - \mathbf{U}_j(\mathbf{o}'_j,\theta)|&\le&\| \mathbb{V}_{j+1}\|_\infty \int_{\mathbb{R}^n}|r_j(\theta,x',\mathbf{o}_j) - r_j(\theta,x',\mathbf{o}'_j)|\mu(dx')\\
&\le& \| \mathbb{V}_{j+1}\|_\infty\| r_j\| |\mathbf{o}_j - \mathbf{o}'_j|
\end{eqnarray*}
By definition of the infimum, given $\eta>0$, there exists $\theta$ such that 
$$\mathbb{V}_j(\mathbf{o}'_j))\ge \mathbf{U}_j(\mathbf{o}'_j,\theta)-\eta.$$
Then, 

\begin{eqnarray*}
\mathbb{V}_j(\mathbf{o}_j)) - \mathbb{V}_j(\mathbf{o}'_j)&\le& \mathbf{U}_j(\mathbf{o}_j,\theta) -\mathbb{V}_j(\mathbf{o}'_j)\\
&\le&  \eta - \mathbf{U}_j(\mathbf{o}'_j,\theta) + \mathbf{U}_j(\mathbf{o}_j,\theta)\\
&\le& |\mathbf{U}_j(\mathbf{o}_j,\theta) - \mathbf{U}_j(\mathbf{o}'_j,\theta)| + \eta\\
&\le& \| \mathbb{V}_{j+1}\|_\infty\| r_j\| |\mathbf{o}_j - \mathbf{o}'_j| + \eta.
\end{eqnarray*}
Similarly, swapping the roles of $\mathbf{o}_j$ and $\mathbf{o}'_j$ and repeating the same argument, we get 

$$\mathbb{V}_j(\mathbf{o}_j)) - \mathbb{V}_j(\mathbf{o}'_j)\le \| \mathbb{V}_{j+1}\|_\infty\| r_j\| |\mathbf{o}_j - \mathbf{o}'_j|+\eta.$$
Hence, 
$$|\mathbb{V}_j(\mathbf{o}_j)) - \mathbb{V}_j(\mathbf{o}'_j)|\le \| \mathbb{V}_{j+1}\|_\infty\| r_j\| |\mathbf{o}_j - \mathbf{o}'_j|.$$
\end{proof}

\textbf{Proof of Lemma \ref{approxlemma}:} Recall 
\begin{equation*}
\varepsilon^{\text{approx}}_n= \inf_{C \in \mathcal{C}_M(n)} \mathbb{E}_M[\widehat{\mathbb{V}}^M_{n+1} (O^C_{n+1})] - \inf_{L \in \mathbb{A}^{\mathbb{H}^n}}\mathbb{E}_M[\widehat{\mathbb{V}}^M_{n+1} (O^L_{n+1})],
\end{equation*}
for $n=m-1, \ldots, 0$. We have

\begin{eqnarray*}
\varepsilon^{\text{approx}}_n&\le& \inf_{C \in \mathcal{C}_M(n)} \mathbb{E}_M[\widehat{\mathbb{V}}^M_{n+1} (O^C_{n+1})] - \mathbb{E}[\mathbb{V}_n(O_n)]\\
&+& \mathbb{E}[\mathbb{V}_n(O_n)] - \inf_{L \in \mathbb{A}^{\mathbb{H}^n}}\mathbb{E}_M[\widehat{\mathbb{V}}^M_{n+1} (O^L_{n+1})].
\end{eqnarray*}
Fix $n \in \{0, \ldots, m-1\}$. By Theorem 4.2 in \cite{leaoohashi}, we know that

$$\mathbb{V}_n(O_n)= \inf_{a \in \mathbb{A}}\mathbb{E}\big[\mathbb{V}_{n+1}(O^a_{n+1})|O_n\big] = \inf_{a \in \mathbb{A}}\mathbb{E}\big[\mathbb{V}_{n+1}(O_n, \mathcal{X}^a_{n+1})|O_n\big],$$
where $\mathbb{E}\big[\mathbb{V}_{n+1}(O^a_{n+1})|O_n\big]$ denotes the conditional expectation w.r.t. $\mathbb{P}$ knowing that $\Delta X^a_{n+1}$ is controlled by $a \in \mathbb{A}$ on a given history $\Xi_n=O_{n}$. By assumption, there exists 

$$a_n^{\text{opt}} \in \arg \min_{a \in \mathbb{A}}\mathbb{E}\big[\mathbb{V}_{n+1}(O_n, \mathcal{X}^a_{n+1})|O_n\big]$$
so that $\mathbb{V}_n(O_n)= \mathbb{E}\big[\mathbb{V}_{n+1}(O_n, \mathcal{X}^{a^{\text{opt}}_n}_{n+1})|O_n\big].$ We then write

\begin{eqnarray*}
\mathbb{E}_M[\widehat{\mathbb{V}}^M_{n+1} (O^C_{n+1})] - \mathbb{E}[\mathbb{V}_n(O_n)]&=&\mathbb{E}_M[\widehat{\mathbb{V}}^M_{n+1} (O^C_{n+1})] - \mathbb{E}[\mathbb{V}_{n+1}(O_n, \mathcal{X}^{a^{\text{opt}}_n}_{n+1})]\\
&=&\mathbb{E}_M[\widehat{\mathbb{V}}^M_{n+1} (O^C_{n+1})] - \mathbb{E}[\mathbb{V}_{n+1}(O_n, \mathcal{X}^{C}_{n+1})]\\
&+&\mathbb{E}[\mathbb{V}_{n+1}(O_n, \mathcal{X}^{C}_{n+1})] - \mathbb{E}[\mathbb{V}_{n+1}(O_n, \mathcal{X}^{a^{\text{opt}}_n}_{n+1})]
\end{eqnarray*}
Recall 

$$
\nonumber\mathbb{P}[ (\mathcal{W}_j, \Delta X^a_j) \in dxdx' | O_{j-1}= b]= r_j(a,x'; b)\mu(dx')\nu(dx),\quad O^C_{n+1}= \big(O_{n}, \mathcal{W}_{n+1}, \Delta X^C_{n+1}\big),$$
and 
$$\Delta X^C_{n+1}= \blacktriangle x_{n+1} \big(\pi_2(O_{n}),C(O_{n}) , \bar{W}^{(1)} ,\ell_n(\pi_3(O_{n}),\bar{W})\big).
$$
Hence, 
$$\big|\mathbb{E}_M[\widehat{\mathbb{V}}^M_{n+1} (O^C_{n+1})] - \mathbb{E}[\mathbb{V}_{n+1}(O_n, \mathcal{X}^{C}_{n+1})]\big|$$
$$\le \|r\|_\infty\int \big|\widehat{\mathbb{V}}^M_{n+1} (O_{n}, x,x') - \mathbb{V}_{n+1} (O_{n}, x,x')|\mu(dx')\nu(dx) $$
$$=\|r\|_\infty \mathbb{E}_M\big[|\widehat{\mathbb{V}}^M_{n+1} (O_{n+1}) - \mathbb{V}_{n+1} (O_{n+1})|\big].$$
By Lemma \ref{Lipvalue} and assumption H2, we have 

$$\Big|\mathbb{E}[\mathbb{V}_{n+1}(O_n, \mathcal{X}^{C}_{n+1})] - \mathbb{E}[\mathbb{V}_{n+1}(O_n, \mathcal{X}^{a^{\text{opt}}_n}_{n+1})]\Big|$$
$$\le \|r\| \|\mathbb{V}_{n+1} \|_\infty\rho_M \mathbb{E}_M\Big[\Big| C(O_n) - a_n^{\text{opt}}(O_n) \Big|\Big].$$
Therefore, 
\begin{eqnarray*}
\inf_{C\in \mathcal{C}_M(n)}\mathbb{E}_M[\widehat{\mathbb{V}}^M_{n+1} (O^C_{n+1})] - \mathbb{E}[\mathbb{V}_n(O_n)]&\le& \|r\| \|\mathbb{V}_{n+1} \|_\infty\rho_M \mathbb{E}_M\Big[\Big| C(O_n) - a_n^{\text{opt}}(O_n) \Big|\Big]\\
&+& \|r\|_\infty \mathbb{E}_M\big[|\widehat{\mathbb{V}}^M_{n+1} (O_{n+1}) - \mathbb{V}_{n+1} (O_{n+1})|\big].
\end{eqnarray*}
Now, for $\eta >0$, there exists $G_\eta \in \mathbb{A}^{\mathbb{H}^n}$ such that 

$$\inf_{C \in \mathbb{A}^{\mathbb{H}^n}}\mathbb{E}_M[\widehat{\mathbb{V}}^M_{n+1} (O^C_{n+1})]+\eta \ge \mathbb{E}_M[\widehat{\mathbb{V}}^M_{n+1} (O^{G_\eta}_{n+1})] $$

$$\mathbb{E}[\mathbb{V}_n(O_n)] = \inf_{L \in \mathbb{A}^{\mathbb{H}^n}} \mathbb{E}[\mathbb{V}_{n+1}(O_n, \mathcal{X}^L_{n+1})]\le \mathbb{E}[\mathbb{V}_{n+1} (O^{G_\eta}_{n+1})] $$
Then, assumption (H1) yields 

$$\mathbb{E}[\mathbb{V}_n(O_n)] - \inf_{L \in \mathbb{A}^{\mathbb{H}^n}}\mathbb{E}_M[\widehat{\mathbb{V}}^M_{n+1} (O^L_{n+1})].
$$
$$\le \mathbb{E}[\mathbb{V}_{n+1} (O^{G_\eta}_{n+1})] -\mathbb{E}_M[\widehat{\mathbb{V}}^M_{n+1} (O^{G_\eta}_{n+1})]  + \eta $$
$$\le \|r\|_\infty \mathbb{E}_M \Big[ |\widehat{\mathbb{V}}^M_{n+1} (O_{n+1})- \mathbb{V}_{n+1} (O_{n+1})\Big| \Big|  \Big] + \eta$$
This concludes the proof of Lemma \ref{approxlemma}. 

 
\section{Proof of Theorem \ref{mainresultRS}}\label{appendix3}

\begin{lemma}\label{Lipweakfor}
Let $\mathcal{H}$ be the set of probability kernels described in (\ref{rsetd1}) and (\ref{rsetd0}). Under Assumption (R1), the associated value functions satisfy

$$|\mathbb{V}^\mathcal{H}_j(\mathbf{o}_j)) - \mathbb{V}^{\mathcal{H}}_j(\mathbf{o}'_j)|\le \| \mathbb{V}^\mathcal{H}_{j+1}\|_\infty\|\rho\| |\mathbf{o}_j - \mathbf{o}'_j|,$$
for $j=m-1, \ldots, 0$. 
\end{lemma}

\begin{proof}
By using assumption (R1), 

\begin{eqnarray*}
\mathbf{U}^\mathcal{H}_{j}(\mathbf{o}_j,\pi_j) &=& \int_{\mathbb{H}} \mathbb{V}^{\mathcal{H}}_{j+1}(\mathbf{o}_j,x,x')\mu_j^{\pi_j}(dx'|\mathbf{o}_j)\nu(dx)\\
&=& \int_{\mathbb{H}} \mathbb{V}^{\mathcal{H}}_{j+1}(\mathbf{o}_j,x,x')\rho^{\pi_j}_j(\mathbf{o}_j,x')\mu(dx')\nu(dx)
\end{eqnarray*}
for $\pi_j \in \mathcal{H}, \mathbf{o}_j \in \mathbb{H}^j$, $j=m-1, \ldots, 0$. Then, 

\begin{eqnarray*}
|\mathbf{U}^{\mathcal{H}}_j(\mathbf{o}_j,\pi_j) - \mathbf{U}^{\mathcal{H}}_j(\mathbf{o}'_j,\pi_j)|&\le&\| \mathbb{V}^{\mathcal{H}}_{j+1}\|_\infty \int_{\mathbb{R}^q}|\rho^{\pi_j}(\mathbf{o}_j,x') - \rho^{\pi_j}(\mathbf{o}'_j,x')|\mu(dx')\\
&\le& \| \mathbb{V}^{\mathcal{H}}_{j+1}\|_\infty\| \rho\| |\mathbf{o}_j - \mathbf{o}'_j|
\end{eqnarray*}
The rest of the proof is similar to Lemma \ref{Lipvalue}. 
\end{proof}
The action (\ref{bracket}), Lemma \ref{Lipweakfor}, assumptions (E1, H0-H2-R1 and T1) and routine duality arguments allow us to prove Theorem \ref{mainresultRS} in a similar way as Theorem \ref{mainresult}. Then, we omit the details.

\section{Proofs of Lemmas \ref{repWH}, \ref{Ylemma}, \ref{glemma}, \ref{fdeltab}}\label{appendix4} 

\subsection{Proof of Lemma \ref{repWH}}
\begin{proof}
Clearly, $B_H^{k,i}(T_n^k)=0$, for $n=0,1$ and $i=1,2$. By definition,
\[
B_H^{k,i}(t)
= \int_0^{\bar t_k} \partial_s K_{H,1}(\bar t_k,s)\big(A^{k,i}(\bar t_k)-A^{k,i}(\bar s_k^+)\big)\,ds
  - \int_0^{\bar t_k} \partial_s K_{H,2}(\bar t_k,s)\,A^{k,i}(s)\,ds.
\]
In case $t=T_n^k$, $\bar t_k=T_n^k$ and, for $s\in(T_{j-1}^k,T_j^k]$, we have $\bar s_k^+=T_j^k$. Moreover $A^{k,i}$ is c\`adl\`ag and piecewise constant on each open interval $(T_{j-1}^k,T_j^k)$ with value $A^{k,i}(T_{j-1}^k)$. Partition the integral over the grid: 
\begin{small}
\begin{align*}
\int_0^{T_n^k} \partial_s K_{H,1}(T_n^k,s)\big(A^{k,i}(T_n^k)-A^{k,i}(\bar s_k^+)\big)\,ds
&= \sum_{j=1}^n \int_{T_{j-1}^k}^{T_j^k} \partial_s K_{H,1}(T_n^k,s)\big(A^{k,i}(T_n^k)-A^{k,i}(T_j^k)\big)\,ds \\
&= \sum_{j=1}^n \big(A^{k,i}(T_n^k)-A^{k,i}(T_j^k)\big)\,
   \big(K_{H,1}(T_n^k,T_j^k)-K_{H,1}(T_n^k,T_{j-1}^k)\big),
\end{align*}
\end{small}
By the fundamental theorem of calculus in the $s$ variable (the kernels are smooth for $0<s<T_n^k$ and integrable near $s=0,T_n^k$ for $H\in(0,1/2)$) and using $A^{k,i}(s)=A^{k,i}(T_{j-1}^k)$ for $s\in(T_{j-1}^k,T_j^k)$, we have 
\begin{small}
\begin{align*}
\int_0^{T_n^k} \partial_s K_{H,2}(T_n^k,s)\,A^{k,i}(s)\,ds
&= \sum_{j=1}^n \int_{T_{j-1}^k}^{T_j^k} \partial_s K_{H,2}(T_n^k,s)\,A^{k,i}(T_{j-1}^k)\,ds \\
&= \sum_{j=1}^n A^{k,i}(T_{j-1}^k)\,\big(K_{H,2}(T_n^k,T_j^k)-K_{H,2}(T_n^k,T_{j-1}^k)\big).
\end{align*}
\end{small}
Observe 
\[
A^{k,i}(T_n^k)-A^{k,i}(T_j^k)=\sum_{m=j+1}^n \Delta A^{k,i}(T_m^k),
\]
and interchange the order of summation:
\begin{align*}
\sum_{j=1}^n \big(A^{k,i}(T_n^k)-A^{k,i}(T_j^k)\big)\,&\big(K_{H,1}(T_n^k,T_j^k)-K_{H,1}(T_n^k,T_{j-1}^k)\big) \\
&= \sum_{m=2}^n \Delta A^{k,i}(T_m^k)\, \sum_{j=1}^{m-1} \big(K_{H,1}(T_n^k,T_j^k)-K_{H,1}(T_n^k,T_{j-1}^k)\big) \\
&= \sum_{m=2}^n \Delta A^{k,i}(T_m^k)\,\big(K_{H,1}(T_n^k,T_{m-1}^k)-K_{H,1}(T_n^k,0)\big).
\end{align*}
Since $K_{H,1}(t,0)=c_H t^{H-\frac12} 0^{\frac12-H} (t-0)^{H-\frac12}=0$ for $H\in(0,\tfrac12)$, this becomes
\[
\sum_{m=2}^n \Delta A^{k,i}(T_m^k)\,K_{H,1}(T_n^k,T_{m-1}^k).
\]

For the second sum, note $A^{k,i}(T_{j-1}^k)=\sum_{m=1}^{j-1}\Delta A^{k,i}(T_m^k)$, hence
\begin{align*}
- \sum_{j=1}^n A^{k,i}(T_{j-1}^k)\,&\big(K_{H,2}(T_n^k,T_j^k)-K_{H,2}(T_n^k,T_{j-1}^k)\big) \\
&= - \sum_{m=1}^{n-1} \Delta A^{k,i}(T_m^k)\, \sum_{j=m+1}^{n} \big(K_{H,2}(T_n^k,T_j^k)-K_{H,2}(T_n^k,T_{j-1}^k)\big) \\
&= - \sum_{m=1}^{n-1} \Delta A^{k,i}(T_m^k)\, \big(K_{H,2}(T_n^k,T_n^k)-K_{H,2}(T_n^k,T_m^k)\big).
\end{align*}
By definition $K_{H,2}(t,t)=0$ (the inner integral is over $[s,t]$ and vanishes at $s=t$), hence this equals
\[
\sum_{m=1}^{n-1} \Delta A^{k,i}(T_m^k)\,K_{H,2}(T_n^k,T_m^k).
\]

Combining the two rearranged pieces gives
\[
B_H^{k,i}(T_n^k)
= \sum_{j=2}^n \Delta A^{k,i}(T_j^k)\,K_{H,1}(T_n^k,T_{j-1}^k)
 + \sum_{j=1}^{n-1} \Delta A^{k,i}(T_j^k)\,K_{H,2}(T_n^k,T_j^k),
\]
as claimed.
\end{proof}

\subsection{Proof of Lemma \ref{Ylemma}}

\begin{proof}
Recall the expression (\ref{stv2}) for $Y(t,x,\varepsilon)$. Since $t\in[\underline M,\overline M]$, we have the uniform Gaussian bound
\[
\phi_t(y)
= \frac{1}{\sqrt{2\pi t}}e^{-y^2/(2t)}
\;\le\;
\frac{1}{\sqrt{2\pi \underline M}}
\exp\!\Big(-\frac{y^2}{2\overline M}\Big)
=: C_0\, e^{-c_0 y^2},
\]
where
\[
C_0=\frac{1}{\sqrt{2\pi \underline M}},
\qquad 
c_0=\frac{1}{2\overline M}.
\]

Fix $x\in (-\varepsilon,\varepsilon)$.  
For all $m\ge1$ we have the estimates
\[
|x+4m\varepsilon|\ge (4m-1)\varepsilon,
\qquad
|x+2\varepsilon+4(m-1)\varepsilon|\ge (4m-3)\varepsilon,
\]
\[
|x-4m\varepsilon|\ge (4m-1)\varepsilon,
\qquad
|x-2\varepsilon-4(m-1)\varepsilon|\ge (4m-3)\varepsilon.
\]
Therefore,
\[
\begin{aligned}
&\big|\phi_t(x+4m\varepsilon)
      -\phi_t(x+2\varepsilon+4(m-1)\varepsilon)\big|
\\ &\qquad\le 
      \phi_t(x+4m\varepsilon)
     +\phi_t(x+2\varepsilon+4(m-1)\varepsilon)
\\ &\qquad\le 
      C_0\exp\!\big(-c_0(4m-1)^2\varepsilon^2\big)
     +C_0\exp\!\big(-c_0(4m-3)^2\varepsilon^2\big)
\\ &\qquad\le 
      2C_0\exp\!\big(-c_0(4m-3)^2\varepsilon^2\big).
\end{aligned}
\]
An identical estimate holds for the terms
\[
\phi_t(x-4m\varepsilon)-\phi_t(x-2\varepsilon-4(m-1)\varepsilon).
\]

Hence, for all $t\in[\underline M,\overline M]$ and $|x| < \varepsilon$,
\[
|Y(t,x,\varepsilon)|
\;\le\;
2 \sum_{m=1}^\infty 
   2C_0 \exp\!\big(-c_0(4m-3)^2\varepsilon^2\big)
\;=\;
4C_0 \sum_{m=1}^\infty 
     \exp\!\big(-c_0(4m-3)^2\varepsilon^2\big).
\]

The series 
\[
\sum_{m=1}^\infty \exp\!\big(-c_0(4m-3)^2\varepsilon^2\big)
\]
is finite (dominated by $\sum_{m\ge1} e^{-c_0\varepsilon^2 m^2}$).  
Thus the right-hand side is a finite constant depending only on 
$(\underline M,\overline M,\varepsilon)$ and not on $t$ or $x$. This proves
\[
\sup_{\underline M\le t\le\overline M}\ 
\sup_{|x| < \varepsilon} |Y(t,x,\varepsilon)| < \infty.
\]

For the derivative. We split the proof into 4 steps. 

\

\noindent Step 1.First note that $\phi_t$ is smooth in $x$, with
\[
\partial_x \phi_t(x)
= -\frac{x}{t}\,\phi_t(x),
\quad t>0.
\]
Observe 
\[
|\partial_x \phi_t(y)|
= \Big|\frac{y}{t}\phi_t(y)\Big|
\le \frac{|y|}{t_{\min}} C_0 e^{-c_0 y^2}
=: C_1 |y| e^{-c_0 y^2},
\]
for all $t\in[\underline{M},\overline{M}]$, $y\in\mathbb{R}$, where $C_1>0$ depends only on $\underline{M}$ and $\overline{M}$. Formally,
\[
\begin{aligned}
\partial_x Y(t,x,\varepsilon)
&= \sum_{m=1}^{\infty}
    \Big[ \partial_x \phi_t(x+4m\varepsilon)
          - \partial_x \phi_t\big(x+2\varepsilon+4(m-1)\varepsilon\big) \Big] \\
&\quad + \sum_{m=1}^{\infty}
    \Big[ \partial_x \phi_t(x-4m\varepsilon)
          - \partial_x \phi_t\big(x-2\varepsilon-4(m-1)\varepsilon\big) \Big].
\end{aligned}
\]
We now estimate the $m$-th term uniformly on 
$t\in[\underline{M},\overline{M}]$ and $|x|<\varepsilon$. Fix $|x|<\varepsilon$ and $m\ge1$. Then
\[
|x \pm 4m\varepsilon| \ge (4m-1)\varepsilon,
\qquad
\big|x \pm (2\varepsilon + 4(m-1)\varepsilon)\big| \ge (4m-3)\varepsilon.
\]
Using the bound on $\partial_x\phi_t$, we get
\[
\begin{aligned}
&\big|\partial_x \phi_t(x+4m\varepsilon)
      - \partial_x \phi_t\big(x+2\varepsilon+4(m-1)\varepsilon\big)\big|
\\
&\qquad\le 
|\partial_x \phi_t(x+4m\varepsilon)| 
+ |\partial_x \phi_t\big(x+2\varepsilon+4(m-1)\varepsilon\big)| \\
&\qquad\le 
C_1 |x+4m\varepsilon| e^{-c_0(x+4m\varepsilon)^2}
+ C_1 \big|x+2\varepsilon+4(m-1)\varepsilon\big|
       e^{-c_0(x+2\varepsilon+4(m-1)\varepsilon)^2} \\
&\qquad\le 
C_1 (4m+1)\varepsilon \, e^{-c_0(4m-1)^2\varepsilon^2}
+ C_1 (4m-1)\varepsilon \, e^{-c_0(4m-3)^2\varepsilon^2}.
\end{aligned}
\]
A similar estimate holds for the ``minus'' terms
$\partial_x \phi_t(x-4m\varepsilon)
 - \partial_x \phi_t(x-2\varepsilon-4(m-1)\varepsilon)$.
Therefore, there exists a constant $C_2>0$ such that, for all
$t\in[\underline{M},\overline{M}]$, $|x|<\varepsilon$, and $m\ge1$,
\[
\big|\text{$m$-th term in } \partial_x Y(t,x,\varepsilon)\big|
\le C_2\, m \, e^{-c (4m-3)^2\varepsilon^2},
\]
for some $c>0$ (take $c\le c_0$).

\medskip

\noindent Step 2: The numerical series
\[
\sum_{m=1}^{\infty} m \, e^{-c (4m-3)^2\varepsilon^2}
\]
converges (e.g. by comparison with $\sum m e^{-c\varepsilon^2 m^2}$). Hence, by the
Weierstrass M-test, the series defining $\partial_x Y(t,x,\varepsilon)$
converges absolutely and uniformly on 
$[\underline{M},\overline{M}]\times(-\varepsilon,\varepsilon)$.

\medskip

\noindent Step 3: Define the partial sums
\[
Y_N(t,x,\varepsilon)
:= \sum_{m=1}^{N} \Big[\phi_t(x+4m\varepsilon)
    -\phi_t\big(x+2\varepsilon+4(m-1)\varepsilon\big)\Big]
 + \sum_{m=1}^{N} \Big[\phi_t(x-4m\varepsilon)
    -\phi_t\big(x-2\varepsilon-4(m-1)\varepsilon\big)\Big].
\]
Each $Y_N$ is $C^1$ in $x$ and $\partial_x Y_N(t,x,\varepsilon)$ is the truncation of the derivative series. By the uniform bounds above,
the sequence $\{\partial_x Y_N\}_{N\ge 1}$ converges uniformly on
$[\underline{M},\overline{M}]\times(-\varepsilon,\varepsilon)$ to the function
\[
\widetilde Y_x(t,x,\varepsilon)
:= \sum_{m=1}^{\infty}
    \Big[\partial_x \phi_t(x+4m\varepsilon)
         - \partial_x \phi_t\big(x+2\varepsilon+4(m-1)\varepsilon\big)\Big]
 +  \sum_{m=1}^{\infty}
    \Big[\partial_x \phi_t(x-4m\varepsilon)
         - \partial_x \phi_t\big(x-2\varepsilon-4(m-1)\varepsilon\big)\Big].
\]
On the other hand, the series defining $Y(t,x,\varepsilon)$ itself converges
uniformly (by the same type of Gaussian estimates), so $Y_N\to Y$ uniformly. By the uniform convergence of $Y_N$ and
uniform convergence of $\partial_x Y_N$), it follows that $Y$ is
continuously differentiable in $x$ and 

\[
\partial_x Y(t,x,\varepsilon) = \widetilde Y_x(t,x,\varepsilon).
\]

Moreover, by the M-test bound,
\[
\sup_{\underline{M}\le t\le \overline{M},\,|x|<\varepsilon}
|\partial_x Y(t,x,\varepsilon)|
\le \sum_{m=1}^{\infty} C_2\, m \, e^{-c (4m-3)^2\varepsilon^2}
< \infty.
\]

\end{proof}

\subsection{Proof of Lemma \ref{glemma}}
\begin{proof}
For $z\in K$, write
\[
g_{c,v}(z)
= \int_0^\infty 
   \frac{1}{|v|}
   H\!\left(t,\frac{z-ct}{v}\right)
   \mathds{1}_{\{|(z-ct)/v|<\varepsilon\}}
   f_J(t)\mathds{1}_{[\underline{M},\overline{M}]}\,dt,
\]
where $H(t,x) = \frac{\phi_t(x) + Y(t,x,\varepsilon)}{p(t,\varepsilon)}$. Equivalently, using 
$\{t>0; |(z-ct)/v|<\varepsilon\} = (\alpha,\beta)$ with
\[
\alpha=\alpha(z,c,v)=\frac{z-\varepsilon v}{c},\qquad
\beta=\beta(z,c,v)=\frac{z+\varepsilon v}{c},
\]
we obtain the representation
\[
g_{c,v}(z)
=
\int_{\alpha(z,c,v)\vee \underline{M}}^{\beta(z,c,v)\wedge \overline{M}}
   F(z,c,v,t)\,dt,
\qquad
F(z,c,v,t)
= \frac{1}{|v|}\,H\!\left(t,\frac{z-ct}{v}\right)f_J(t).
\]
Observe $\alpha(z,c,v)\vee \underline{M}, \beta(z,c,v)\wedge \overline{M} \in [\underline{M}, \overline{M}]$, for every $z \in K, 0< c_{\min}\le |c|\le c_{\max} $ and $0 < v_{\min}\le |v|\le v_{\max}$. Moreover, 

$$\partial_x H(t,x)=\frac{1}{p_t(\varepsilon)}\Bigg\{\frac{-x}{t}\phi_t(x) + \partial_x Y(t,x,\varepsilon)\Bigg\}.$$


Hence, Lemma \ref{Ylemma} yields 
\[
\sup_{t\in[\underline{M},\overline{M}],\,|x|\le\varepsilon}
\Big\{|H(t,x)| + |\partial_x H(t,x)|\Big\} \le C_H,
\qquad
\sup_{t\in[\underline{M},\overline{M}]} |f_J(t)| \le C_J.
\]
for positive constants $C_H$ and $C_J$.  

Next, we evaluate the derivative w.r.t. $v$. Using Leibniz' rule,
\[
\partial_v g_{c,v}(z)
=
F(z,c,v,\beta)\mathds{1}_{[\underline{M},\overline{M}]}\,\partial_v \beta
-
F(z,c,v,\alpha)\mathds{1}_{[\underline{M},\overline{M}]}\,\partial_v \alpha
+
\int_{\alpha}^{\beta} \partial_v F(z,c,v,t)\mathds{1}_{[\underline{M},\overline{M}]}\,dt.
\]

The boundary terms are uniformly bounded because
\[
|\partial_v \alpha|=|\partial_v\beta|=\frac{\varepsilon}{|c|}
\le \frac{\varepsilon}{c_{\min}},
\qquad
|F|\le \frac{1}{v_{\min}} C_H C_J.
\]
Moreover, 
\[
\partial_v F
=
\left(\frac{d}{dv}\frac{1}{|v|}\right)
  H\!\left(t,\tfrac{z-ct}{v}\right)f_J(t)
+
\frac{1}{|v|}
 \partial_x H\!\left(t,\tfrac{z-ct}{v}\right)
 \Big(-\tfrac{z-ct}{v^2}\Big)
 f_J(t).
\]
On the domain of integration we have $|(z-ct)/v|<\varepsilon$, and 
$|v|\ge v_{\min}$, so
\[
\left|\partial_v F\right|
\le
\left(\frac{1}{v_{\min}^2} C_H + \frac{\varepsilon}{v_{\min}^2} C_H\right)C_J
= C_v'.
\]
The length of the interval is
\[
|\beta-\alpha|=\frac{2\varepsilon |v|}{|c|}
\le \frac{2\varepsilon v_{\max}}{c_{\min}}.
\]
Thus
\[
|\partial_v g_{c,v}(z)| \le C_v
\]
for every $z\in K$, with $C_v$ independent of $(c,v,z)$.

Next, we evaluate the derivative w.r.t. $c$. The proof is identical. Observe 
\[
\partial_c  \alpha = -\frac{z-\varepsilon v}{c^2},
\qquad
\partial_c\beta  = -\frac{z+\varepsilon v}{c^2},
\]
which are uniformly bounded since $z\in K$ and $|c|\ge c_{\min}$.
Differentiating $F$ in $c$ gives
\[
\partial_c F
=
\frac{1}{|v|}\,
\partial_x H\!\left(t,\tfrac{z-ct}{v}\right)
\left(-\frac{t}{v}\right)f_J(t),
\]
which is uniformly bounded because $t\in[\underline{M},\overline{M}]$ 
and $|v|\ge v_{\min}>0$.
The interval length is the same as before. Thus
\[
|\partial_c g_{c,v}(z)|\le C_c,
\qquad z\in K, 0 < c_{\min}\le |c|\le c_{\max}, 0 < v_{\min}\le |v|\le v_{\max},
\]
for a constant $C_c$ independent of $(c,v,z)$. Combining the two bounds yields a constant $C=C_v+C_c$ such that
\[
|g_{c,v}(x)| + \Big|\partial_v g_{c,v}(x)\Big|
+
\Big|\partial_c g_{c,v}(x)\Big|
\le C,
\]
for every $x \in K$, $0 < c_{\min}\le |c|\le c_{\max}, 0 < v_{\min}\le |v|\le v_{\max}.$
\end{proof}

\noindent \textbf{Proof of Lemma \ref{fdeltab}}:

\begin{proof}
Fix $C>0$. For $y\neq 0$,
\[
g'_y(t)
=\frac{e^{-y^2/(2t)}}{\sqrt{2\pi}}\left(-\frac{3}{2}\,\frac{|y|}{t^{5/2}}
+\frac{1}{2}\,\frac{|y|^3}{t^{7/2}}\right).
\]
For $p>0$ and $a>0$ define $\Phi_{p,a}(t):=t^{-p}e^{-a/t}$ on $t\ge C$. A standard calculus check shows
\begin{equation}\label{eq:Phi-bound}
\sup_{t\ge C}\Phi_{p,a}(t)\ \le\ \max\!\left\{C^{-p}e^{-a/C},\ \Big(\frac{p}{a}\Big)^p e^{-p}\right\},
\end{equation}
since the stationary point is at $t^\ast=a/p$.

\smallskip

Write $a=\tfrac{y^2}{2}$.
Using \eqref{eq:Phi-bound} with $p=\tfrac{5}{2}$ and $p=\tfrac{7}{2}$ we get
\[
\sup_{t\ge C}|g'_y(t)|
\ \le\ \frac{1}{\sqrt{2\pi}}\!\left(\tfrac{3}{2}\,|y|\,\sup_{t\ge C}\Phi_{5/2,a}(t)
+ \tfrac{1}{2}\,|y|^3\,\sup_{t\ge C}\Phi_{7/2,a}(t)\right)
\ \le\ C_2(C)\,\Big(e^{-y^2/(2\bar{M})} + |y|^{-4}\Big),
\]
for some finite constant $C_2(C)$ (explicit from \eqref{eq:Phi-bound}). Likewise,
\[
\sup_{t\ge C}|g_y(t)|
\ \le\ C_1(C)\,\Big(e^{-y^2/(2 C)} + |y|^{-2}\Big),
\]
for another finite constant $C_1(C)$ (apply \eqref{eq:Phi-bound} with $p=\tfrac{3}{2}$).

Since $\sum_{k\in\mathbb Z}e^{-c(2k+1)^2}<\infty$ for any $c>0$ and
$\sum_{k\in\mathbb Z}|2k+1|^{-m}<\infty$ for $m>1$ (in particular $m=2,4$), the series
\[
\sum_{k\in\mathbb Z}\sup_{t\ge C}|g_{y_k}(t)| \quad\text{and}\quad
\sum_{k\in\mathbb Z}\sup_{t\ge C}|g'_{y_k}(t)|
\]
converge. By the Weierstrass M-test, $\sum_k (-1)^k g_{y_k}$ and $\sum_k (-1)^k g'_{y_k}$ converge uniformly on $[C,\infty)$.
Therefore $f_\Delta$ is $C^1$ on $[C,\infty)$ and
\[
f'_\Delta(t)=\sum_{k\in\mathbb Z}(-1)^k g'_{y_k}(t)
\quad\text{uniformly on }[C,\infty).
\]
Moreover,
\[
\sup_{t\ge C}|f'_\Delta(t)|
\ \le\ \sum_{k\in\mathbb Z}\ \sup_{t\ge C}|g'_{y_k}(t)|
\ \le\ C_2(C)\!\left(\sum_{k\in\mathbb Z} e^{-(2k+1)^2/(2C)} + \sum_{k\in\mathbb Z} |2k+1|^{-4}\right)
\ <\ \infty.
\]
This completes the proof.
\end{proof}

\

\noindent \textbf{Acknowledgements}. The work of AO was partially supported by the Projeto Universal (CNPq) grant 408884/2023-1. The work of AO and SS were partially supported by the Europlace Institute of Finance, project "Tactical Asset Allocation using
Machine Learning: Memory, Frictions, and VIX."

\end{document}